\documentclass{article}

\usepackage{microtype}
\usepackage{graphicx}
\usepackage{subfigure}
\usepackage{booktabs} % for professional tables

\usepackage{hyperref}

\usepackage[T1]{fontenc}
\usepackage[accepted]{icml2024}

\usepackage{amsmath}
\usepackage{amssymb}
\usepackage{mathtools}
\usepackage{amsthm}
\usepackage{enumitem}
\usepackage{xspace}
\usepackage{Definitions}
\usepackage{algpseudocode}
\usepackage[capitalize,noabbrev]{cleveref}

\theoremstyle{plain}
\usepackage[textsize=tiny]{todonotes}

\newcommand{\algabb}{{$\mu$LV-Rep}\xspace}
\newcommand{\AlgName}{{Multi-step Latent Variable Representation}\xspace}

\icmltitlerunning{Provable Representation with Efficient Planning for Partially Observable Reinforcement Learning}

\begin{document}

\twocolumn[
\icmltitle{Provable Representation with Efficient Planning for Partially Observable Reinforcement Learning
% \icmltitle{Efficient Reinforcement Learning from Partial Observability
% via Representation View
}



\icmlsetsymbol{equal}{*}

\begin{icmlauthorlist}
\icmlauthor{Hongming Zhang${^*}$}{ua}
\icmlauthor{Tongzheng Ren${^*}$}{ut}
\icmlauthor{Chenjun Xiao}{cuhk}
\icmlauthor{Dale Schuurmans}{google,ua}
\icmlauthor{Bo Dai}{google,gt}
\end{icmlauthorlist}

\icmlaffiliation{ua}{University of Alberta}
\icmlaffiliation{ut}{UT Austin}
\icmlaffiliation{gt}{Georgia Tech}
\icmlaffiliation{google}{Google DeepMind}
\icmlaffiliation{cuhk}{The Chinese University of Hong Kong, Shenzhen}

\icmlcorrespondingauthor{Bo Dai}{bodai@google.com}

\icmlkeywords{Machine Learning, ICML}

\vskip 0.3in
]



\printAffiliationsAndNotice{\icmlEqualContribution} % otherwise use the standard text.

\begin{abstract}
In most real-world reinforcement learning applications, state information is only partially observable, which breaks the Markov decision process assumption and  leads to inferior performance for algorithms that conflate observations with state. Partially Observable Markov Decision Processes (POMDPs), on the other hand, provide a general framework that allows for partial observability to be accounted for in \emph{learning, exploration and planning}, but presents significant computational and statistical challenges. To address these difficulties, we develop a representation-based perspective that leads to a coherent framework and tractable algorithmic approach for practical reinforcement learning from partial observations. We provide a theoretical analysis for justifying the statistical efficiency of the proposed algorithm, and also empirically demonstrate the proposed algorithm can surpass state-of-the-art performance with partial observations across various benchmarks, advancing reliable reinforcement learning towards more practical applications.
\end{abstract}

\setlength{\abovedisplayskip}{1pt}
\setlength{\abovedisplayshortskip}{1pt}
\setlength{\belowdisplayskip}{1pt}
\setlength{\belowdisplayshortskip}{1pt}
\setlength{\jot}{1pt}

\setlength{\floatsep}{1ex}
\setlength{\textfloatsep}{1ex}

\vspace{-6mm}
\section{Introduction}\label{sec:intro}
\vspace{-1mm}

Reinforcement learning~(RL) addresses the problem of making sequential decisions that maximize a cumulative reward through interaction and observation in an environment~\citep{mnih2013playing,levine2016end}. The Markov decision process~(MDP) has been the standard mathematical model used for most RL algorithm design. However, the success of MDP-based RL algorithms~\citep{zhang2022making,ren2023spectral} relies on an assumption that  state information is fully observable, which implies that the optimal policy is memoryless, \ie, optimal actions can be selected based purely on the current state~\citep{puterman2014markov}. However, such an assumption typically does not hold in practice. For example, in chatbot learning~\citep{jiang2021towards} or video game control~\citep{mnih2013playing} only dialogue exchanges or images are observed, from which state information only can be partially inferred. The violation of full observability can lead to significant performance degeneration of MDP-based RL algorithms in such scenarios.

The Partially Observed Markov Decision Process (POMDP) \citep{aastrom1965optimal} has been proposed to extend the classical MDP formulation by introducing observation variables that only give partial information about the underlying latent state~\citep{hauskrecht2000planning,roy2002exponential,chen2016pomdp}. This extension greatly expands the practical applicability of POMDPs over MDPs, but the additional uncertainty over the underlying state given only observations creates a non-Markovian dependence between successive observations, even though Markovian dependence is preserved between latent states. Consequently, the optimal policy for a POMDP is no longer memoryless but \emph{entire-history} dependent, expanding the state complexity exponentially w.r.t.\ horizon length. Such a non-Markovian dependence creates significant computational and statistical challenges in \emph{planning}, and  in \emph{learning} and \emph{exploration}. In fact, without additional assumptions, computing an optimal policy for a  POMDP with known dynamics (\ie, planning) is PSPACE-complete~\citep{papadimitriou1987complexity}, while the sample complexity of learning for POMDPs grows exponentially w.r.t.\ the horizon~\citep{jin2020sample}. 

Despite the worst case hardness of POMDPs, given their importance in practice, there has been extensive work on developing practical RL algorithms that can cope with partial observations. One common heuristic is to extend MDP-based RL algorithms by maintaining a history window over observations to encode a policy or value function, \eg, recurrent neural networks~\citep{wierstra2007solving,hausknecht2015deep,zhu2017improving}. Such algorithms have been applied to many real-world applications with image- or text-based observations~\citep{berner2019dota,jiang2021towards}, sometimes even surpassing human-level performance~\citep{mnih2013playing, kaufmann2023champion}.

These empirical successes have motivated investigation into \emph{structured} POMDPs that allow some of the core computational and statistical complexities to be overcome, and provides an improved understanding of exploitable structure and practical new algorithms with rigorous justification. For example, the concept of \emph{decodability} has been used to express POMDPs where the latent state can be exactly recovered from a window of \emph{past} observations~\citep{efroni2022provable,guo2023provably}. \emph{Observability} is another special structure, where the $m$-step emission model is assumed to be full-rank, allowing the latent state to be identified from $m$ \emph{future} observation sequences~\citep{jin2020sample,golowich2022learning,liu2022partially,liu2023optimistic}. Such structures eliminate unbounded history dependence, and thus, reduce the computational and statistical complexity. However, most works rely on the existence of an ideal computational oracle for planning, which, unsurprisingly, is infeasible in most cases, hence difficult to apply in practice. Although there have been a few attempts to overcome the computational complexity of POMDPS, these algorithms are either only applicable to the tabular setting~\citep{golowich2022learning} or rely on an integration oracle that quickly become intractable for large observation spaces~\citep{guo2018neural}. This gap immediately motivates the question:
\vspace{-2mm}
\begin{center}
    %\emph{How can we design an {\bf efficient} and {\bf practical} RL algorithm for partial observations by exploiting the structures? }
    \emph{Can {\bf efficient} and {\bf practical} RL algorithms be designed for partial observations by exploiting natural structures?}
\vspace{-2mm}
\end{center}
By ``efficient" we mean the statistical complexity avoids an exponential dependence on history length, 
while by ``practical" we mean that every component of \emph{learning, planning} and \emph{exploration} can be easily implemented and applied in practical settings. 
In this paper, we provide an {\bf affirmative} answer to this question. More specifically, 
\vspace{-2mm}
\begin{itemize}[leftmargin=*,itemsep=0pt,topsep=3pt]
% \vspace{-0.5em}
    \item We reveal for the first time that an $L$-decodable POMDP admits a sufficient representation, the \emph{\AlgName~(\algabb)}, that supports exact and tractable linear representation of the value functions (\cref{subsec:linear_repr}), breaking the fundamental computational barriers explained in more detail in~\cref{sec:difficulties}.
 % \vspace{-0.5em}
    \item We design a computationally efficient planning algorithm that can implement both the principles of optimism and pessimism in the face of uncertainty for online and offline POMDPs respectively, by leveraging the learned sufficient representation~\algabb (\cref{subsec:algorithm}).
% \vspace{-0.5em}
    \item We provide a theoretical analysis of the sample complexity of the proposed algorithm, justifying its efficiency in balancing exploitation versus exploration in~\cref{sec:theory}.  
% \vspace{-0.5em}
    \item We conduct a comprehensive empirical comparison to current existing RL algorithms for POMDPs on several benchmarks, demonstrating the superior empirical performance of~\algabb (\cref{sec:exp}).
\end{itemize}
\vspace{-2mm}
\section{Preliminaries}\label{sec:prelim}
\vspace{-1mm}
We follow the definition of a POMDP given in~\citep{efroni2022provable,liu2022partially,liu2023optimistic},
which is formally denoted as a tuple $\mathcal{P} = (\mathcal{S}, \mathcal{A}, \mathcal{O}, r, H, \rho_0, \PP, \OO)$, where $\mathcal{S}$ is the state space, $\mathcal{A}$ is the action space, and $\mathcal{O}$ is the observation space. The positive integer $H$ denotes the horizon length, {\color{black} $\rho_0$ is the initial state distribution,} $r: \mathcal{O} \times \mathcal{A} \to [0, 1]$ is the reward function, $\PP(\cdot | s, a): \mathcal{S} \times \mathcal{A} \to \Delta(\mathcal{S})$ is the transition kernel capturing dynamics over latent states, and $\OO(\cdot | s): \mathcal{S} \to \Delta(\mathcal{O})$ is the emission kernel, which induces an observation from a given state.

Initially, the agent starts at a state $s_0$ drawn from $\rho_0(s)$. At each step $h$, the agent selects an action $a$ from $\mathcal{A}$. This leads to the generation of a new state $s_{h+1}$ following the distribution $\PP(\cdot | s_h, a_h)$, from which the agent observes $o_{h+1}$ according to $\OO(\cdot | s_{h+1})$. The agent also receives a reward $r(o _{h+1}, a_{h+1})$. Observing $o$ instead of the true state $s$ leads to a non-Markovian transition between observations, which means we need to consider policies
% Consequently, we define a policy $\pi = {\pi_t}$, where 
$\pi_h: \mathcal{O} \times (\mathcal{A} \times \mathcal{O})^{h} \to \Delta(\mathcal{A})$ that depend on the entire history, denoted by $\tau_h = \{o_0, a_0, \cdots, o_h\}$. 
Let $[H]\defeq \cbr{0,\ldots, H}$.
Then the value associated with policy $\pi=\{\pi_h\}_{h\in[H]}$ is defined as $v^\pi = \mathbb{E}_\pi \left[\sum_{h\in [H]} r(o_h, a_h)\right]$. The goal is to find the optimal policy $\pi^* = \arg\max_\pi v^\pi$. Note that the MDP given by $\mathcal{M} = (\mathcal{S}, \mathcal{A}, r, H, \rho_0, \PP)$ is a special case of a POMDP, where the state space $\mathcal{S}$ is equivalent to the observation space $\mathcal{O}$, and the emission kernel $\OO(o|s)$ is defined as $\delta(o=s)$. 

Define the belief function $b(\cdot): \mathcal{O} \times (\mathcal{A} \times \mathcal{O})^{h} \to \Delta(\mathcal{S})$. Let $b(s_1|o_1) = \PP(s_1|o_1)$. Then we can recursively compute:
% \vspace{-2mm}
\begin{align}\label{eq:belief}
    b(s_{h+1}|\tau_{h+1})\!\propto\!\!\int_{\mathcal{S}}b(s_h|\tau_h) \PP(s_{h+1}|s_h, a_h)\OO(o_{h+1}|s_{h+1}) ds_h.
\end{align}
With such a definition, one can convert a POMDP to an equivalent MDP over beliefs, denoted as $\mathcal{M}_b = \left(\mathcal{B}, \mathcal{A}, R_h, H, \mu_b, T_b\right)$, where $\mathcal{B}\subseteq \Delta(\mathcal{S})$ represents the set of possible beliefs, $\mu_b(\cdot) = \int b(\cdot|o_1)\mu(o_1)do_1$, and
\begin{multline}
    % R_h(b, a) = \int b_h(s_h) r(s_h, a) ds_h, \nonumber \\
    \PP_b\left(b_{h+1}|b_h, a_h\right) \\
    = \int \mathbf{1}_{b_{h+1} = b(\tau_h, a_h, o_{h+1})}\PP(o_{h+1}|b_h, a_h) d o_{h+1}. \label{eq:belief_mdp}
\end{multline}
Notice that 
$b(\cdot) \in \Bcal$ is a mapping $b: \mathcal{O} \times (\mathcal{A} \times \mathcal{O})^{h} \to \Delta(\mathcal{S})$, \ie, each belief corresponds to a density measure over the state space, thus $\PP_b\rbr{\cdot|b_h, a_h}$ represents a conditional operator that characterizes the transition between belief distributions.
For a given policy $\pi:\mathcal{B} \to \Delta(\mathcal{A})$, we can define the state value function $V_{h}^{\pi}(b_h)$ and state-action value function $Q_{h}^\pi(b_h, a_h)$ respectively for the belief MDP (\ie, equivalently for the original POMDP) as:
\begin{eqnarray}\label{eq:belief_value}
\textstyle
 &V_h^\pi(b_h) = \mathbb{E}\left[\sum_{t=h}^{H} r(o_t, a_t)|b_h\right], \quad \nonumber \\
 &Q_h^{\pi}(b_h, a_h) = \mathbb{E}\left[\sum_{t=h}^{H} r(o_t, a_t)|b_h, a_h\right].
\end{eqnarray}
Therefore, the Bellman equation can be expressed as
\begin{align}\label{eq:bellman_belief}
& V_h^\pi(b_h) = \mathbb{E}_\pi\left[Q_h^\pi(b_h, a_h)\right],\quad \nonumber \\
& Q_h^\pi(b_h, a_h) = r(o_h, a_h) + \mathbb{E}_{\PP_b}\left[V_{h+1}^\pi(b_{h+1})\right].
\end{align}
Although this reformulation of a POMDP provides a reduction to an equivalent MDP, 
it is crucial to recognize that the equivalent MDP is based on beliefs, which are also not directly observed. More importantly, these beliefs are densities that depend on the \emph{entire history}, hence the %history-state 
joint
distribution is supported on a space of growing dimension, leading to exponential representation complexity 
%over exponential large or infinite number of support, 
even when the number of states is finite. These difficulties result in infeasible computational and statistical complexity~\citep{papadimitriou1987complexity,jin2020sample}, and %thus, justifying the suboptimality in 
reveal the inherent suboptimality of directly applying MDP-based RL algorithms in a POMDP.
Moreover, incorporating function approximation in learning and planning is significantly more involved in a POMDP than an MDP.

Consequently, several special structures have been investigated in the theoretical literature to reduce the statistical complexity of learning in a POMDP.
Specifically, $L$-\emph{decodability} and $\gamma$-\emph{observability} have been introduced in~\citep{du2019provably,efroni2022provable} and ~\citep{golowich2022learning,even2007value} respectively.
\vspace{-1mm}
\begin{definition}[$L$-decodability~\citep{efroni2022provable}]
\label{assump:decodable}
    $\forall h\in [H]$, define 
    \vspace{-1mm}
    \begin{align}\label{eq:decodable}
        & x_h \in \mathcal{X} := (\mathcal{O}\times \mathcal{A})^{L-1}\times \mathcal{O}, \nonumber \\
        & x_h = (o_{h-L+1}, a_{h-L+1}, \cdots, o_{h}).
    \end{align} 
    A POMDP is $L$-decodable if there exists a decoder $p^*:\mathcal{X}\to \Delta(\mathcal{S})$ such that $p^*(x_h) = b(\tau_h)$.
\end{definition}
\vspace{-1mm}
\begin{definition}[$\gamma$-observability~\citep{golowich2022learning,even2007value}]
    Denote $\inner{\OO}{b} \defeq \int \OO_h\rbr{\cdot|s}b\rbr{s}ds$, for
    arbitrary beliefs $b$ and $b'$ over states.
    A POMDP is $\gamma$-observable if $\nbr{\inner{\OO}{b} - \inner{\OO}{b'}}_1\ge \gamma \nbr{b - b'}_1$. 
\end{definition}
\vspace{-1mm}
Note that we slightly generalize the decodability definition of \citet{efroni2022provable}, which assumes $b(\tau_h)$ is a Dirac measure on $\mathcal{S}$.  It is worth noticing that $\gamma$-observability and $L$-decodability are highly related. Existing works have shown that a $\gamma$-observable POMDP can be well approximated by a decodable POMDP with a history of suitable length $L$~\citep{golowich2022learning,uehara2022provably,guo2023provably} (see~\cref{appendix:observability} for a detailed discussion). Hence, in the main text, we focus on $L$-decodable POMDPs, which can be directly extended to $\gamma$-observable POMDPs. 

Although these structural assumptions have been introduced to ensure the sample complexity of learning in a structured POMDP is reduced, the \emph{computational} tractability of planning and exploring given such structures remains open. 
\vspace{-3mm}
\section{Difficulties in Learning with POMDPs}\label{sec:difficulties}
\vspace{-2mm}

% Obviously, such structures for a POMDP reduce the history dependence, and thus reduce statistical complexity. However, the computational tractability of planning and exploration given such structures remains open.
Before attempting to design practical RL algorithms for structured POMDPs, we first demonstrate the generic difficulties of learning in POMDPs, including the difficulty of planning even in a \emph{known} POMDP.
These challenges only become more difficult when combined with the need to \emph{explore} in an unknown POMDP during learning. %which will be significantly more challenging accompanying with \emph{exploration in learning}, when the POMDP is unknown. 

The major difficulty of {\bf planning in a known POMDP} lies in evaluating the value function~\citep{ross2008online}, which is necessary for both temporal-difference and policy gradient based planning and learning methods. For clarity in explaining the difficulties, consider estimating $Q^\pi$ for a given policy $\pi$. (The same difficulties arise in estimating $Q^*$.) %for POMDPs. 
Based on the Bellman equation~\eqref{eq:bellman_belief}, we have:
\begin{align*}
Q^\pi_h\rbr{b_h, a_h}
= r\rbr{o_h, a_h} + \EE_{\PP(o_{h+1}|b_h, a_h)}\sbr{{V^\pi_{h+1}\rbr{b_{h+1}}}}.
\end{align*}
To execute dynamic programming in the $h^{\textrm{th}}$ time step, \ie, to determine $Q^\pi_h$, we require several components:
\begin{itemize}[leftmargin=*,topsep=0pt,parsep=0pt,noitemsep]
    \item[{\bf i)}]  The belief $b_{h+1}(\cdot)$ must be calculated, as defined in~\eqref{eq:belief}, which requires marginalization and renormalization, with a computational complexity proportional to the number of states. For continuous control, %where the number of states is infinite, 
    the calculation involves integration, which is intractable in general.      
    \item[{\bf ii)}] Representing the value function $V_{h+1}^\pi(\cdot)$ %or $V^*_{h+1}(\cdot)$, 
    quickly becomes intractable, even for discrete state and action spaces. In general, expressing $V_{h+1}^\pi(\cdot)$ requires space proportional to size of the support of beliefs, which is exponential w.r.t.\ the horizon. 
    (The optimal value function 
    $V_{h+1}^*\rbr{\cdot}$ can be represented by convex piece-wise linear function~\citep{smallwood1973optimal}, but the number of components can grow exponentially in the horizon.) %the number of observations at each iteration. %There is no known result for the function space of value function yet for continuous state POMDPs. 
    For continuous state POMDPS, no compact characterization is yet known for the space of value functions.
    \item[{\bf iii)}] The expectation $\EE_{\PP(o_{h+1}|b_h, a_h)}\sbr{V^\pi\rbr{b_{h+1}}}$ requires integration w.r.t.\ the distribution $\PP(o_{h+1}|b_h, a_h)= \int \OO(o_{h+1}|s_{h+1})\PP\rbr{s_{h+1}|s_h, a_h}b_h\rbr{s_h}ds_h ds_{h+1}$ with computational cost proportional to the size of the support of beliefs, which is also exponential in the horizon. 
\end{itemize}
These calculations are generally intractable in a POMDP, particularly one with a continuous state space.
    %generally intractable in general cases with continuous state, with complicated $\OO(o_{h+1}|s_{h+1})$, $\PP\rbr{s_{h+1}|s_h, a_h}$ and the value functions are 
%As discussed, 
Due to this intractability, the exact calculation of the value function via the Bellman equation is generally infeasible. 
Consequently, many approximation methods have been proposed for each component of planning~\citep{ross2008online}. 
However, the approximation error in each step can be amplified when these steps are composed through the Bellman recursion,
%will be amplified through the compositionality of these steps and furthermore the recursion in Bellman equation, l
leading to significantly suboptimal performance. 

Due to promising theoretical results that exploit $\gamma$-observability or $L$-decodability to reduce the exponential dependence on horizon in terms of statistical complexity, there have been many attempts to exploit these properties to also achieve computationally tractable planning for RL in a POMDP.
\citet{golowich2022learning} construct an approximate MDP for an $L$-observable POMDP, whose approximation error can be characterized for planning. However, the proposed algorithm only works for POMDPs with discrete observations, states and actions. 
Alternatively,~\citet{guo2023provably} exploits low-rank latent dynamics in an $L$-decodable POMDP, which induces a \emph{nonlinear} representation for value functions. 
% which induces the factorization $\PP\rbr{o_{h+1}|b_h, a_h}= \inner{\phi\rbr{x_h, a_h}}{\mu\rbr{o_{h+1}}}$. 
% The factorization reveals a structure of $Q^\pi$ through Bellman equation, \ie, 
% %$= \phi\rbr{o_{h-L+1: h}, a_{L-h+1: h}}^\top\mu\rbr{o_{h+1}}$. 
% \begin{multline*}
%     Q_h^\pi\rbr{b_h, a_h} = r_h\rbr{o_h, a_h} + \\
%     \gamma \inner{\phi\rbr{x_h, a_h}}{\int  \mu\rbr{o_{h+1}}  V_{h+1}\rbr{x_h, a_h, o_{h+1}}  d{o_{h+1}}},    
% \end{multline*}
% which enables the potential extension for continuous observation and state. 
Unfortunately, this nonlinear structure has not been shown to support any practical algorithm. 
% In fact, due to the dependency of $\rbr{x_h, a_h}$ in $\int  \mu\rbr{o_{h+1}}  V_{h+1}\rbr{x_h, a_h, o_{h+1}}  d{o_{h+1}}$, the linear span of $\phi\rbr{x, a}$ is not enough for representing $Q^\pi$. 
To implement the backup step in dynamic programming for $Q^\pi$, the algorithm in~\citep{guo2023provably} requires an integration-oracle to handle the nonlinearity, which is generally infeasible.

Another barrier to practical learning in a POMDP is {\bf exploration}, which is usually implemented in planning by appealing to the principle of optimism in the face of uncertainty~(OFU). 
Unfortunately, OFU usually requires not only the estimation of the model or value function, but also their uncertainty, \eg, a confidence interval in frequentist approaches~\citep{liu2022partially,liu2023optimistic} or the posterior in Bayesian approaches~\citep{simchowitz2021bayesian}. Practical and general implementations are not yet known in either case.

% For instance, the Barycentric spanner exploration is proposed in~\citep{golowich2022learning} for tabular cases, which is still not applicable for general POMDPs; the maximum likleihood~(MLE) ball is considered for exploration 

These major issues, which have not been carefully considered in empirical RL algorithms that leverage history-based neural networks~\citep{wierstra2007solving,hausknecht2015deep,zhu2017improving}, interact in ways that exaggerate the difficulty of developing efficient 
% and practical RL 
algorithms for POMDPs.

\vspace{-2mm}
\section{\AlgName}\label{sec:method}
\vspace{-2mm}

In this section, we develop the first {\bf computationally efficient} planning and exploration algorithm by further leveraging the structure of an $L$-decodable POMDP. 
The main contribution is to identify a latent variable representation that can \emph{linearly represent} the value function of an arbitrary policy in an $L$-step decodable POMDP. From this representational perspective, we show that explicit belief calculation can be bypassed and the backup step can be directly calculated via optimization.\
We then show how a latent variable representation learning algorithm can be designed, based upon~\citep{ren2023latent}, that leads to the proposed~\emph{\AlgName~(\algabb)} method. 

%%%%%----------------------------------------------------------------
\vspace{-2mm}
\subsection{Efficient Policy Evaluation from Key Observations}\label{subsec:linear_repr}
\vspace{-1mm}
%%%%%----------------------------------------------------------------

Although the equivalent belief MDP provides a Markovian Bellman recursion~\eqref{eq:bellman_belief}, as discussed in~\cref{sec:difficulties}, this belief viewpoint hampers tractability in planning. In the following, we introduce the key observations that allow us to resolve the aforementioned difficulties in planning step-by-step. 

\vspace{-1em}
\paragraph{Belief Elimination.}
%We specify our first key observation from the $L$-decodability to reduce exponential complexity, %\ie, 
Our first key observation is that $L$-decodability eliminates an aspect of exponential complexity:
%\vspace{-2mm}
\vspace{-6mm}
% \begin{center}
\begin{itemize}[leftmargin=*, topsep=0pt, parsep=0pt, partopsep=0pt]
    \item 
    \emph{In an $L$-decodable POMDP, it is sufficient to recover the belief state by an $L$-step memory $x_h$ rather than the entire history $\tau_h$.}
\end{itemize}
\vspace{-2mm}
This is exactly the definition of $L$-decodability given in Definition~\ref{assump:decodable}. 
From this observation that beliefs $b_h\rbr{s}$ can be represented with a decoder from $L$-step windows $x_h$ \eqref{eq:decodable}, 
we obtain the simplification
$Q_h^\pi\rbr{b_h\rbr{\tau_h}, a_h} = Q_h^\pi\rbr{b_h\rbr{x_h}, a_h}$. In addition, $L$-decodability also helps reduce dependence in the dynamics, leading to $\PP\rbr{o_{h+1}|b_h\rbr{\tau_h}, a_h} = \PP\rbr{o_{h+1}|b_h\rbr{x_h}, a_h}$. This outcome reduces the statistical complexity, as previously exploited in~\citep{efroni2022provable,guo2023provably}. However, it is not sufficient to ensure computationally friendly planning. 

Next we make a second key observation that allows us to eliminate the explicit use of beliefs in the Bellman equation:
%
%\vspace{-1mm}
\vspace{-5mm}
% \begin{center}
\begin{itemize}[leftmargin=*, topsep=0pt, parsep=0pt, partopsep=0pt]
    \item
    \emph{An observation-based value function and transition model bypasses the necessity for belief computation.}
\end{itemize}
% \end{center}
\vspace{-1mm}
That is, since $b_h\rbr{\cdot}$ is a mapping from the history $x_h$ to the space of probability densities over state, we can reparametrize the composition $Q_h^\pi\rbr{b_h(x_h), a_h}$ as $Q_h^\pi\rbr{x_h, a_h}$, which is a function directly over $\rbr{x_h, a_h}$. 
Similarly, we can also consider the transition dynamics directly defined over the observation history, $\PP\rbr{o_{h+1}|x_h, a_h}$, rather than work with the far more complex dynamics between successive beliefs. % dependent dynamics.  
With such a reformulation, we avoid explicit dependence on beliefs in the $Q^\pi$-function and dynamics in the Bellman recursion~\eqref{eq:bellman_belief}.
This resolves the first difficulty that arises from explicit belief calculation, leading to the simplified Bellman equation:
\begin{equation}\label{eq:obs_bellman}
\resizebox{.9\hsize}{!}{
    $Q_h^\pi\rbr{x_h, a_h} = r\rbr{o_h, a_h} + \EE_{\PP^\pi\rbr{o_{h+1}|x_h,a_h}}\sbr{V_{h+1}^\pi\rbr{x_{h+1}}}.
    $}
\end{equation}
\vspace{-2em}
\paragraph{Linear Representation for $Q^\pi$.} 
Next, we address a second key difficulty that arises in representing $V_{h+1}^\pi(x_{h+1}) = \EE_{a\sim\pi}\sbr{Q_{h+1}^\pi\rbr{x_{h+1}, a_{h+1}}}$. 
We seek an exact representation that can be recursively maintained.
Inspired by the success of representing value functions exactly in linear MDPs~\citep{jin2020provably,yang2020reinforcement}, it is natural to consider $x_h$ as a mega-state~\citep{efroni2022provable} and factorize the observation distribution in \eqref{eq:obs_bellman} as $\PP\rbr{o_{h+1}|b_h, a_h}= \inner{\phi\rbr{x_h, a_h}}{\mu\rbr{o_{h+1}}}$ over basis functions $\phi$ and $\mu$. % for~\eqref{eq:obs_bellman}. 
However, it is important to realize there is an additional dependence of $V_{h+1}^\pi(x_{h+1})$ on $(x_h, a_h)$, since 
\[
x_{h+1} = ({\color{red} o_{h-L+2}, a_{h-L+2}, \cdots, o_h, a_{h},} o_{h+1})
\]
has an overlap with 
\[
(x_h, a_h) = (o_{h-L+1}, a_{h-L+1}, {\color{red} o_{h-L+2}, a_{h-L+2}, \cdots, o_{h}, a_h}).
\]
This overlap %between the successive histories $(x_h, a_h)$ and $x_{h+1}$ 
breaks the maintained linear structure, since
\begin{multline*}
    Q_h^\pi\rbr{b_h, a_h} = r_h\rbr{o_h, a_h} + \\
    \textstyle 
    \inner{\phi\rbr{x_h, a_h}}{\underbrace{\int  \mu\rbr{o_{h+1}}  V_{h+1}\rbr{x_h, a_h, o_{h+1}}  d{o_{h+1}}}_{w\rbr{x_h, a_h}}}.   
\end{multline*}
That is, the nonlinearity in $w\rbr{x_h, a_h}$ prevents the direct extension of linear MDPs with mega-states~\citep{efroni2022provable} to the POMDP case~\citep{guo2023provably}. 

Nevertheless, a linear representation for the value function can still be recovered if the dependence of the $h+1$-step value function on $(x_h, a_h)$ can be eliminated.
The next key observation shows how this overlap might be avoided:
\vspace{-2mm}
% \begin{center}
\begin{itemize}[leftmargin=*, topsep=0pt, parsep=0pt, partopsep=0pt]
    \item
    \emph{By $L$-step decodability, $V_{h+L}^\pi(x_{h+L})$ is independent of $\rbr{x_h, a_h}$.}
\end{itemize}
% \end{center}
\vspace{-2mm}
This observation inspires us to consider the \emph{$L$-step} Bellman equation for $Q_h^\pi\rbr{\tau_h, a_h}$, which can be easily derived by expanding~\eqref{eq:bellman_belief} forward in time to:
% \vspace{-3mm}
% \begin{footnotesize}
\begin{align}\label{eq:bellman_multistep}
&Q_h^\pi\rbr{x_h, a_h} = \\
& \textstyle\EE^\pi_{x_{h+1:h+L}|x_h, a_h} \left[\sum_{i = h}^{h+L-1} r(o_i, a_i) + V_{h+L}^\pi(x_{h+L})\right]. \nonumber
\end{align}
% \end{footnotesize}
At first glance, the $L$-step forward expansion only involves $V_{h+L}^\pi(x_{h+L})$, which would seem to eliminate the overlapping dependence of $x_{h+L}$ on $(x_h, a_h)$ due to $L$-step decodability. 
However, the forward steps are conducted according to the policy $\pi_{h+1:h+L-1}$, which still depends on parts of $(x_h, a_h)$, hence the distribution $\PP^\pi\rbr{x_{h+L}|x_h, a_h}$ in the expectation in \eqref{eq:bellman_multistep} still retains a dependence on $(x_h,a_h)$. 

We introduce the final key observation
that allows us to fully eliminate the dependence:  
\vspace{-2mm}
% \begin{center}
\begin{itemize}[leftmargin=*, topsep=0pt, parsep=0pt, partopsep=0pt]
    \item 
    \emph{For any policy $\pi$, there exists a corresponding policy $\nu_{\pi}$ that conditions on a sufficient %$h$-step 
    latent variable to generate the same expected observation dynamics while being independent of history older than $L$ steps. } 
\end{itemize}
\vspace{-2mm}
The $\nu_\pi$ is known as the \emph{``moment matching policy"}~\citep{efroni2022provable}. 
The existence of such an equivalent $\nu^\pi$ is guaranteed by $L$-decodability: 
we defer the detailed construction of $\nu_\pi$ to~\cref{appendix:moment_matching} to avoid distraction, and focus on algorithm design in the main text. 

With such a policy $\nu_\pi$, the dependence of $\pi_{h+1:h+L-1}$ on $\rbr{x_h, a_h}$ at the $L^{{\textrm th}}$-step in $\PP^{\pi}\rbr{x_{h+L}|x_h, a_h}$ can be eliminated. Specifically, consider
% \begin{small}
\begin{align}\label{eq:decomposition}
\hspace{-1mm}\PP^\pi\rbr{x_{h+L}|x_h, a_h} \hspace{-1mm} & =\hspace{-1mm} \int\hspace{-1mm} p(z_{h+1}|x_h, a_h)\PP^{\nu_{\pi}}\rbr{x_{h+L}|z_{h+1}}d{z_{h+1}} 
\nonumber \\
% \textstyle
&\hspace{-1mm}=\hspace{-1mm} \inner{p(\cdot|x_h, a_h)}{\PP^{\nu_{\pi}}(x_{h+L}|\cdot)}_{L_2(\mu)},
\end{align}
% \end{small}
where $z$ denotes the latent variable and the first equality follows from the construction of $\nu_\pi$. 
We emphasize that the idea of the moment matching policy was only considered as a proof technique in~\citep{efroni2022provable}, and not previously exploited to uncover linear structure for algorithm design. 

\vspace{-1em}
\paragraph{Remark (Identifiability):} It should be noted that we deliberately use a latent variable $z$ rather than $s$ in~\eqref{eq:decomposition}, to emphasize the learned latent variable structure can be different from the ground truth state, thus avoiding an \emph{identifiability} assumption. 
Nevertheless, the learned structure has the same effect in representing $Q^\pi$ linearly.

Based on the above observations, we have addressed the necessary components for a linear representation of the value function, achieved by introducing~\eqref{eq:decomposition} into~\eqref{eq:bellman_multistep}.
For the first term in %the $L$-step Bellman equation~
\eqref{eq:bellman_multistep}, for $\forall k\in \cbr{1, \ldots, L-1}$, we have
% \begin{scriptsize}
\begin{align}\label{eq:linear_first}
& 
\resizebox{.5\hsize}{!}{$\EE^\pi_{o_{h+k}|x_h, a_h}\sbr{r\rbr{o_{h+k}, a_{h+k}}}$}  \\
& 
\resizebox{.98\hsize}{!}{
$= \inner{p(\cdot|x_h, a_h)}{\underbrace{\int \PP^{\nu_\pi} \rbr{o_{h+k}, a_{h+k}|\cdot}r\rbr{o_{h+k}, a_{h+k}}do_{h+k}da_{h+k}}_{w_k^\pi(\cdot)}}.$ \nonumber
}
\end{align}
% \end{scriptsize}
That is, using the ``moment matching policy" trick, the policy $\nu_\pi$ is independent of $(x_h, a_h)$, while $\PP^{\nu_\pi} \rbr{o_{h+k}, a_{h+k}|\cdot}$ is independent of history, leading to the linear representation in~\eqref{eq:linear_first}. 
Similarly, for the second term in~\eqref{eq:bellman_multistep}, we have
% \begin{scriptsize}
\begin{align}\label{eq:linear_second}
    &
    \resizebox{0.85\hsize}{!}{$\EE_{\pi}\sbr{V_{h+L}^\pi\rbr{x_{h+L}}} = \int \PP^\pi\rbr{x_{h+L}|x_h, a_h}V^\pi\rbr{x_{h+L}}dx_{h+L}$} \\
    & 
    \resizebox{.85\hsize}{!}{
    $= \inner{p(\cdot|x_h, a_h)}{\underbrace{\int \PP^{\nu_\pi} \rbr{x_{h+L}|\cdot}V^\pi\rbr{x_{h+L}}dx_{h+L}}_{w^\pi_{h+L}(\cdot)}}$
    \nonumber
    }.
\end{align}
% \end{scriptsize}
Recall that $x_{h+L}$ does not overlap with $x_h$, so %using the same ``moment match policy" trick, 
under $\nu^\pi$,
$w_{h+L}^\pi\rbr{\cdot}$ is independent of $(x_h, a_h)$. 

By combining~\eqref{eq:linear_first} and~\eqref{eq:linear_second} and defining $w^\pi = \sum_{k=h}^{h+L}w_{h+k}^\pi$, we conclude that, in an $L$-step decodable POMDP,  the value function $Q^\pi$ can be represented linearly in $p\rbr{\cdot|x_h, a_h}$ as
\begin{equation}
    Q_h^\pi\rbr{x_h, a_h} = \inner{p\rbr{\cdot|x_{h}, a_h}}{w^\pi(\cdot)}_{L_2(\mu)},
\end{equation}
under the assumption that $r\rbr{o_h, a_h} = \inner{p(\cdot|x_h, a_h)}{\omega^r(\cdot)}$, which can be easily achieved by feature augmentation~\citep{ren2023stochastic}. 

\vspace{-3mm}
\paragraph{Least Square Policy Evaluation.} The final difficulty in a practical algorithm design is addressing the expectation calculation in the Bellman equation~\eqref{eq:bellman_multistep}. With the linear representation established for the value function $Q_h^\pi$, the entire backup step for dynamic programming with the Bellman recursion can be %bypassed through 
replaced by
a least squares regression in the space spanned by $\phi\rbr{x_h, a_h}$. Specifically, at step $h$, we seek an estimate of $Q_{h+L}^\pi(x_{h+L}, a_{x+L})$ that is expressed as $\inner{\wtil^\pi_{h+L}}{p(\cdot|x_{h+L}, a_{a+L})}$, which can be obtained by the optimization: 
\begin{align}\label{eq:lspe}
    \min_{w^\pi_h}\,\, &\EE^\pi_{x_{h+1: h+L},x_h, a_h}\bigg[\bigg( \inner{w^\pi_h}{p(\cdot|x_{h}, a_{h})}- 
    \\
    &\Big(\sum_{i=h}^{h+L-1} r(o_i, a_i) + \inner{\wtil^\pi_{h+L}}{p(\cdot|x_{h+L}, a_{h+L})}\Big)\bigg)^2\bigg], \nonumber
\end{align}
This problem can be easily solved by stochastic gradient descent. 
Setting the gradient to zero, one obtains the optimal solution to~\eqref{eq:lspe} as:
\begin{multline*}
    \inner{w^\pi_h}{p(\cdot|x_{h}, a_{h})} = \\
\EE_{x_{h+1: h+L}}\sbr{\sum_{i=h}^{h+L-1} r(o_i, a_i) + \inner{\wtil^\pi_{h+L}}{p(\cdot|x_{h+L}, a_{h+L})}},    
\end{multline*}
which completes the Bellman backup.

In summary, we obtain an efficient policy evaluation algorithm for $L$-decodable POMDPs that only requires {\bf least squares optimization} upon a {\bf linear representation} of an {\bf observation-based} value function.

\vspace{-1em}
\paragraph{Remark (Connection to Linear MDPs~\citep{jin2020provably,yang2020reinforcement}): } 
The linear structure we have revealed for POMDP value functions has some similarity to value function representations in linear MDPs.
One key difference between the factorization in~\eqref{eq:decomposition} and that in a linear MDP, \ie, $\PP\rbr{s'|s, a} = \inner{\phi\rbr{s, a}}{\mu\rbr{s'}}$,
% or $\inner{p(\cdot|s, a)}{p(s'|\cdot)}$,
is that the latter factorizes for different transitions of the representation. Therefore, in a linear MDP, one obtains a \emph{policy-independent} decomposition, where both the components $\phi\rbr{s, a}$ and $\mu\rbr{s'}$ from the transition dynamics are invariant w.r.t.\ the policy. 
By contrast, in~\eqref{eq:decomposition} for a POMDP, one of the factors, $\PP^{\nu_\pi}\rbr{x_{h+L}|\cdot}$, depends on the policy. Nevertheless, we have demonstrated that this does not affect the linear representation ability of $p\rbr{\cdot|x_h, a_h}$ for $Q^\pi$. 

\vspace{-1em}
\paragraph{Remark (Connection to PSR~\citep{littman2001predictive}:} Both the linear representation revealed in this paper and the predictive state representation (PSR)~\citep{littman2001predictive} bypass explicit belief calculation by factorizing the observation transition system. However, there are significant differences between the two structures that affect planning and exploration. Specifically, the PSR is based on the assumption that, for any finite sequence of events $y_{h+1:k} = \rbr{o_{h+1:h+k}, a_{h:h+k-1}}$, $k\in \NN_+$, following a history $x_h$,  the probability can be linearly factorized as
$\PP\rbr{o_{h+1:h+k}|x_h, a_{h:h+k-1}} = \inner{\omega_{y_{h+1:k}}}{\PP\rbr{U|x_h}}$, where $\omega_{y_{h+1:k}}\in \RR^{d}$, $U\defeq \sbr{u_i}_{i=1}^d$ is a set of core test events, and  $\PP\rbr{U|x_h}$ is the predictive state representation at step $h$. Then, the forward observation dynamics can be represented in a PSR via Bayes' rule, $\PP\rbr{o_{h+2:k}|x_h, a_{h:h+k-1}, o_{h+1}} = \frac{\inner{\omega_{y_{h+2:k}}}{\PP\rbr{U|x_h}}}{\inner{\omega_{y_{h+1}}}{\PP\rbr{U|x_h}}}$, which introduces a \emph{nonlinear} operation, making planning and exploration both extremely difficult. 

%%%%%--------------------------------------------------------------
\vspace{-2mm}
\subsection{Learning with Exploration}\label{subsec:algorithm}
\vspace{-2mm}
%%%%%--------------------------------------------------------------

\begin{algorithm}[t] 
\caption{Online Exploration for $L$-step decodable POMDPs 
with Latent Variable Representation} \label{alg:lvrep_pomdp}
\begin{algorithmic}[1]
    \State \textbf{Input:} Model Class $\mathcal{M}=\{\{(p_h(z|x_h, a_h), p_h(o_{h+1}|z)\}_{h\in[H]}\}$, Variational Distribution Class $\mathcal{Q} = \{\{q_h(z|x_h, a_h, o_{h+1})\}_{h\in[H]}\}$, Episode Number $K$.
    \State \textbf{Initialize} $\pi_0^h(s) = \mathcal{U}(\mathcal{A}), \forall h \in [H]$ where $\mathcal{U}(\mathcal{A})$ denotes the uniform distribution on $\mathcal{A}$; $\mathcal{D}_{0,h} = \emptyset, \mathcal{D}_{0, h}^\prime=\emptyset, \forall h\in [H]$.
    \For{episode $k=1,\cdots,K$}
    \State Initialize $\mathcal{D}_{k, h} = \mathcal{D}_{k-1, h}$, $\mathcal{D}_{k, h}^\prime = \mathcal{D}_{k-1, h}^\prime$
    \For{Step $h=1, \cdots, H$}
    \State Collect the transition $(x_h, a_h, o_{h+1}, a_{h+1}, \cdots, o_{h+L-1}, a_{h+L-1}, o_{h+L})$ where $x_h\sim d_{\mathcal{P}}^{\pi_k, h}$, $a_{h:h+L-1}\sim \mathcal{U}(\mathcal{A})$, $o_{h+i}\sim \mathbb{P}^{\mathcal{P}}(\cdot |x_{h+i-1}, a_{h+i-1}), \forall i\in [L]$. 
    \State $\mathcal{D}_{k, h} = \mathcal{D}_{k, h} \cup \{x_h, a_h, o_{h+1}\}$, $\mathcal{D}_{k, h + i}^\prime = \mathcal{D}_{k, h + i}^\prime \cup \{x_{h+i}, a_{h+i}, o_{h+i+1}\}$, $\forall i\in [L]$. \label{line:data_collection}
    \EndFor
    \State Learn the latent variable model $\hat{p}_k(z|x_h, a_h)$ with $\mathcal{D}_{k, h} \cup \mathcal{D}_{k, h}^\prime$ via maximizing the ELBO, and obtain the learned model $\widehat{\mathcal{P}}_k = \{(\hat{p}_{h, k}(z|x_h, a_h), \hat{p}_{h, k}(o_{h+1}|z))\}_{h\in[H]}$.\label{line:representation_online} 

    \State (Optional) Set the exploration bonus $\hat{b}_{k, h}(s,a)$. \label{line:bonus} %with $\mathcal{D}_{k, h}$.
    \State Update policy \label{line:online_plan}
    $ \pi_k=\mathop{\arg\max}_{\pi}V^{\pi}_{\widehat{\mathcal{P}}_k,r+\hat {b}_k}$.
    \EndFor
    \State \textbf{Return } $\pi^1,\cdots,\pi^K$.
\end{algorithmic}
\end{algorithm}

We have developed a linear representation for $Q_h^\pi$ that enables efficient planning. This section discusses how to learn and explore on top of this representation. The full algorithm is presented in Algorithm~\ref{alg:lvrep_pomdp}.

\vspace{-1em}
\paragraph{Variational Learning of~\algabb.} As we generally do not have the latent variable representation $p(\cdot|x_h, a_h)$ beforehand, it is essential to perform  representation learning with online collected data. 
One straightforward idea is to apply maximum likelihood estimation on $\PP^\pi \rbr{x_{h+k}|x_h, a_h}$. Although this is theoretically correct, 
due to the overlap between $x_{h+k}$ and $x_h$, a naive parametrization unnecessarily wastes memory and computational costs. Recall that we only need $p\rbr{z_h|x_h}$ to represent $Q_h^\pi$, and from 
% \begin{small}
\begin{align}\label{eq:obs_dynamics}
& \quad p\rbr{{o_{h+1:h+l}}|x_h, a_h} = \int_{\mathcal{Z}} p(z_h|x_h, a_h) \\
& \cdot\underbrace{\prod_{i=1}^l\left[\int_{\mathcal{Z}} \PP^\pi (z_{h+i}|z_{h+i-1}, a_i) p(o_{h+i}|z_{h+i}) d z_{h+i}\right]}_{\PP^\pi \rbr{{o_{h+1:h+l}}|z_h}} dz_h,\nonumber
\end{align}
% \end{small}
we can obtain $p(\cdot|x_h, a_h)$ by performing maximum likelihood estimation (MLE) on $p\rbr{{o_{h+1:h+l}}|x_h, a_h}$ for arbitrary $l\in \NN_+$. 
To obtain a tractable surrogate for the MLE of the latent variable model~\eqref{eq:obs_dynamics}, 
we exploit the evidence lower bound~(ELBO)~\citep{ren2023latent}:
\begin{align}\label{eq:elbo}
    & \quad \log p\rbr{{o_{h+1:h+l}}|x_h, a_h} \nonumber \\ 
    &= \log \int_\Zcal  p(z_h|x_h, a_h) \PP^\pi \rbr{{o_{h+1:h+l}}|z_h} \nonumber \\
    &= \resizebox{0.95\hsize}{!}{$\log \int_\Zcal\frac{p(z_h|x_h, a_h) \PP^\pi \rbr{{o_{h+1:h+l}}|z_h}}{q(z|x_{h}, a_h, o_{h+1:h+l})}q(z|x_{h}, a_h, o_{h+1:h+l})$} \nonumber \\
    &=\max_{q\in\Delta\rbr{\Zcal}}\EE_{q(\cdot|x_{h}, a_h, o_{h+1:h+l})}\sbr{\log  \PP^\pi \rbr{{o_{h+1:h+l}}|z_h}} \nonumber \\ 
    & \quad - D_{KL}\rbr{q(z|x_{h}, a_h, o_{h+1:h+l})|| p(z_h|x_h)},
\end{align}
where the last equation comes from Jensen's inequality, with equality holding when $q(z|x_{h}, a_h, o_{h+1:h+l}) \propto p(z_h|x_h, a_h) \PP^\pi \rbr{{o_{h+1:h+l}}|z_h}$. One can use~\eqref{eq:elbo} with data to fit the~\algabb. For the ease of the presentation, we choose $l=1$ in Algorithm~\ref{alg:lvrep_pomdp}. 

\vspace{-1em}
\paragraph{Practical Parametrization of $Q^\pi$ with~\algabb.} With~\algabb, we can represent $Q_h^\pi\rbr{x_h, a_h} = \inner{p(z|x_h)}{w_h^\pi(z)}_{L_2(\mu)}$. If the latent variable $z$ in $p(z|x_h)$ is an enumeratable discrete variable, $Q^\pi\rbr{x_h, a_h} = \sum_{i=m} w^\pi\rbr{z_i}p(z_i|x_h)$, can be simply represented. 

However, a discrete latent variable is not differentiable, which causes difficulty in learning. Therefore, we use a continuous latent variable $z$, which induces an infinite-dimensional $w(z)$.  To address this challenge, we follow the trick in LV-Rep~\citep{ren2023latent} that forms $Q^\pi\rbr{x_h, z_h}$ as an expectation: 
\[
Q^\pi\rbr{x_h, a_h} = \inner{p(z|x_h)}{w
^\pi(z)} = \EE_{p(z|x_h)}\sbr{w^\pi(z)}
\]
which can then be approximated by a Monte-Carlo method or random feature quadrature~\citep{ren2023latent}:
%respectively,
\begin{align*}%\label{eq:finite_q}
\resizebox{.95\hsize}{!}{
    $Q^\pi\rbr{x_h, a_h}\approx\frac{1}{m}\sum_{i=1}^m w^\pi(z_i)\,\, \text{or}\,\, \frac{1}{m}\sum_{i=1}^m \wtil^\pi(\xi_i)\varphi\rbr{z_i, \xi_i}$
    }
\end{align*}
with samples $z_i \sim p(z|x_h)$ and $\xi_i \sim P(\xi)$ as the random feature measure for the RKHS containing $w(z)$. Both approximations can be implemented by a neural network. Due to space limitations, we defer the derivation of the random feature quadrature to~\cref{appendix:tech_background}. 

\vspace{-1em}
\paragraph{Planning and Exploration with~\algabb.} Given an accurate estimator for $Q$ functions, we can perform planning with standard dynamic programming \citep[e.g.][]{munos2008finite}. However, dynamic programming involves an $\mathop{\arg\max}$ operation, which is only tractable when $|\mathcal{A}| < \infty$. To deal with  continuous actions, we leverage popular policy gradient methods like SAC~\citep{haarnoja2018soft}, with the critic parameterized by~\algabb.

To improve the exploration, we can leverage the idea of \citet{uehara2021representation, ren2023latent} and add an additional ellipsoid bonus to implement optimism in the face of uncertainty. Specifically, if we use random feature quadrature, we can compute such a bonus via:
\begin{align*}
    & \hat{\psi}_{h, k}(x_h, a_h) = \left[\varphi(z_i; \xi_i)\right]_{i\in [m]}, \quad \nonumber \\ 
    & \quad \text{where} \ \{z_i\}_{i\in [m]} \sim \hat{p}_{k, h}(z|x_h, a_h), \{\xi_i\}_{i\in [m]} \sim P(\xi), \nonumber\\
    & \hat{b}_{k, h}(s, a) = \alpha_k \hat{\psi}_{h, k}(x_h, a_h) \hat{\Sigma}_{k, h}^{-1} \hat{\psi}_{h, k}(x_h, a_h), 
\end{align*}
where $\hat{\Sigma}_{k,h} = \sum_{\mathcal{D}_{k, h}} \hat{\psi}_{k, h}(x_{h, i}, a_{h, i}) \hat{\psi}_{k, h}(x_{h, i}, a_{h, i})^\top + \lambda I$, and  $\alpha_k$, $\lambda$ are user-specified constants. {\color{black} Similarly, the bonus can be used to implement pessimism in the face of uncertainty for offline settings, which we defer to~\cref{appendix:offline}.} %, due to space limitation.}
%%%%%%%%%%%%%%%%%%%%%%%%%%%%%%%%%%%%%%%%%%%%%%%%%%%%%%%%%%%
%\vspace{-3mm}
\vspace{-8mm}
\section{Theoretical Analysis}\label{sec:theory}
\vspace{-2mm}
%%%%%%%%%%%%%%%%%%%%%%%%%%%%%%%%%%%%%%%%%%%%%%%%%%%%%%%%%%%
In this section, we provide a formal sample complexity analysis of the proposed algorithm. We start from the following assumptions, that are commonly used in the literature~\citep[e.g.][]{agarwal2020flambe, uehara2021representation,ren2023latent}.
\begin{assumption}[Finite Candidate Class with Realizability]
    \label{assumption:function_class}
    $|\mathcal{M}| < \infty$ and $\{(p_h^*(z|x_h, a_h), p_h^*(o_{h+1}|z))\}_{h\in[H]} \in \mathcal{M}$. 
    Meanwhile, for all $(p_h(z|x_h, a_h), p(o_{h+1}|z)) \in \mathcal{M}$, $p_h(z|x_h, a_h, o_{h+1}) \in \mathcal{Q}$. 
\end{assumption}

\begin{assumption}[Normalization Conditions]
    \label{assumption:normalization}
    $\forall \mathcal{P}\in \mathcal{M}, (x_h, a_h) \in \mathcal{X}\times\mathcal{A}, \|p_h(\cdot|x_h, a_h)\|_{\mathcal{H}_K} \leq 1$ for some kernel $K$. 
    Furthermore, $\forall g:\mathcal{X} \to \mathbb{R}$ such that $\|g\|_{\infty} \leq 1$, we have $\left\|\int_{\mathcal{X}} p(x_{h+L}|\cdot) g(x_{h+L}) d x_{h+L}\right\|_{\mathcal{H}_K} \leq C$.
\end{assumption}
Then, we have the following sample complexity bound for \algabb, with detailed proofs given in~\cref{appendix:analysis}.
\begin{theorem}[PAC Guarantee, Informal version of Theorem~\ref{thm:pac_guarantee_online}]
    \label{thm:pac_online_informal}
    Assume the kernel $K$ satisfies the regularity conditions in Appendix~\ref{sec:technical_conditions}. If we properly choose the exploration bonus $\hat{b}_k(x, a)$, we can obtain an $\varepsilon$-optimal policy with probability at least $1-\delta$ after we interact with the environments for $N = \mathrm{poly}\left(C, H, |\mathcal{A}|^L, L, \varepsilon, \log(|\mathcal{M}|/\delta)\right)$ episodes.
\end{theorem}

\vspace{-4mm}
\section{Related Work}\label{sec:related_work}
\vspace{-2mm}

Representation has been previously considered in partially observable reinforcement learning, but for different purposes. Vision-based representations~\citep{yarats2020image,seo2023masked} have been designed to extract compact features from raw pixel observations. We emphasize that this type of observation feature does not explicitly capture dynamics properties, and is essentially orthogonal to (but naturally compatible with) the proposed representation. 
Many dynamic-aware representation methods have been developed, such as bi-simulation~\citep{ferns2004metrics,gelada2019deepmdp,zhang2020learning}, successor features~\citep{dayan1993improving,barreto2017successor,kulkarni2016deep}, spectral representation~\citep{mahadevan2007proto,wu2018laplacian,duan2019state}, and contrastive representation~\citep{oord2018representation,nachum2021provable, yang2021trail}.
Our proposed representation for POMDPs has been inspired by  recent progress~\citep{jin2020provably, yang2020reinforcement,agarwal2020flambe,uehara2022provably} in revealing  low-rank structure in the transition kernel of MDPs, and inducing effective linear representations for the state-action value function for an arbitrary policy. These prior discoveries have led to a series of practical and provable RL algorithms in the MDP setting, achieving a delicate balance between learning, planning and exploration~\citep{ren2022free,zhang2022making,ren2023spectral, ren2023latent}. Although these algorithms demonstrate theoretical  and empirical benefits, they rely on the Markovian assumption, hence are not applicable to the POMDP setting we consider here. %partially observable RL settings.  

There have been several attempts to exploit the low-rank representation in POMDPs to reduce statistical complexity~\citep{efroni2022provable}. \citet{azizzadenesheli2016reinforcement} and \citet{guo2016pac} exploit spectral learning for model estimation without exploration;~\citet{jin2020sample} explores within an spectral estimation set; \citet{uehara2022provably} builds upon the Bellman error ball; 
\citet{zhan2022pac,liu2022partially} consider the MLE confidence set for low-rank structured models; and~\citep{huang2023provably} construct a UCB-type algorithm upon the MLE ball of PSR. However, these algorithms rely on  intractable oracles for planning, and fewer works consider exploiting  low-rank structure to achieve computationally tractable planning. 
One exception is~\citep{zhang2023energybased,guo2023provably}, which still includes intractable operations, \ie, infinite-dimensional operations or integrals.
\begin{figure*}[ht]
    % \vskip -0.1in
    \vspace{-3mm}
\begin{center}
    % \subfigure{\includegraphics[width=0.245\textwidth]{pic/metaworld_assembly_performance.pdf}}
    % \subfigure{\includegraphics[width=0.245\textwidth]{pic/metaworld_coffee_push_performance.pdf}}
    % \subfigure{\includegraphics[width=0.245\textwidth]{pic/metaworld_disassemble_performance.pdf}}
    % \subfigure{\includegraphics[width=0.245\textwidth]{pic/metaworld_peg_insert_side_performance.pdf}}
    % \vspace{-2mm}
    % \\
    % \subfigure{\includegraphics[width=0.245\textwidth]{pic/metaworld_pick_out_of_hole_performance.pdf}}
    % \subfigure{\includegraphics[width=0.245\textwidth]{pic/metaworld_push_performance.pdf}}
    % \subfigure{\includegraphics[width=0.245\textwidth]{pic/metaworld_soccer_performance.pdf}}
    % \subfigure{\includegraphics[width=0.245\textwidth]{pic/metaworld_sweep_into_performance.pdf}}
    % \vspace{-3mm}
    % \subfigure{\includegraphics[width=0.245\textwidth]{pic/metaworld_assembly_performance.pdf}}
    \subfigure{\includegraphics[width=0.245\textwidth]{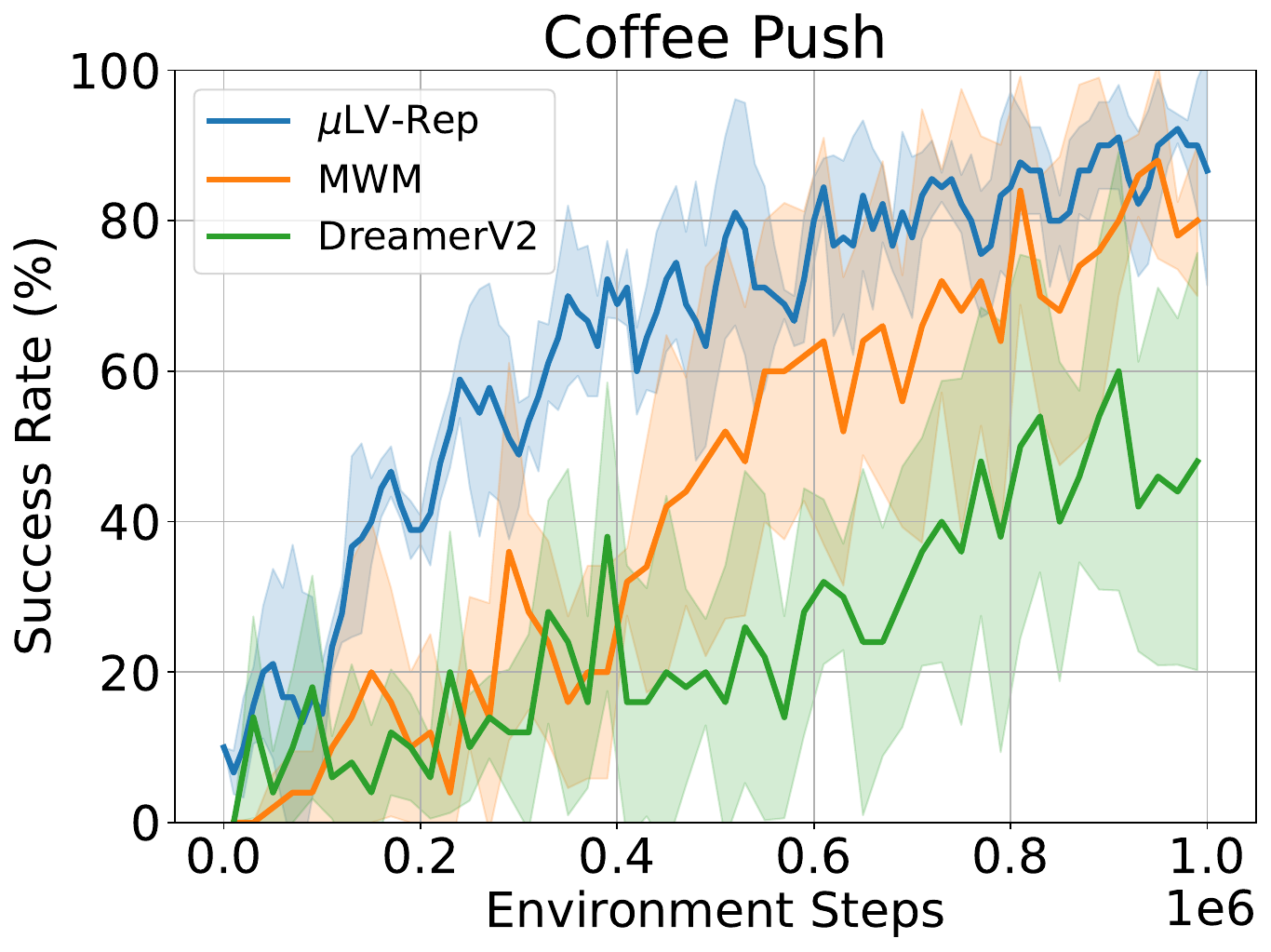}}
    % \subfigure{\includegraphics[width=0.245\textwidth]{pic/metaworld_disassemble_performance.pdf}}
    % \subfigure{\includegraphics[width=0.245\textwidth]{pic/metaworld_door_close_performance.pdf}}
    % \vspace{-2mm}
    % \\
    % \subfigure{\includegraphics[width=0.245\textwidth]{pic/metaworld_pick_out_of_hole_performance.pdf}}
    \subfigure{\includegraphics[width=0.245\textwidth]{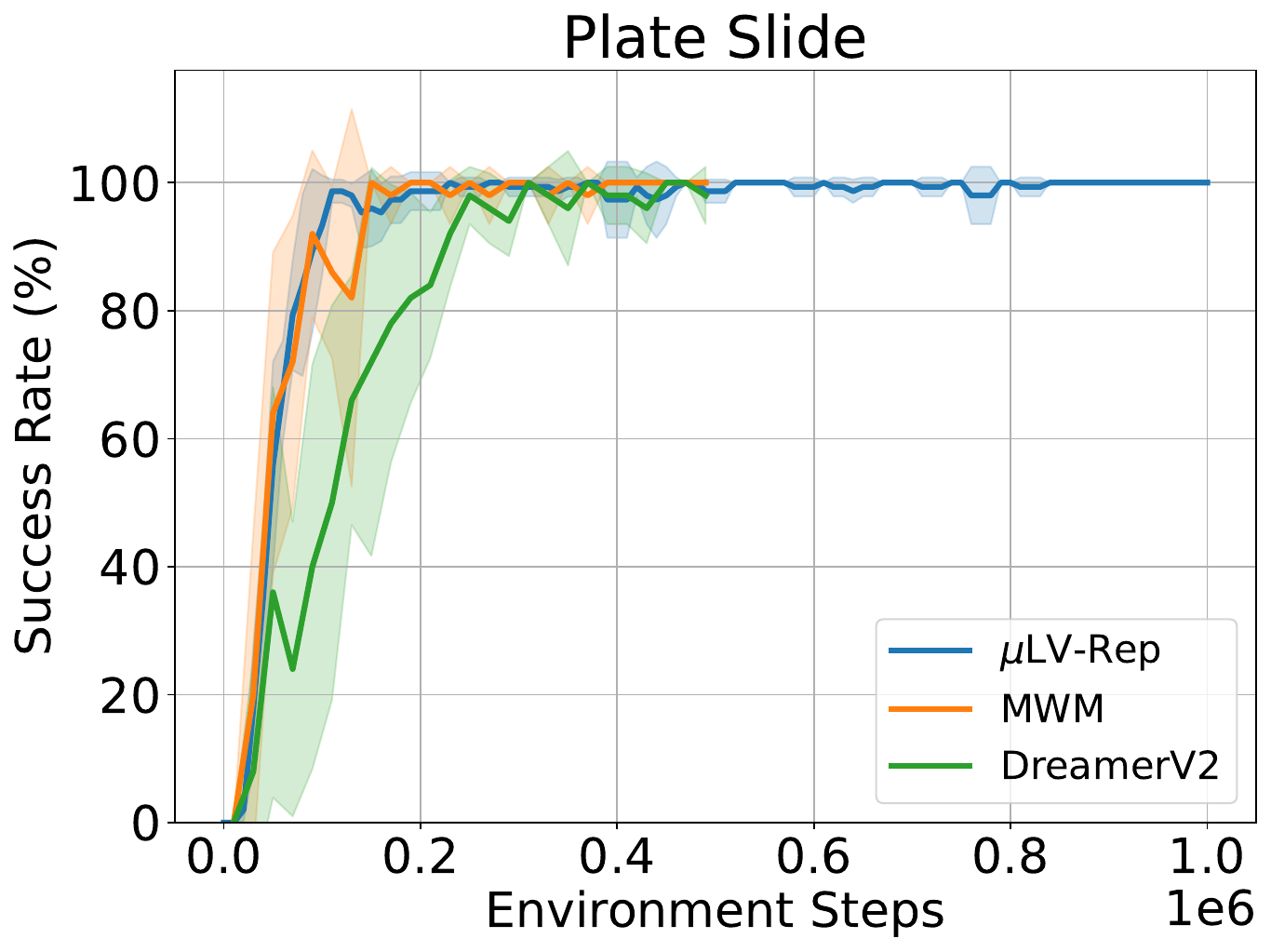}}
    \subfigure{\includegraphics[width=0.245\textwidth]{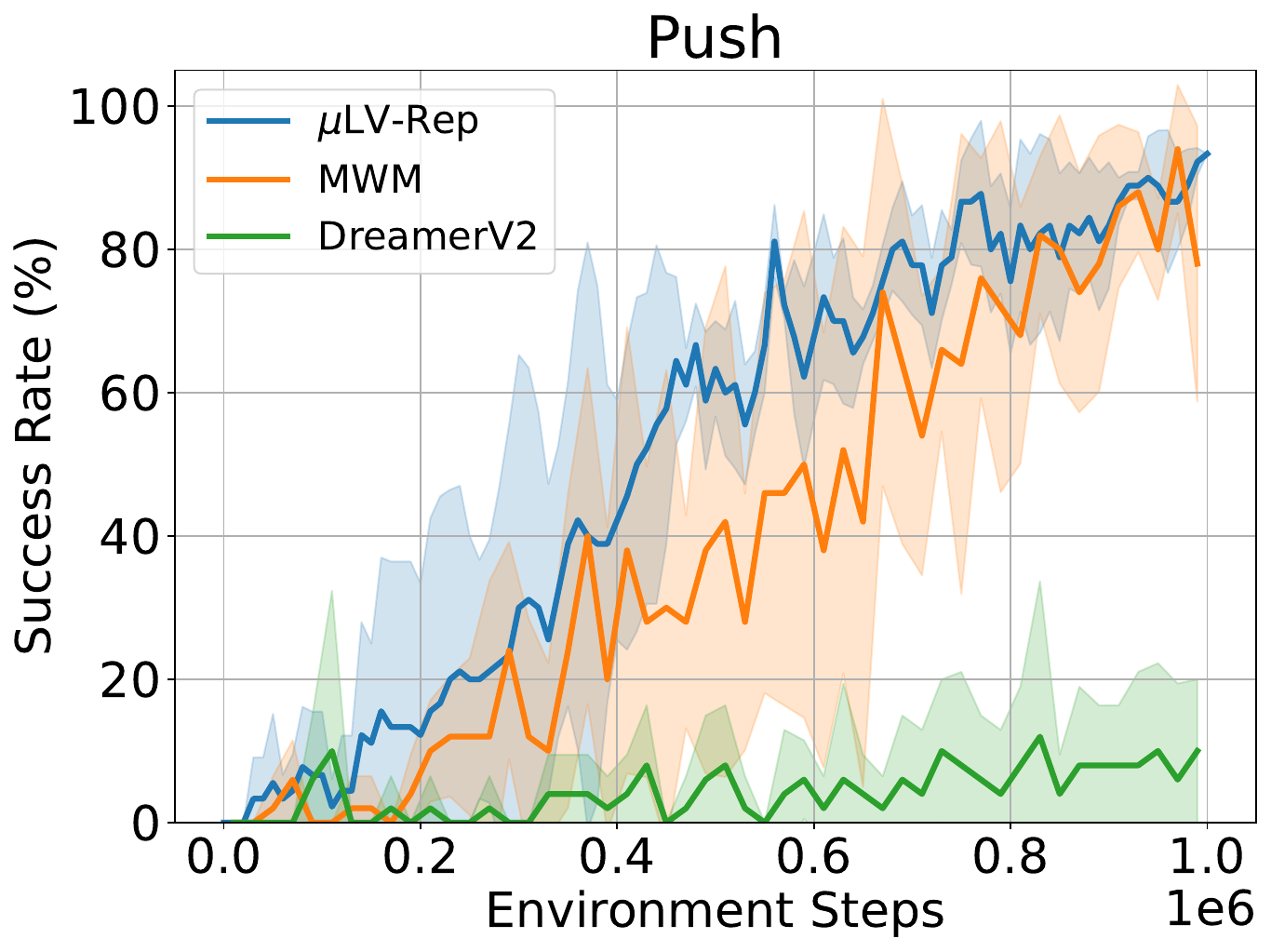}}
    \subfigure{\includegraphics[width=0.245\textwidth]{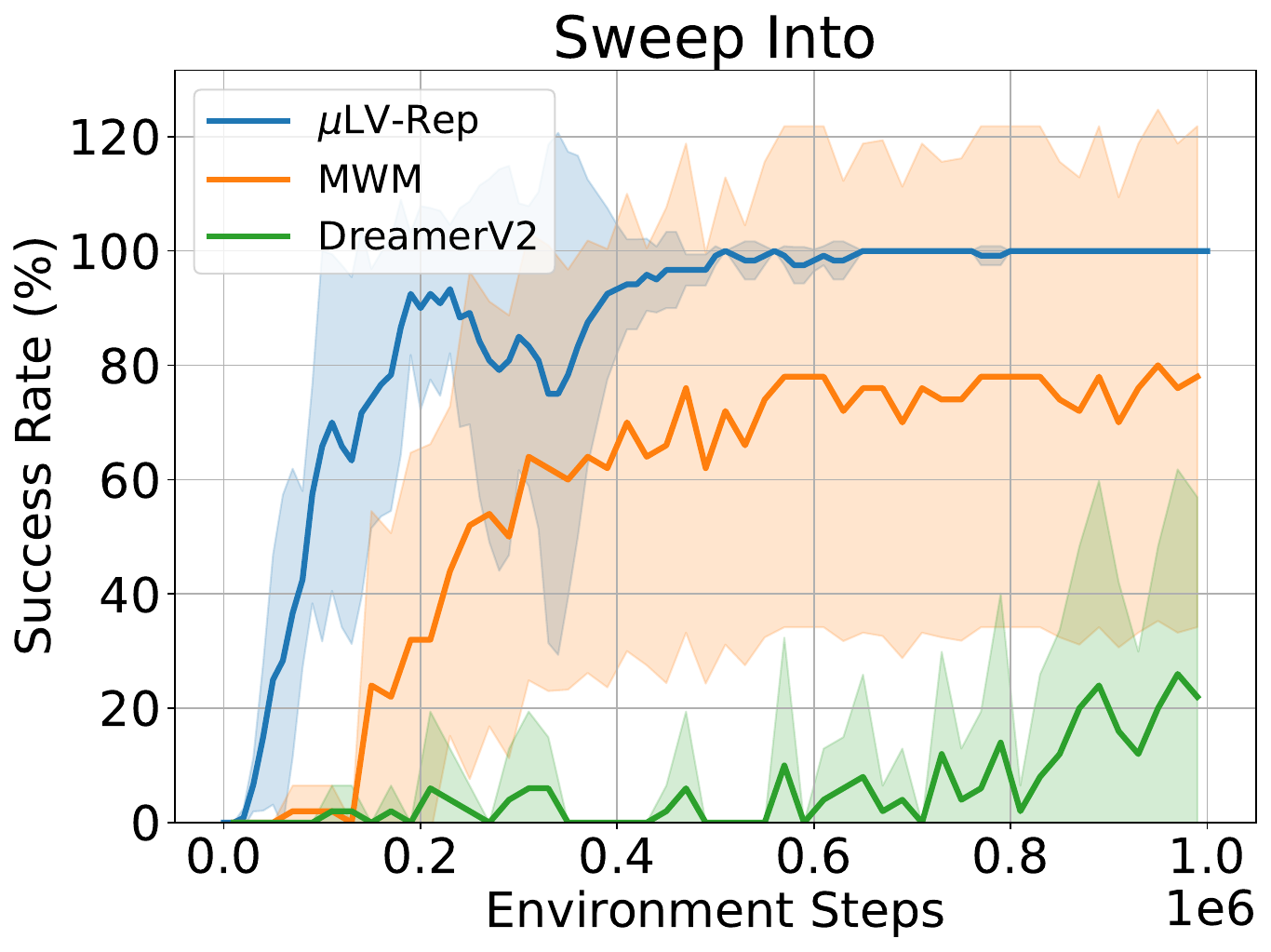}}
    \vspace{-4mm}
    \caption{
    Learning curves on visual robotic manipulation tasks from Meta-world measured by success rate. 
    % The solid line and shaded regions represent the mean and standard deviation across five random seeds. 
    Our method shows better or comparable sample efficiency compared to baseline methods. 
    Learning curves on all 50 tasks are reported in \cref{sec:imp detail}. 
    \normalsize}
    \label{fig:metaworld}
\end{center}
\vspace{-2mm}
\end{figure*}

\begin{figure*}[h]
    % \vskip -0.1in
    \vspace{-2mm}
\begin{center}
    \subfigure{\includegraphics[width=1\textwidth]{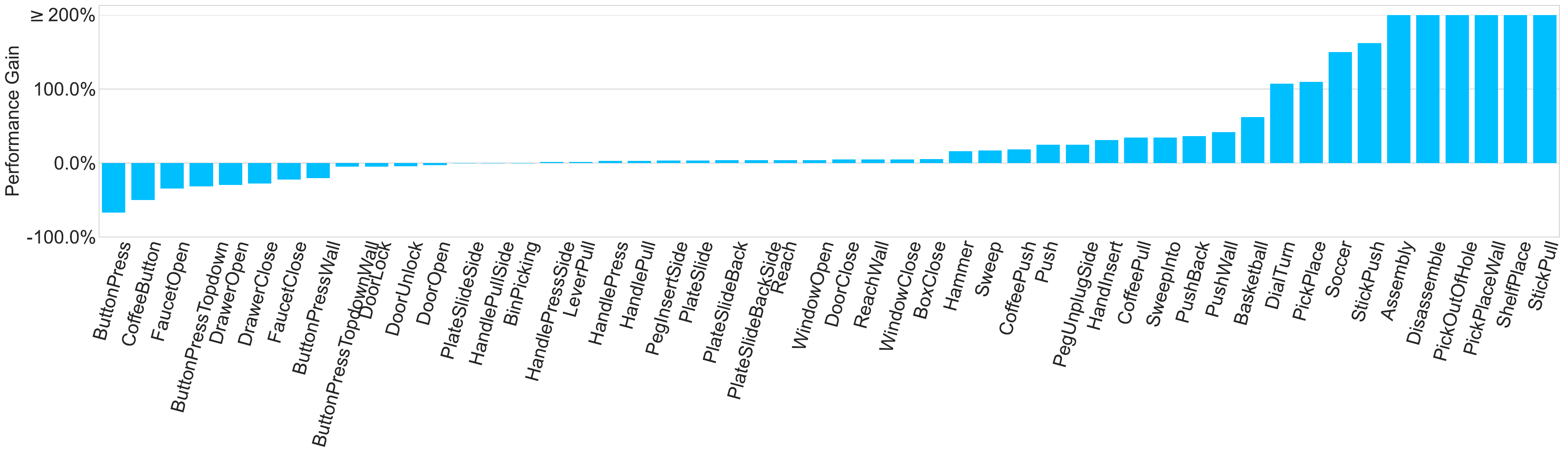}}
    \vspace{-10mm}
    \caption{
    The performance gain on 50 Meta-world tasks after 1 million interactions. 
    % The blue bars represent the percentage improvement in comparison to the highest performance among the baselines. 
    Our results surpass or are comparable to (with a difference of less than or equal to 10\%) the best baselines on 41 out of the 50 tasks.
    % Success Rate on 50 Meta-world tasks after 1 million interactions. The blue bars show our results, the gray bars show the best performance among baselines. 
    % Our method succeeds on most of tasks (> 90\% on 33 tasks). The results are better than or equivalent to (within 10\% difference) the best baselines on 41 tasks.
    \normalsize}
    \label{fig:metaworld final_performance}
\end{center}
\vspace{-6mm}
\end{figure*}

%%%%%%%%%%%%%%%%%%%%%%%%%%%%%%%%%%%%%%%%%%%%%%%%%%
\vspace{-3mm}
\section{Experimental Evaluation}\label{sec:exp}
\vspace{-2mm}
%%%%%%%%%%%%%%%%%%%%%%%%%%%%%%%%%%%%%%%%%%%%%%%%%%

We evaluate the proposed method on Meta-world~\citep{yu2019meta}, which is an open-source simulated benchmark consisting of 50 distinct robotic manipulation tasks with visual observations. 
We also provide experiment results on partial observable control problems constructed based on OpenAI gym MuJoCo~\citep{todorov2012mujoco} in \cref{sec:extra-experiments}. 
Comparison on classic POMDP benchmarks, such as the RockSample problem, is worth further study.~\citep{smith2012heuristic,kurniawati2009sarsop}.

We notice that directly acquiring a robust control representation through predicting future visual observations can be challenging due to the redundancy of information in images for effective decision-making. As a result, a more advantageous approach involves obtaining an image representation first and subsequently learning a latent representation based on this initial image representation.
In particular, we employ visual observations with dimensions of 64 × 64 × 3 and apply a variational autoendocer (VAE)~\citep{kingma2013auto} to learn a representation of these visual observations.  
The VAE is 
%first pre-trained with random trajectories  then 
trained during the online learning procedure. 
This produces compact vector representations for the images, which are then forwarded as input to the representation learning method. 
As detailed in \cref{subsec:algorithm}, we learn the latent representations by making predictions about future outcomes using a history of length $L$.  
To achieve this, we employ a continuous latent variable model similar to \citep{ren2023latent,Hafner2020Dream}, approximating distributions with Gaussians parameterized by their mean and variance.  
We then apply Soft Actor-Critic (SAC) as the planner \citep{haarnoja2018soft}, which takes the learned representation as input of the critic.
We apply $L=3$ across all domains (with an ablation study  provided in \cref{appendix:ablation}). 
More implementation details, including network architectures and hyper-parameters, are provided in \cref{sec:imp detail}. 

We consider two baseline methods, DreamerV2~\citep{hafner2021mastering} and the Masked World Model (MWM)~\citep{seo2023masked}. 
DreamerV2 is designed to acquire latent representations, which are subsequently input into a recurrent neural network for environment modeling. 
MWM utilizes an autoencoder equipped with convolutional layers and vision transformers (ViT) rather than reconstructing visual observations. This autoencoder is employed to reconstruct pixels based on masked convolutional features, allowing MWM to learn a latent dynamics model that operates on the representations derived from the autoencoder.

\cref{fig:metaworld} presents the learning curves of all algorithms as measured
on the success rate for 1M environment steps, averaged over 5 random seeds. 
%The solid line and shaded regions represent the mean and standard deviation across 5 random seeds. 
We observe that \algabb exhibits superior performance compared to both DreamerV2 and MWM, demonstrating faster convergence and achieving higher final performance across all four tasks reported. We also provide the learning curves for all 50 Meta-world tasks in \cref{sec:extra-experiments}. 
%In particular, \algabb achieves near-perfect success rate after 500k steps while all baseline methods struggle even for 1 million steps. 
%Additionally, 
It is noteworthy that while MWM achieves comparable sample efficiency with \algabb in certain tasks, it incurs higher computational costs and longer running times due to the incorporation of ViT network in its model. In our experimental configurations, \algabb demonstrates a training speed of 21.3 steps per second, outperforming MWM, which achieves a lower training speed of 8.1 steps per second. This highlights the computational efficiency of the proposed method. 
\cref{fig: metaworld performance} illustrates the final performance of \algabb across 50 Meta-world tasks (blue bars), comparing it with the best performance achieved by DreamerV2 and MWM (gray bars). We observe that \algabb achieves better than 90\% success on 33 tasks. It also performs better or equivalent to (within 10\% difference) the best of DreamerV2 and MWM on 41 tasks.

%%%%%%%%%%%%%%%%%%%%%%%%%%%%%%%%%%%%%%%%%%%%%%%%%%%%%%%%%%%%%%%%%%%%%%%%
\vspace{-2mm}
\section{Conclusion}\label{sec:conclusion}
\vspace{-1mm}
%%%%%%%%%%%%%%%%%%%%%%%%%%%%%%%%%%%%%%%%%%%%%%%%%%%%%%%%%%%%%%%%%%%%%%%%

% The structured assumptions for POMDPs have been introduced for algorithms design with efficient statistical complexity. However, most of them still rely on some ideal computational oracle. 
In this paper, we aimed to develop a practical RL algorithm for structured POMDPs that obtained efficiency in terms of both statistical and computational complexity. 
We revealed some of the challenges in computationally exploiting the low-rank structure of a POMDP,
%s for computational tractability. Then, we derive 
then derived
a linear representation for the $Q^\pi$-function, which automatically implies a practical learning method, with tractable planning and exploration, as in~\algabb. We theoretically analyzed the sub-optimality of the proposed~\algabb, and empirically demonstrated its advantages on several benchmarks.

\section*{Impact Statement}
This paper presents work whose goal is to advance the field of Machine Learning. There are many potential societal consequences of our work, none which we feel must be specifically highlighted here.

% \clearpage
\bibliography{example_paper}
\bibliographystyle{icml2024}

%%%%%%%%%%%%%%%%%%%%%%%%%%%%%%%%%%%%%%%%%%%%%%%%%%%%%%%%%%%%%%%%%%%%%%%%%%%%%%%
%%%%%%%%%%%%%%%%%%%%%%%%%%%%%%%%%%%%%%%%%%%%%%%%%%%%%%%%%%%%%%%%%%%%%%%%%%%%%%%
% APPENDIX
%%%%%%%%%%%%%%%%%%%%%%%%%%%%%%%%%%%%%%%%%%%%%%%%%%%%%%%%%%%%%%%%%%%%%%%%%%%%%%%
%%%%%%%%%%%%%%%%%%%%%%%%%%%%%%%%%%%%%%%%%%%%%%%%%%%%%%%%%%%%%%%%%%%%%%%%%%%%%%%
\newpage
\appendix
\onecolumn

\section{More Related Work}\label{appendix:more_related}

\paragraph{Partially Observable RL.} 
The majority of existing practical RL algorithms for partially observable settings can also be categorized into model-based \textit{vs}. model-free.

The model-based algorithms~\citep{kaelbling1998planning} for partially observed scenarios are naturally derived based on the definition of POMDPs, where both the emission and transition models are learned from data. The planning procedure for optimal policy is conducted over the posterior of latent state, \ie, beliefs, which is approximately inferred based on learned dynamics and emission model. With different model parametrizations, (ranging from Gaussian processes to deep models), and different planning methods, a family of algorithms has been proposed~\citep{deisenroth2012solving,igl2018deep, gregor2019shaping, zhang2019solar, lee2020stochastic, hafner2021mastering}. 
However, due to the compounding errors from {\bf i)}, mismatch in model parametrization, {\bf ii)}, inaccurate beliefs calculation,  {\bf iii)}, approximation in planning over nonlinear dynamics, and {\bf iv)}, neglecting of exploration, such methods might suffer from sub-optimal performances in practice. 

As we discussed in~\cref{sec:intro}, the memory-based policy and value function have been exploited to extend the MDP-based model-free RL algorithms to handle the non-Markovian dependency induced by partial observations. For example, the value-based algorithms introduces memory-based neural networks to Bellman recursion, including temporal difference learning with explicit concatenation of 4 consecutive frames as input~\citep{mnih2013playing} or recurrent neural networks for longer windows~\citep{bakker2001reinforcement,hausknecht2015deep,zhu2017improving}, and DICE~\citep{nachum2020reinforcement} with features extracted from transformer~\citep{jiang2021towards}; the policy gradient-based algorithms have been extended to partially observable setting by introducing recurrent neural network for policy parametrization~\citep{schmidhuber1990reinforcement,wierstra2007solving,heess2015memory,ni2021recurrent}. The actor-critic approaches exploits memory-based value and policy together~\citep{ni2021recurrent,meng2021memory}.
Despite their simplicity in the algorithm extension, these algorithms demonstrate potentials in real-world applications. However, the it has been observed that the sample complexity for purely model-free RL with partial observations is very high~\citep{mnih2013playing,barth2018distributed,yarats2021improving}, and the exploration remains difficult, and thus, largely neglected. 

\section{Observability Approximation}\label{appendix:observability}

Although the proposed~\algabb is designed based on the $L$-step decodability in POMDPs,~\citet{golowich2022learning} shows that the $\gamma$-observable POMDPs can be $\epsilon$-approximated with a $L = \Otil\rbr{\gamma^{-4}\log\rbr{{\abr{S}}/{\epsilon}}}$-step decodable POMDP. By exploiting the low-rank structure in the latent dynamics, this result has been extend with function approximator~\citep{uehara2022provably}. Specifically, 
\begin{theorem}[Proprosition 7~\citep{guo2023provably}, Lemma 12~\citep{uehara2022provably}]\label{thm:obs_approx}
    Given a $\gamma$-observable POMDP with $d$-rank latent transition, there exists an $L$-step decodable POMDP $\Mcal$ with $L=\Otil\rbr{\gamma^{-4}\log\rbr{d/\epsilon}}$, $\forall \epsilon>0$, such that
    \begin{equation}
        \EE_{a_{1:h}, o_{2:h}\sim \pi}\sbr{\nbr{\PP_h\rbr{o_{h+1}|o_{1:h}, a_{1:h}} - \PP_h^\Mcal\rbr{o_{h+1}|x_h, a_h}}_1}\le \epsilon. 
    \end{equation}
    where $\pi_h\in \Delta\rbr{\prod_{h=1}^H \Acal^{\mathscr{H}_h}}$ with $\Hscr_h\defeq \Acal^{h-1}\times \Ocal^{h}$, is mapping the whole history to a distribution of action. 
\end{theorem}

With this understanding, the proposed~\algabb can be directly applied for $\gamma$-observable POMDPs, while still maintains theoretical guarantees. Due to the space limitation, please refer to~\cite{uehara2022provably, guo2023provably} for the details of the proofs. 
\section{Moment Matching Policy}\label{appendix:moment_matching}
We provide a formal definition of the moment matching policy here.
\begin{definition}[Moment Matching Policy \citep{efroni2022provable}]
    With the $L$-decodability assumption,
    for $h\in [H]$, $h^\prime \in [h-L+1, h]$ and $l = h^\prime - h + L - 1$, we can define the moment matching policy $\nu^{\pi, h} = \{\nu_{h^\prime}^{h, \pi} : \mathcal{S}^l \times \mathcal{O}^l \times \mathcal{A}^{l-1} \to \Delta(\mathcal{A})\}_{h^\prime = h-L + 1}^h$ introduced by~\citet{efroni2022provable} , such that
    \begin{align*}
        & \nu_{h^\prime}^{\pi, h}(a_{h^\prime} | (s_{h-L+1:h^\prime}, o_{h-L+1:h^\prime}, a_{h-L+1: h^\prime - 1}))\\
        & \qquad :=  \mathbb{E}_{\pi}^{\mathcal{P}} [\pi_{h^\prime}(a_{h^\prime}|x_{h^\prime})|(s_{h-L+1:h^\prime}, o_{h-L+1:h^\prime}, a_{h-L+1: h^\prime - 1})], \quad \forall h^\prime \leq h-1,
    \end{align*}
    and $\nu_h^{\pi, h} = \pi_h$. We further define $\tilde{\pi}^h$, which takes first $h-L$ actions from $\pi$ and the remaining $L$ actions from $\nu^{\pi, h}$. 
\end{definition}
The main motivation to define such moment matching policy is that, we want to define a policy that is conditionally independent from the past history for theoretical justification while indistinguishable from the history dependent policy to match the practical algorithm.
By Lemma B.2 in \citet{efroni2022provable}, under the $L$-decodability assumption, for a fixed $h\in [H]$, we have  $d_{h}^{\mathcal{P}, \pi}(x_h) = d_{h}^{\mathcal{P}, \tilde{\pi}_h}(x_h)$, for all $L$-step policy $\pi$ and $x_h \in \mathcal{X}_h$. As $\nu^{\pi, h}_h = \pi_h$.  we have $d_h^{\mathcal{P}, \pi}(x_h, a_h) = d_h^{\mathcal{P}, \tilde{\pi}^h}(x_h, a_h)$, and hence $\mathbb{E}_{\pi}^{\mathcal{P}}(x_h, a_h) = \mathbb{E}_{\tilde{\pi}^h}^{\mathcal{P}}(x_h, a_h)$. This enables the factorization in~\eqref{eq:linear_second} without the dependency of the overlap observation trajectory. 
\section{Pessimism in the Offline Setting}\label{appendix:offline}
Similar to \citet{uehara2021representation, ren2023latent}, the proposed algorithm can be directly extended to the offline setting by converting the optimism into the pessimism. Specifically, we can learn the latent variable model, set the penalty with the data and perform planning with the penalized reward. The whole algorithm is shown in Algorithm \ref{alg:lvrep_pomdp_offline}. Following the identical proof strategy from \citet{uehara2021representation, ren2023latent}, we can obtain a similar sub-optimal gap guarantee for $\hat{\pi}^*$. 
\begin{algorithm}[t] 
\caption{Offline Learning for $L$-step decodable POMDPs 
with Latent Variable Representation} \label{alg:lvrep_pomdp_offline}
\begin{algorithmic}[1]
    \State \textbf{Input:} Model Class $\mathcal{M}=\{\{(p_h(z|x_h, a_h), p_h(o_{h+1}|z)\}_{h\in[H]}\}$, Variational Distribution Class $\mathcal{Q} = \{\{q_h(z|x_h, a_h, o_{h+1})\}_{h\in[H]}\}$, Offline Dataset $\{\mathcal{D}_h\}_{h=1}^{H}$
    \State Learn the latent variable model $\hat{p}(z|x_h, a_h)$ with $\mathcal{D}_h$ via maximizing the ELBO, and obtain the learned model $\widehat{\mathcal{P}} = \{(\hat{p}_{h}(z|x_h, a_h), \hat{p}_{h}(o_{h+1}|z))\}_{h\in[H]}$.\label{line:representation_offline} 

    \State Set the exploitation penalty $\hat{b}_{h}(s,a)$ with $\mathcal{D}_{k}$.\label{line:penalty}
    \State Learn the policy \label{line:offline_plan}
    $ \hat{\pi}^*=\mathop{\arg\max}_{\pi}V^{\pi}_{\widehat{\mathcal{P}},r-\hat {b}_k}$.
    \State \textbf{Return } $\hat{\pi}^*$.
\end{algorithmic}
\end{algorithm}
\section{Technical Background}\label{appendix:tech_background}
In this section, we revisit several core concepts of the kernel and the reproducing kernel Hilbert space (RKHS) that will be used in the theoretical analysis. For a complete introduction, we refer the reader to \citet{ren2023latent}.

\begin{definition}[Kernel and Reproducing Kernel Hilbert Space (RKHS) \citep{aronszajn1950theory,paulsen2016introduction}]
    The function $k:\mathcal{X} \times \mathcal{X} \to \mathbb{R}$ is called a kernel on $\mathcal{X}$ if there exists a Hilbert space $\mathcal{H}$ and a mapping $\phi:\mathcal{X} \to \mathcal{H}$ (termed as a feature map), such that $\forall x, x^\prime\in\mathcal{X}$, $k(x, x^\prime) = \langle \phi(x), \phi(x^\prime) \rangle_{\mathcal{H}}$. The kernel $k$ is said to be positive semi-definite if $\forall n \geq 1, \{a_i\}_{i\in [n]}\subset \mathbb{R}$ and mutually distinct $\{x_i\}_{i\in [n]}$, we have
    \begin{align*}
        \sum_{i\in [n]} \sum_{j\in [n]} a_i a_j k(x_i, x_j) \geq 0.
    \end{align*}
    The kernel $k$ is said to be positive definite if the inequality is strict (which means we can replace $\geq$ with $>$).

    With a given kernel $k$, we can define the Hilbert space $\mathcal{H}_k$ consists of $\mathbb{R}$-valued function on $\mathcal{X}$ as a reproducing kernel Hilbert space associated with $k$ if both of the following conditions hold:
    \begin{itemize}
        \item $\forall x\in\mathcal{X}$, $k(x, \cdot) \in \mathcal{H}_k$.
        \item Reproducing Property: $\forall x\in \mathcal{X}, f\in \mathcal{H}_k, f(x) = \langle f, k(x, \cdot)\rangle_{\mathcal{H}_k}$. 
    \end{itemize}
    The RKHS norm of $f\in\mathcal{H}_k$ is induced by the inner product, i.e. $\|f\|_{\mathcal{H}_k} := \sqrt{\langle f, f \rangle_{\mathcal{H}_k}}$. 
\end{definition}

\begin{theorem}[Mercer's Theorem \citep{riesz2012functional, steinwart2012mercer}]
    \label{thm:mercer}
    Let $k$ be a continuous positive definite kernel defined on $\mathcal{X} \times \mathcal{X}$. There exists at most countable $\{\mu_i\}_{i\in I}$ such that $\mu_1 \geq \mu_2 \geq \cdots > 0$ and a set of orthonormal basis $\{e_i\}_{i\in I}$ on $L_2(\mu)$ where $\mu$ is a Borel measure on $\mathcal{X}$, such that
    \begin{align*}
        \forall x, x^\prime \in \mathcal{X}, \quad k(x, x^\prime) = \sum_{i\in I} \mu_i e_i(x) e_i(x^\prime),
    \end{align*}
    where the convergence is absolute and uniform. 
\end{theorem}

\begin{definition}[Random Feature]
    \label{def:random_feature_representation}
    The kernel $k:\mathcal{X} \times \mathcal{X} \to \mathbb{R}$ has a random feature representation if there exists a function $\psi:\mathcal{X} \times \Xi \to \mathbb{R}$ and a probability measure $P$ over $\Xi$ such that
    \begin{align*}
        k(x, x^\prime) = \int_{\Xi} \psi(x; \xi) \psi(x^\prime; \xi) dP(\xi).
    \end{align*}
\end{definition}

\paragraph{Remark (random feature quadrature):} We here justify the random feature quadrature~\citep{ren2023latent} for completeness. 

We can represent $Q_h^\pi$ as an expectation, 
\[
Q_h^\pi\rbr{x_h, a_h} = \inner{p(z|x_h)}{w_h^\pi(z)} = \EE_{p(z|x_h)}\sbr{w_h^\pi(z)}_{L_2(\mu)}
\]
Under the assumption that $w_h^\pi(\cdot)\in \Hcal_k$, where $\Hcal_k$ denoting some RKHS with some kernel $k\rbr{\cdot, \cdot}$. When $k\rbr{\cdot, \cdot}$ can be represented through random feature, \ie, 
\[
k\rbr{x, y} = \EE_{P(\xi)}\sbr{\psi\rbr{x; \xi}\psi\rbr{y; \xi}},
\]
the $w_h^\pi\rbr{z}$ admits a representation as
\[
w_h^\pi\rbr{z} = \EE_{P(\xi)}\sbr{\wtil^\pi_h\rbr{\xi}\psi\rbr{z; \xi}}. 
\]
Therefore, we plug this random feature representation of $w_h^\pi(z)$ to $Q_h^\pi\rbr{x_h, a_h}$, we obtain
\begin{align}\label{eq:random_feature}
    Q_h^\pi\rbr{x_h, a_h} = \EE_{p(z|x_h), P(\xi)}\sbr{\wtil_h^\pi(\xi)\psi\rbr{z; \xi}}.
\end{align}
Applying Monte-Carlo approximation to~\eqref{eq:random_feature}, we obtain the random feature quadrature.
% in~\eqref{eq:finite_q}.

\section{Theoretical Analysis}\label{appendix:analysis}
\subsection{Technical Conditions}
\label{sec:technical_conditions}
We adopt the following assumptions for the reproducing kernel, which have been used in \citet{ren2023latent} for the MDP setting.
\begin{assumption}[Regularity Conditions]
    \label{assump:trace}
    $\mathcal{Z}$ is a compact metric space with respect to the Lebesgue measure $\nu$ when $\mathcal{Z}$ is continuous. Furthermore, $\int_{\mathcal{Z}} k(z, z) d \nu \leq 1$.
\end{assumption}

\begin{assumption}[Eigendecay Conditions] Assume $\{\nu_i\}_{i\in I}$ defined in Theorem~\ref{thm:mercer} satisfies one of the following conditions:
\begin{itemize}[leftmargin=20pt, parsep=0pt, partopsep=0pt]
\item $\beta$-finite spectrum: for some positive integer $\beta$, we have $\nu_i = 0$, $\forall i > \beta$.
\item $\beta$-polynomial decay: $\nu_i \leq C_0 i^{-\beta}$ with absolute constant $C_0$ and $\beta > 1$.
\item $\beta$-exponential decay: $\nu_i \leq C_1 \exp(-C_2 i^\beta)$, with absolute constants $C_1$, $C_2$ and $\beta > 0$.
\end{itemize}
We will use $C_{\mathrm{poly}}$ to denote constants in the analysis of $\beta$-polynomial decay that only depends on $C_0$ and $\beta$, and $C_{\mathrm{exp}}$ to denote constants in the analysis of $\beta$-exponential decay that only depends on $C_1$, $C_2$ and $\beta$, to simplify the dependency of the constant terms. Both of $C_{\mathrm{poly}}$ and $C_{\mathrm{exp}}$ can be varied step by step.
\end{assumption}

\subsection{Formal Proof}
Before we proceed, we first define
\begin{align*}
    \rho_{k, h} = \frac{1}{k} \sum_{i\in[K]} d_{\mathcal{P}, h}^{\pi_k},
\end{align*}
and $\circ^L \mathcal{U}(\mathcal{A})$ means uniformly taking actions in the consecutive $L$ steps.
\begin{lemma}[$L$-step back inequality for the true model]
\label{lem:back_true}
    Given a set of functions $\left[g_h\right]_{h\in [H]}$, where $g_h:\mathcal{X}\times \mathcal{A} \to \mathbb{R}$, $\|g_h\|_{\infty} \leq B$, $\forall h\in [H]$, we have that $\forall \pi$, 
    \begin{align*}
        \sum_{h\in[H]} \mathbb{E}_{\pi}^{\mathcal{P}}[g(x_h, a_h)] \leq & \sum_{h\in [H]} \mathbb{E}_{(x_{h-L}, a_{h-L})\sim d_{\mathcal{P}, h-L}^\pi}^{\mathcal{P}}\left[\left\|p^*(\cdot|x_{h-L}, a_{h-L})\right\|_{L_2(\mu), \Sigma_{\rho_{k, h-L}, p^*}^{-1}}\right]\\
        & \cdot \sqrt{k|\mathcal{A}|^L \cdot \mathbb{E}_{(\tilde{x}_h, \tilde{a}_h) \sim \rho_{k, h-L} \circ^L \mathcal{U}(\mathcal{A})}\left[g(\tilde{x}_h, \tilde{a}_h)^2\right] + \lambda B^2  C}
    \end{align*}
\end{lemma}
\begin{proof}
    The proof can be adapted from the proof of Lemma 6 in \citet{ren2023latent}, and we include it for the completeness. Recall the moment matching policy $\nu^{\pi}$ Since $\nu^{\pi, h}$ does not depend on $(x_{h-L}, a_{h-L})$, we can make the following decomposition:
    \begin{align*}
        &\mathbb{E}_{\tilde{\pi}^h}^{\mathcal{P}}(x_h, a_h) \\
        = & \mathbb{E}_{(x_{h-L}, a_{h-L}) \sim \pi}^{\mathcal{P}}\left[\int_{s_{h-L+1}} \langle p^*(\cdot|x_{h-L}, a_{h-L}),  p^*(s_{h-L+1}|\cdot) \rangle_{L_2(\mu)} \cdot \mathbb{E}_{a_{h-L+1:h \sim \nu^{\pi, h}}}^{\mathcal{P}}[g(x_h, a_h)|s_{h - L + 1}] d s_{h-L+1}\right]\\
        \leq & \mathbb{E}_{(x_{h-L}, a_{h-L}) \sim \pi}^{\mathcal{P}}\left\|p^*(\cdot|x_{h-L}, a_{h-L})\right\|_{L_2(\mu), \Sigma_{\rho_{k, h-L}, p^*}^{-1}}\\
        & \cdot \left\|\int_{s_{h-L+1}} p^*(s_{h-L+1}|\cdot) \mathbb{E}[g(x_h, a_h)|s_{h-L+1}, \nu^{\pi, h}] d s_{h-L+1}\right\|_{L_2(\mu), \Sigma_{\rho_{k, h-L}, p^*}}. 
    \end{align*}
    Direct computation shows that
    \begin{align*}
        & \left\|\int_{s_{h-L+1}} p^*(s_{h-L+1}|\cdot) \mathbb{E}^{\mathcal{P}}[g(x_h, a_h)|s_{h-L+1}, \nu^{\pi, h}] d s_{h-L+1}\right\|_{L_2(\mu), \Sigma_{\rho_{k, h-L}, p^*}}\\
        = & k \mathbb{E}_{(\tilde{x}_{h-L}, \tilde{a}_{h-L} )\sim \rho_{k, h-L}} \left[\mathbb{E}_{s_{h-L+1}\sim \mathbb{P}_{h-L}^\mathcal{P}(\cdot|x_{h-L}, a_{h-L})}^{\mathcal{P}}[g(x_h, a_h)|s_{h-L+1}, \nu^{\pi, h}]\right]^2 \\
        & + \left\|\int_{s_{h-L+1}} p^*(s_{h-L+1}|\cdot) \cdot \mathbb{E}^{\mathcal{P}}[g(x_h, a_h)|s_{h-L+1}, \nu^{\pi, h}] d s_{h-L+1}\right\|_{\mathcal{H}}^2\\
        \leq & k \mathbb{E}_{(\tilde{x}_{h-L}, \tilde{a}_{h-L}) \sim \rho_{k, h-L}} \mathbb{E}^{\mathcal{P}}_{s_{h-L+1}\sim \mathbb{P}_{h-L}^\mathcal{P}(\cdot|x_{h-L}, a_{h-L}), a_{h-L+1:h} \sim \nu^{\pi, h}} \left[g(x_h, a_h)\right]^2 + \lambda B^2 C \\
        \leq & k |\mathcal{A}|^L \mathbb{E}^{\mathcal{P}}_{(\tilde{x}_h, \tilde{a}_h) \sim \rho_{k, h-L} \circ^L \mathcal{U}(\mathcal{A})} [g(\tilde{x}_h, \tilde{a}_h)]^2 + \lambda B^2 C,
    \end{align*}
    which finishes the proof.
\end{proof}
\begin{lemma}[$L$-step back inequality for the learned model]
\label{lem:back_learned}
    Assume we have a set of functions $\left[g_h\right]_{h\in [H]}$, where $g_h:\mathcal{X}\times \mathcal{A} \to \mathbb{R}$, $\|g_h\|_{\infty} \leq B$, $\forall h\in [H]$. Given Lemma~\ref{lem:mle}, we have that $\forall \pi$, 
    \begin{align*}
        \sum_{h\in[H]} \mathbb{E}_{\pi}^{\widehat{\mathcal{P}}_k}[g(x_h, a_h)] \leq & \sum_{h\in [H]} \mathbb{E}_{(x_{h-L}, a_{h-L})\sim d_{\widehat{\mathcal{P}}_k, h-L}^\pi]}^{\widehat{\mathcal{P}}_k}\left[\left\|\hat{p}(\cdot|x_{h-L}, a_{h-L})\right\|_{L_2(\mu), \Sigma_{\rho_{k, h-2L} \circ^L \mathcal{U}(\mathcal{A}), \hat{p}}^{-1}}\right]\\
        & \cdot \sqrt{k |\mathcal{A}|^{L}  \cdot \mathbb{E}_{(\tilde{x}_h, \tilde{a}_h) \sim \rho_{k, h-2L} \circ^{2L} \mathcal{U}(\mathcal{A})}\left[g(\tilde{x}_h, \tilde{a}_h)^2\right] + \lambda B^2  C + kL|\mathcal{A}|^{L-1} B^2 \zeta_k}
    \end{align*}
    % \Tongzheng{TODO:check}
\end{lemma}
\begin{proof}
    The proof can be adapted from the proof of Lemma 5 in \citet{ren2023latent}, and we include it for the completeness. We define a similar moment matching policy and make the following decomposition: 
    \begin{align*}
        &\mathbb{E}_{\pi}^{\widehat{\mathcal{P}}_k}(x_h, a_h) \\
        = & \mathbb{E}_{(x_{h-L}, a_{h-L}) \sim \pi}^{\widehat{\mathcal{P}}_k}\left[\int_{s_{h-L+1}} \langle \hat{p}(\cdot|x_{h-L}, a_{h-L}),  \hat{p}(s_{h-L+1}|\cdot) \rangle_{L_2(\mu)} \cdot \mathbb{E}^{\widehat{\mathcal{P}}_k}[g(x_h, a_h)|s_{h-L+1}, \nu^{\pi, h}] d s_{h-L+1}\right]\\
        \leq & \mathbb{E}_{(x_{h-L}, a_{h-L}) \sim \pi}^{\widehat{\mathcal{P}}_k}\left\|\hat{p}(\cdot|x_{h-L}, a_{h-L})\right\|_{L_2(\mu), \Sigma_{\rho_{k, h-2L} \circ^L \mathcal{U}(\mathcal{A}), \hat{p}}^{-1}} \\
        & \cdot \left\|\int_{s_{h-L+1}} \hat{p}(s_{h-L+1}|\cdot) \mathbb{E}^{\widehat{\mathcal{P}}_k}[g(x_h, a_h)|s_{h-L+1}, \nu^{\pi, h}] d s_{h-L+1}\right\|_{L_2(\mu), \Sigma_{\rho_{k, h-2L}\circ^{L} \mathcal{U}(\mathcal{A}), \hat{p}}}. 
    \end{align*}

Direct computation shows that
\begin{align*}
    & \left\|\int_{s_{h-L+1}} \hat{p}(s_{h-L+1}|\cdot) \mathbb{E}^{\widehat{\mathcal{P}}_k}[g(x_h, a_h)|s_{h-L+1}, \nu^{\pi, h}] d s_{h-L+1}\right\|_{L_2(\mu), \Sigma_{\rho_{k, h-2L}\circ^{L}\mathcal{U}(\mathcal{A}), \hat{p}}}^2\\
    = & k \mathbb{E}_{(\tilde{x}_{h-L}, \tilde{a}_{h-L} )\sim \rho_{k, h-2L} \circ^{L} \mathcal{U}(\mathcal{A})} \left[\mathbb{E}_{s_{h-L+1} \sim \mathbb{P}_{h-L}^{\widehat{\mathcal{P}}_k}(\cdot|\tilde{x}_{h-L}, \tilde{a}_{h-L})} \mathbb{E}^{\widehat{\mathcal{P}}_k}[g(x_h, a_h)|s_{h-L+1},\nu^{\pi, h}]\right]^2 \\
    & + \left\|\int_{s_{h-L+1}} \hat{p}(s_{h-L+1}|\cdot) \mathbb{E}[g(x_h, a_h)|s_{h-L+1}, \nu^{\pi, h}] d s_{h-L+1}\right\|_{\mathcal{H}}^2\\
    \leq & k \mathbb{E}_{(\tilde{x}_{h-L}, \tilde{a}_{h-L}) \sim \rho_{k, h-2L} \circ^{L} \mathcal{U}(\mathcal{A})}  \mathbb{E}_{s_{h-L+1} \sim \mathbb{P}_{h-L}^{\widehat{\mathcal{P}}_k}(\cdot|\tilde{x}_{h-L}, \tilde{a}_{h-L}), \nu^{\pi, h}}^{\widehat{\mathcal{P}}_k}[g(x_h, a_h)]^2 +  \lambda B^2 C \\
    \leq &  k|\mathcal{A}|^{L} \mathbb{E}_{(\tilde{x}_{h-L}, \tilde{a}_{h-L}) \sim \rho_{k, h-2L} \circ^{L} \mathcal{U}(\mathcal{A})}  \mathbb{E}^{\widehat{\mathcal{P}}_k}_{a_{h-L+1:h} \sim \circ^L\mathcal{U}(\mathcal{A})}[g(x_h, a_h)]^2 +  \lambda B^2 C \\
    \leq & k |\mathcal{A}|^{L} \mathbb{E}_{(\tilde{x}_h, \tilde{a}_h) \sim \rho_{k, h-2L}\circ^{2L} \mathcal{U}(\mathcal{A})} [g(\tilde{x}_h, \tilde{a}_h)]^2 + k L |\mathcal{A}|^{L-1} B^2 \zeta_k + \lambda B^2 C,
\end{align*}
% \Tongzheng{TODO: check} 
where we use the MLE guarantee for each individual step to obtain the last inequality. This finishes the proof. 
\end{proof}
\begin{lemma}[Almost Optimism]\label{lem:optimism}
    For episode $k\in [K]$, set 
    \begin{align*}
        \hat{b}_{k, h} = \min\left\{\alpha_k \|\hat{p}_{k}(\cdot|x_{h-L}, a_{h-L})\|_{L_2(\mu), \hat{\Sigma}_{k,h,\hat{p}_{k}}^{-1}}, 2\right\},
    \end{align*}
    with $\alpha_k = \frac{\sqrt{5kL |\mathcal{A}|^L \zeta_k + 4\lambda d}}{c}$,
    \begin{align*}
    \hat{\Sigma}_{k,h,\hat{p}_k}:L_2(\mu) \to L_2(\mu), \quad \hat{\Sigma}_{k,h, \hat{p}_k} := \sum_{(x_{h, i}, a_{h, i}) \in \mathcal{D}_{k, h}} \left[\hat{p}_{k}(z|x_{h, i}, a_{h, i}) \hat{p}_{k}(z|x_{h, i}, a_{h, i})^\top\right] + \lambda T_K^{-1}
    \end{align*}
    where $T_K$ is the integral operator associated with $K$ (i.e. $T_K f = \int f(x) K(x, \cdot) dx$) and $\lambda$ is set for different eigendecay of $K$ as follows:
    \begin{itemize}
        \item $\beta$-finite spectrum: $\lambda = \Theta(\beta\log K + \log (K|\mathcal{P}|/\delta))$
        \item $\beta$-polynomial decay: $\lambda = \Theta(C_{\mathrm{poly}}K^{1/(1 + \beta)} + \log (K|\mathcal{P}|/\delta))$;
        \item $\beta$-exponential decay: $\lambda = \Theta(C_{\mathrm{exp}}(\log K)^{1/\beta} + \log (K|\mathcal{P}|/\delta))$;
    \end{itemize}
    $c$ is an absolute constant, then with probability at least $1-\delta$, $\forall k\in[K]$ we have
    \begin{align*}
        V^{\pi^*, \widehat{\mathcal{P}}_k, r + \widehat{b}_k} - V^{\pi^*, \mathcal{P}, r} \geq -\sqrt{|\mathcal{A}|^{L+1} \zeta_k}
    \end{align*}
    \begin{proof}
        With Lemma~\ref{lem:simulation}, we have that
        \begin{align*}
            & V^{\pi^*, \widehat{\mathcal{P}}_k, r + \widehat{b}^k} - V^{\pi^*, \mathcal{P}, r} \\
            = & \sum_{h\in[H]} \mathbb{E}_{(x_h, a_h)\sim d_{\widehat{\mathcal{P}}_k, h}^{\pi^*}}\left[\widehat{b}_h^k(x_h, a_h) + \mathbb{E}_{o^\prime \sim \mathbb{P}_h^{\widehat{\mathcal{P}}_k}(\cdot|x_h, a_h)}[V_{h+1}^{\pi^*, \mathcal{P}, r}(x_{h+1}^\prime)] - \mathbb{E}_{o^\prime \sim \mathbb{P}_h^{\mathcal{P}}(\cdot|x_h, a_h)}[V_{h+1}^{\pi^*, \mathcal{P}, r}(x_{h+1}^\prime)]\right]\\
            \geq & \sum_{h\in[H]} \mathbb{E}_{(x_h, a_h)\sim d_{\widehat{\mathcal{P}}_k, h}^{\pi^*}}\left[\min\left[c\alpha_k \left\|\hat{p}(\cdot|x_{h-L}, a_{h-L})\right\|_{L_2(\mu), \Sigma_{\rho_{k, h-L}, \hat{p}}^{-1}}, 2\right] + \mathbb{E}_{o^\prime \sim \mathbb{P}_h^{\widehat{\mathcal{P}}_k}(\cdot|x_h, a_h)}[V_{h+1}^{\pi^*, \mathcal{P}, r}(x_{h+1}^\prime)] \right.\\
            & \left. - \mathbb{E}_{o^\prime \sim \mathbb{P}_h^{\mathcal{P}}(\cdot|x_h, a_h)}[V_{h+1}^{\pi^*, \mathcal{P}, r}(x_{h+1}^\prime)]\right],
        \end{align*}
        where in the last step we replace the empirical covariance with the population counterpart thanks to Lemma 17 in \citet{ren2023latent}. Define
        \begin{align*}
            g_h(z_h, a_h) = \mathbb{E}_{o^\prime \sim \mathbb{P}_h^{\mathcal{P}}(\cdot|x_h, a_h)}[V_{h+1}^{\pi^*, \mathcal{P}, r}(x_{h+1}^\prime)] - \mathbb{E}_{o^\prime \sim \mathbb{P}_h^{\widehat{\mathcal{P}}_k}(\cdot|x_h, a_h)}[V_{h+1}^{\pi^*, \mathcal{P}, r}(x_{h+1}^\prime)],
        \end{align*}
        With H\"older's inequality, we have that $\|g_h\|_{\infty} \leq 2$.

        Furthermore, with Lemma~\ref{lem:back_learned}, we have that
        \begin{align*}
            & \sum_{h\in [H]}\mathbb{E}_{(x_h, a_h) \sim d_{\widehat{\mathcal{P}}_k, h}^{\pi^*}} [g_h(x_h, a_h)]\\
            \leq & \sum_{h\in [H]} \mathbb{E}_{(x_{h-L}, a_{h-L})\sim d_{\widehat{\mathcal{P}}_k, h}^{\pi^*}} \left[\left\|\hat{p}(\cdot|x_{h-L}, a_{h-L})\right\|_{L_2(\mu), \Sigma_{\rho_{k, h-L}, \hat{p}}^{-1}}\right] \\
            & \cdot \sqrt{k |\mathcal{A}|^L \cdot \mathbb{E}_{(\tilde{x}_h, \tilde{a}_h) \sim \rho_{h-2L}\circ^{2L} \mathcal{U}(\mathcal{A})}\left[g(\tilde{x}_h, \tilde{a}_h)^2\right] + 4\lambda  C + 4kL|\mathcal{A}|^{L-1}  \zeta_k}\\
            \leq &  \sum_{h\in [H]} \mathbb{E}_{(x_{h-L}, a_{h-L})\sim d_{\widehat{\mathcal{P}}_k, h}^{\pi^*}} \left[c \alpha_k\left\|\hat{p}(\cdot|x_{h-L}, a_{h-L})\right\|_{L_2(\mu), \Sigma_{\rho_{k, h-L}, \hat{p}}^{-1}}\right],
        \end{align*}
        where we use Lemma~\ref{lem:mle} in the last step. Now we deal with the case with $h\in [L]$. Note that, $\forall h\in [L]$
        \begin{align*}
            & \mathbb{E}_{(x_h, a_h) \sim \pi}[g_h(x_h, a_h)] \\
            \leq & |\mathcal{A}|^h\mathbb{E}_{x_1\sim d_1, a_{1:h} \sim \circ^h\mathcal{U}(\mathcal{A})} \left\|\mathbb{P}^{\widehat{\mathcal{P}}_k}_{h}(\cdot|x_h, a_h) - \mathbb{P}^{\mathcal{{P}}}_h(\cdot|x_h, a_h)\right\|_1\\
            \leq & \sqrt{\mathbb{E}_{x_1 \sim d_1, a_{1:h}\circ^L \mathcal{U}(\mathcal{A})} \left\|\mathbb{P}^{\widehat{\mathcal{P}}_k}_{h}(\cdot|x_h, a_h) - \mathbb{P}^{\mathcal{{P}}}_h(\cdot|x_h, a_h)\right\|_1^2}\\
            \leq & \sqrt{|\mathcal{A}|^h \zeta_k},
        \end{align*}
        where in the last step we use Lemma~\ref{lem:mle}. We finish the proof by summing over $h\in [L]$.
    \end{proof}
    \begin{lemma}[Regret]\label{lem:regret}
        With probability at least $1-\delta$, we have that
        \begin{itemize}
            \item For $\beta$-finite spectrum, we have
            \begin{align*}
                \sum_{k=1}^K V^{\pi^*, \mathcal{P}, r} - V^{\pi_k, \mathcal{P}, r} \lesssim \sum_{k=1}^K V^{\pi^*, \mathcal{P}, r} - V^{\pi_k, \mathcal{P}, r} \lesssim H^2 \beta^{3/2} |\mathcal{A}|^L\log K \sqrt{CLK \log (K|\mathcal{M}|/\delta)};
            \end{align*}
            \item For $\beta$-polynomial decay, we have
            \begin{align*}
                \sum_{k=1}^K V^{\pi^*, \mathcal{P}, r} - V^{\pi_k, \mathcal{P}, r} \lesssim  C_{\mathrm{poly}}H^2  |\mathcal{A}|^L K^{\frac{1}{2} + \frac{1}{1 + \beta}}\sqrt{CL \log (K|\mathcal{M}|/\delta)};
            \end{align*}
            \item For $\beta$-exponential decay, we have
            \begin{align*}
                \sum_{k=1}^K V^{\pi^*, \mathcal{P}, r} - V^{\pi_k, \mathcal{P}, r} \lesssim  C_{\mathrm{exp}}H^2 |\mathcal{A}|^L(\log K)^{1 + \frac{3}{2\beta}} \sqrt{CLK \log (K|\mathcal{M}|/\delta)};
            \end{align*}
        \end{itemize}
    \end{lemma}
    \begin{proof}
        With Lemma~\ref{lem:optimism} and Lemma~\ref{lem:simulation}, we have
        \begin{align*}
            & V^{\pi^*, \mathcal{P}, r} - V^{\pi_k, \mathcal{P}, r}\\
            \leq & V^{\pi^*, \widehat{\mathcal{P}}_k, r + \widehat{b}^k} + \sqrt{|\mathcal{A}|^{L+1}\zeta_k} - V^{\pi_k, \mathcal{P}, r}\\
            \leq & V^{\pi^k, \widehat{\mathcal{P}}_k, r + \widehat{b}^k} + \sqrt{|\mathcal{A}|^{L+1}\zeta_k} - V^{\pi_k, \mathcal{P}, r}\\
            = & \sum_{h\in[H]} \mathbb{E}_{(x_h, a_h)\sim d_{\mathcal{P}, h}^{\pi_k}}\left[\widehat{b}_h^k(x_h, a_h) + \mathbb{E}_{o^\prime \sim \mathbb{P}_h^{\widehat{\mathcal{P}}_k}(\cdot|x_h, a_h)}\left[V_{h+1}^{\pi_k, \widehat{\mathcal{P}}_k, r + \widehat{b}_h^k}(x_{h+1}^\prime)\right] - \mathbb{E}_{o^\prime \sim \mathbb{P}_h^{\mathcal{P}}(\cdot|x_h, a_h)}\left[V_{h+1}^{\pi_k, \widehat{\mathcal{P}}_k, r + \widehat{b}_h^k}(x_{h+1}^\prime)\right]\right],\\
            & + \sqrt{|\mathcal{A}|^{L+1}\zeta_k}.
        \end{align*}
        Note that $\left\|\widehat{b}_h^k\right\|_{\infty} \leq 2$. Applying Lemma~\ref{lem:back_true}, we have that
        \begin{align*}
            & \sum_{h\in[H]} \mathbb{E}_{(x_h, a_h) \sim d_{\mathcal{P}, h}^{\pi_k}} \left[\widehat{b}_h^k(x_h, a_h)\right]\\
            \leq & \sum_{h\in[H]} \mathbb{E}_{(\tilde{x}_{h-L}, \tilde{a}_{h-L})\sim d_{\mathcal{P}, h}^{\pi_k}} \left[\left\|p^*(\cdot|x_{h-L}, a_{h-L})\right\|_{L_2(\mu), \Sigma_{\rho_{k, h-L}, p^*}^{-1}}\right]\\
            & \cdot \sqrt{k|\mathcal{A}|^L \cdot \mathbb{E}_{(\tilde{x}_h, \tilde{a}_h) \sim \rho_{k, h-L} \circ^L \mathcal{U}(\mathcal{A})}\left[\widehat{b}_{h}^k(\tilde{x}_h, \tilde{a}_h)^2\right] + 4\lambda C}
        \end{align*}
        Following the proof of Lemma 8 in \citet{ren2023latent}, we have that: 
        \begin{itemize}
            \item for $\beta$-finite spectrum, 
            \begin{align*}
                 k\mathbb{E}_{(\tilde{x}_h, \tilde{a}_h) \sim \rho_{k, h-L}\circ^L \mathcal{U}(\mathcal{A})}\left[\widehat{b}_{h}^k(\tilde{x}_h, \tilde{a}_h)^2\right] = O(\beta \log K);
            \end{align*}
            \item for $\beta$-polynomial decay,
            \begin{align*}
                 k\mathbb{E}_{(\tilde{x}_h, \tilde{a}_h) \sim \rho_{k, h-L}\circ^L \mathcal{U}(\mathcal{A})}\left[\widehat{b}_{h}^k(\tilde{x}_h, \tilde{a}_h)^2\right] = O\left(C_{\mathrm{poly}} K^{\frac{1}{2(1 + \beta)}}\log K\right);
            \end{align*}
            \item for $\beta$-exponential decay,
            \begin{align*}
                 k\mathbb{E}_{(\tilde{x}_h, \tilde{a}_h) \sim \rho_{k, h-L}\circ^L \mathcal{U}(\mathcal{A})}\left[\widehat{b}_{h}^k(\tilde{x}_h, \tilde{a}_h)^2\right] = O\left(C_{\mathrm{exp}}(\log K)^{1+ 1/\beta}\right).
            \end{align*}
        \end{itemize}
        We then consider
        \begin{align*}
            \sum_{h\in[H]} \mathbb{E}_{(x_h, a_h)\sim d_{\mathcal{P}, h}^{\pi_k}}\left[ \mathbb{E}_{o^\prime \sim \mathbb{P}_h^{\widehat{\mathcal{P}}_k}(\cdot|x_h, a_h)}\left[V_{h+1}^{\pi_k, \widehat{\mathcal{P}}_k, r + \widehat{b}_h^k}(x_{h+1}^\prime)\right] - \mathbb{E}_{o^\prime \sim \mathbb{P}_h^{\mathcal{P}}(\cdot|x_h, a_h)}\left[V_{h+1}^{\pi_k, \widehat{\mathcal{P}}_k, r + \widehat{b}_h^k}(x_{h+1}^\prime)\right]\right].
        \end{align*}
        Define
        \begin{align*}
            g(x_h, a_h) = \frac{1}{2H+1}\left[\mathbb{E}_{o^\prime \sim \mathbb{P}_h^{\widehat{\mathcal{P}}_k}(\cdot|x_h, a_h)}\left[V_{h+1}^{\pi_k, \widehat{\mathcal{P}}_k, r + \widehat{b}_h^k}(x_{h+1}^\prime)\right] - \mathbb{E}_{o^\prime \sim \mathbb{P}_h^{\mathcal{P}}(\cdot|x_h, a_h)}\left[V_{h+1}^{\pi_k, \widehat{\mathcal{P}}_k, r + \widehat{b}_h^k}(x_{h+1}^\prime)\right]\right].
        \end{align*}
        With H\"older's inequality and note that $\left\|\widehat{b}_h^k\right\|\leq 2$, we have that $\|g\|_{\infty} \leq 2$. With Lemma~\ref{lem:back_true}, we have that
        \begin{align*}
            & \sum_{h\in[H]} \mathbb{E}_{(x_h, a_h)\sim d_{\mathcal{P}, h}^{\pi_k}} [g(x_h, a_h)]\\
            \leq & \sum_{h\in [H]} \mathbb{E}_{(x_{h-L}, a_{h-L})\sim d_{\mathcal{P}, h}^{\pi_k}}^{\mathcal{P}}\left[\left\|p^*(\cdot|x_{h-L}, a_{h-L})\right\|_{L_2(\mu), \Sigma_{\rho_{k, h-L}, p^*}^{-1}}\right] \cdot \sqrt{k|\mathcal{A}|^L \mathbb{E}_{(\tilde{x}_h, \tilde{a}_h) \sim \rho_{k, h-L} \circ^L \mathcal{U}(\mathcal{A})}\left[g(\tilde{x}_h, \tilde{a}_h)^2\right] + 4\lambda  C}\\
            \leq & \sum_{h\in [H]} \mathbb{E}_{(x_{h-L}, a_{h-L})\sim d_{\mathcal{P}, h}^{\pi_k}}^{\mathcal{P}}\left[\left\|p^*(\cdot|x_{h-L}, a_{h-L})\right\|_{L_2(\mu), \Sigma_{\rho_{k, h-L}, p^*}^{-1}}\right] \cdot \sqrt{k|\mathcal{A}|^L \zeta_k + 4\lambda  C}\\
            \leq & c\alpha_k \sum_{h\in [H]} \mathbb{E}_{(x_{h-L}, a_{h-L})\sim d_{\mathcal{P}, h}^\pi}^{\mathcal{P}}\left[\left\|p^*(\cdot|x_{h-L}, a_{h-L})\right\|_{L_2(\mu), \Sigma_{\rho_{k, h-L}, p^*}^{-1}}\right]
        \end{align*}
        With Cauchy-Schwartz inequality, we know that
        \begin{align*}
            & \sum_{k\in [K]} \mathbb{E}_{(x_{h-L}, a_{h-L})\sim d_{\mathcal{P}, h}^{\pi_k}}^{\mathcal{P}}\left[\left\|p^*(\cdot|x_{h-L}, a_{h-L})\right\|_{L_2(\mu), \Sigma_{\rho_{k, h-L}, p^*}^{-1}}\right] \\
            \leq & \sqrt{K \sum_{k\in [K]} \mathbb{E}_{(x_{h-L}, a_{h-L})\sim d_{\mathcal{P}, h}^{\pi_k}}^{\mathcal{P}}\left[\left\|p^*(\cdot|x_{h-L}, a_{h-L})\right\|_{L_2(\mu), \Sigma_{\rho_{k, h-L}, p^*}^{-1}}^2\right]}.
        \end{align*}
        Following the proof of Lemma 8 in \citet{ren2023latent}, we have that
        \begin{itemize}
            \item for $\beta$-finite spectrum,
            \begin{align*}
                \sum_{k\in [K]} \mathbb{E}_{(x_{h-L}, a_{h-L})\sim d_{\mathcal{P}, h}^{\pi_k}}^{\mathcal{P}}\left[\left\|p^*(\cdot|x_{h-L}, a_{h-L})\right\|_{L_2(\mu), \Sigma_{\rho_{k, h-L}, p^*}^{-1}}^2\right] = O(\beta \log K);
            \end{align*}
            \item for $\beta$-polynomial decay,
            \begin{align*}
                \sum_{k\in [K]} \mathbb{E}_{(x_{h-L}, a_{h-L})\sim d_{\mathcal{P}, h}^{\pi_k}}^{\mathcal{P}}\left[\left\|p^*(\cdot|x_{h-L}, a_{h-L})\right\|_{L_2(\mu), \Sigma_{\rho_{k, h-L}, p^*}^{-1}}^2\right] = O\left(C_{\mathrm{poly}} K^{\frac{1}{2(1 + \beta)}}\log K\right);
            \end{align*}
            \item for $\beta$-exponential decay,
            \begin{align*}
                \sum_{k\in [K]} \mathbb{E}_{(x_{h-L}, a_{h-L})\sim d_{\mathcal{P}, h}^{\pi_k}}^{\mathcal{P}}\left[\left\|p^*(\cdot|x_{h-L}, a_{h-L})\right\|_{L_2(\mu), \Sigma_{\rho_{k, h-L}, p^*}^{-1}}^2\right] = O\left(C_{\mathrm{exp}}(\log K)^{1+ 1/\beta}\right).
            \end{align*}
        \end{itemize}
        Combine the previous steps and take the dominating term out, we have that
        \begin{itemize}
            \item for $\beta$-finite spectrum, 
            \begin{align*}
                \sum_{k=1}^K V^{\pi^*, \mathcal{P}, r} - V^{\pi_k, \mathcal{P}, r} \lesssim H^2 \beta^{3/2} |\mathcal{A}|^L\log K \sqrt{CLK \log (K|\mathcal{M}|/\delta)};
            \end{align*}
            \item for $\beta$-polynomial decay,
            \begin{align*}
                \sum_{k=1}^K V^{\pi^*, \mathcal{P}, r} - V^{\pi_k, \mathcal{P}, r} \lesssim C_{\mathrm{poly}}H^2  |\mathcal{A}|^L K^{\frac{1}{2} + \frac{1}{1 + \beta}}\sqrt{C L \log (K|\mathcal{M}|/\delta)};
            \end{align*}
            \item for $\beta$-exponential decay,
            \begin{align*}
                \sum_{k=1}^K V^{\pi^*, \mathcal{P}, r} - V^{\pi_k, \mathcal{P}, r} \lesssim C_{\mathrm{exp}}H^2 |\mathcal{A}|^L(\log K)^{1 + \frac{3}{2\beta}} \sqrt{CLK \log (K|\mathcal{M}|/\delta)};
            \end{align*}
        \end{itemize}
        which finishes the proof.
    \end{proof}
\end{lemma}
\begin{theorem}[PAC Guarantee]
    \label{thm:pac_guarantee_online}
    After interacting with the environments for $K H$ episodes
    \begin{itemize}
        \item $K =  \Theta\left(\frac{C H^4 L \beta^3 |\mathcal{A}|^{2L} \log (|\mathcal{P}|/\delta)}{\varepsilon^2}\log^3 \left(\frac{C H^4 L \beta^3 |\mathcal{A}|^{2L} \log (|\mathcal{P}|/\delta)}{\varepsilon^2}\right)\right)$ for $\beta$-finite spectrum;
        \item $K = \Theta\left(C_{\mathrm{poly}}\left(\frac{H^2 L |\mathcal{A}|^L\sqrt{C\log (|\mathcal{P}|/\delta)}}{\varepsilon} \log^{3/2}\left(\frac{\sqrt{C}H^2 L |\mathcal{A}|^L\log (|\mathcal{P}|/\delta)}{\varepsilon} \right)\right)^{\frac{2(1+\beta)}{\beta - 1}}\right)$ for $\beta$-polynomial decay;
        \item $K =\Theta\left(\frac{C_{\mathrm{exp}}C H^4 L |\mathcal{A}|^{2L}\log (|\mathcal{P}|/\delta)}{\varepsilon^2} \log^{\frac{3 + 2\beta}{\beta}}\left(\frac{C H^4 L|\mathcal{A}|^{2L}\log (|\mathcal{P}|/\delta)}{\varepsilon^2} \right)\right)$ for $\beta$-exponential decay;
    \end{itemize} 
    we can obtain an $\varepsilon$-optimal policy with high probability.
\end{theorem}
\begin{proof}
    This is a direct extension of the proof of Theorem 9 in \citet{ren2023latent}.
\end{proof}
\section{Technical Lemma}
\begin{lemma}[Simulation Lemma]
\label{lem:simulation}
    For two MDPs $\mathcal{M} = (P, r)$ and $\mathcal{M}^\prime = (P^\prime, r + b)$, we have
    \begin{align*}
        & V_{P^\prime, r+b}^\pi - V_{P, r}^\pi \\
        = & \sum_{h\in [H]} \mathbb{E}_{(s_h, a_h)\sim d_{P, \pi}^h} \left[b_h(s_h, a_h) + \mathbb{E}_{s_{h+1}\sim P^\prime(s_h, a_h)} \left[V_{P^\prime, r+b, h+1}^\pi(s_{h+1})\right] -  \mathbb{E}_{s_{h+1}\sim P(s_h, a_h)} \left[V_{P^\prime, r+b, h+1}^\pi(s_{h+1})\right]\right],
    \end{align*}
    and 
    \begin{align*}
        & V_{P^\prime, r+b}^\pi - V_{P, r}^\pi \\
        = & \sum_{h\in [H]} \mathbb{E}_{(s_h, a_h)\sim d_{P^\prime, \pi}^h} \left[b_h(s_h, a_h) + \mathbb{E}_{s_{h+1}\sim P^\prime(s_h, a_h)} \left[V_{P, r, h+1}^\pi(s_{h+1})\right] -  \mathbb{E}_{s_{h+1}\sim P(s_h, a_h)} \left[V_{P, r, h+1}^\pi(s_{h+1})\right]\right],
    \end{align*}
\end{lemma}
For the proof, see \citet{uehara2021representation} for an example.

\begin{lemma}[MLE Guarantee]
\label{lem:mle}
For any episoode $k\in [K]$, step $h\in [H]$, define $\rho_h$ as the joint distribution of $(x_h, a_h)$ in the dataset $\mathcal{D}_{h, k}$ at episode $k$. Then with probability at least $1-\delta$, we have that
\begin{align*}
    \mathbb{E}_{(x_h, a_h) \sim \mathcal{D}_{h, k}} \left\|\mathbb{P}_h^{\mathcal{P}}(\cdot|x_h, a_h) - \mathbb{P}_h^{\widehat{\mathcal{P}}_k}(\cdot|x_h, a_h)\right\|_1^2 \leq \zeta_k,
\end{align*}
where $\zeta_k = O(\log (Hk|\mathcal{M}|/\delta)/k)$
\end{lemma}
For the proof, see \citet{agarwal2020flambe}.
%%%%%%%%%%%%%%%%%%%%%%%%%%%%%%%%%%%%%%%%%%%%%%%%%%
% \vspace{-2mm}
\section{Implementation Details on Image-based Continuous Control}\label{sec:imp detail}
% \vspace{-2mm}
%%%%%%%%%%%%%%%%%%%%%%%%%%%%%%%%%%%%%%%%%%%%%%%%%%

% \Bo{@Haoming, please make sure this part is consistent with the experiments. Do we still have experiments on DMC?}

We evaluate our method on Meta-world~\citep{yu2019meta}~\footnote{\url{https://github.com/Farama-Foundation/Metaworld}} and DeepMind Control Suites~\citep{tassa2018deepmind}~\footnote{\url{https://github.com/google-deepmind/dm\_control}} to demonstrate its capability for complex visual control tasks.
% ~\footnote{\url{https://github.com/google-deepmind/dm_control}}
Meta-world is an open-source simulated benchmark consisting of 50 distinct robotic manipulation tasks. The DeepMind Control Suite is a set of continuous control tasks with a standardized structure and interpretable rewards, intended to serve as performance benchmarks for reinforcement learning agents.
The visualization of some tasks from the two domains are shown in \cref{fig:env dmc,fig:env metaworld}.
With only one frame of the visual observation, we will miss some information related to the task, for example the speed, thus these tasks are partially observable. The performance for Meta-world tasks are shown in \cref{fig: metaworld performance}.
% Extended results on DeepMind Control Suites are shown in \cref{fig:dmc}.

\begin{figure}[htbp]
    % \vskip -0.1in
\begin{center}
    \subfigure[Coffee Pull]{\includegraphics[width=0.245\textwidth]{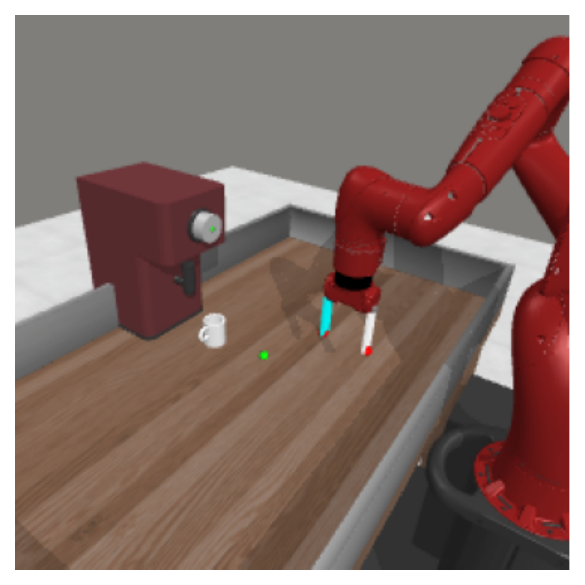}}
    \hspace{-0.1in}
    \subfigure[Hand Insert]{\includegraphics[width=0.245\textwidth]{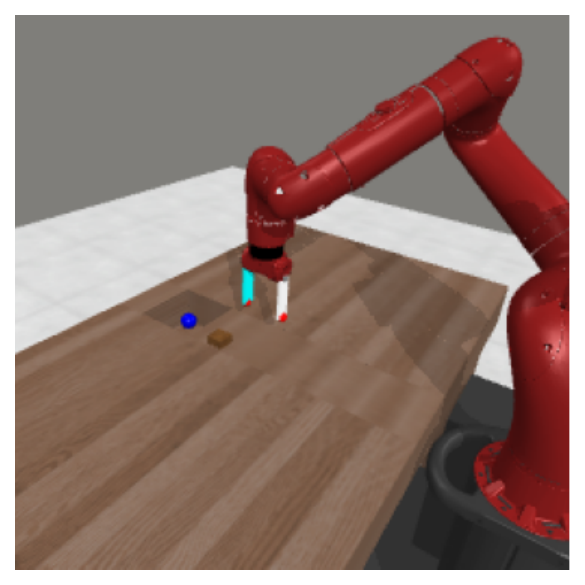}}
    \hspace{-0.1in}
    \subfigure[Push]{\includegraphics[width=0.245\textwidth]{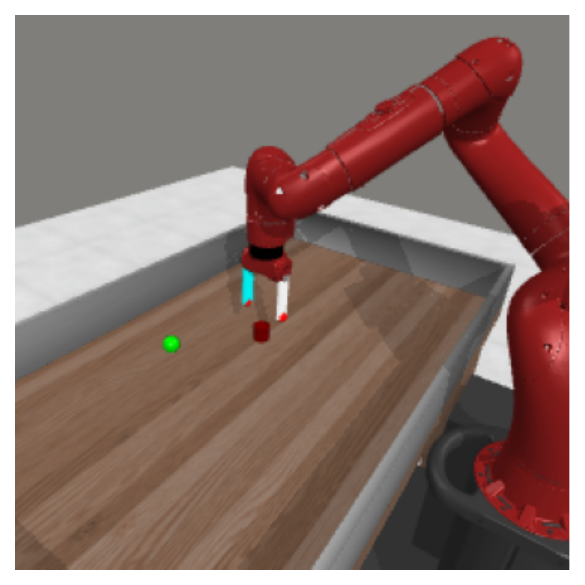}}
    \hspace{-0.1in}
    \subfigure[Plate Slide]{\includegraphics[width=0.245\textwidth]{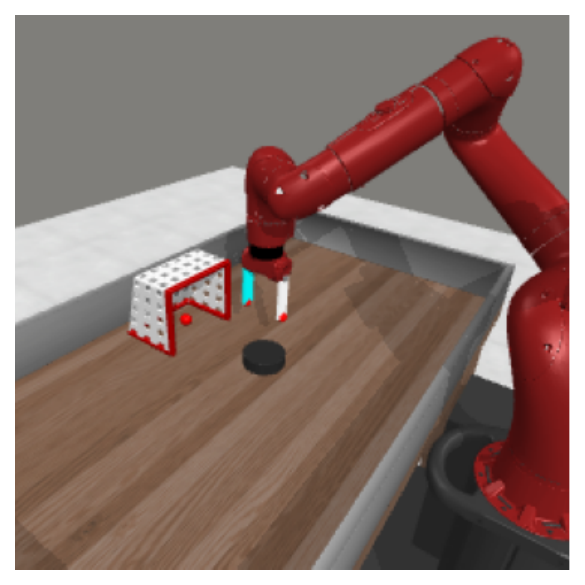}}
    \caption{\footnotesize
    Visualization of the visual robotic manipulation tasks in Meta-world. 
    \normalsize}
    \label{fig:env metaworld}
\end{center}
\end{figure}

\begin{figure}[htbp]
    % \vskip -0.1in
\begin{center}
    \subfigure[Acrobot Swingup]{\includegraphics[width=0.245\textwidth]{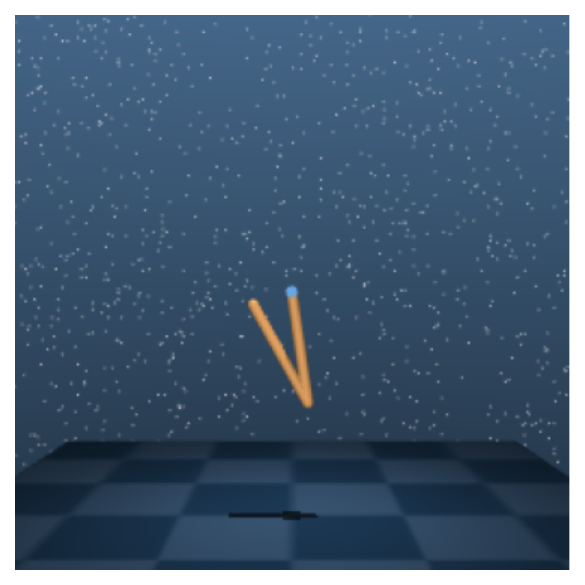}}
    \hspace{-0.1in}
    \subfigure[Reacher hard]{\includegraphics[width=0.245\textwidth]{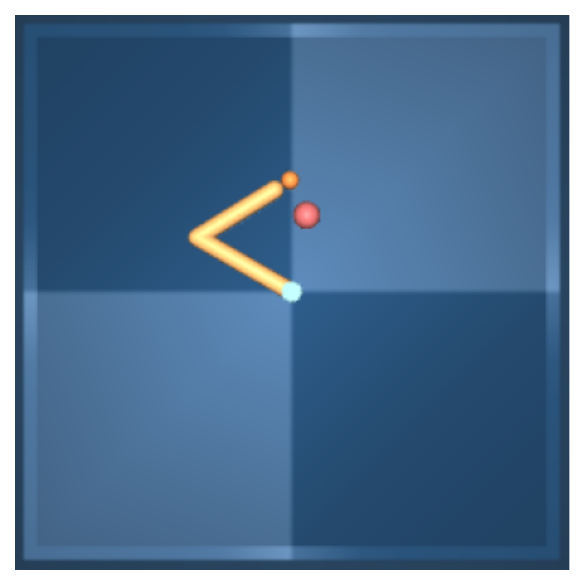}}
    \hspace{-0.1in}
    \subfigure[Reach Duplo]{\includegraphics[width=0.245\textwidth]{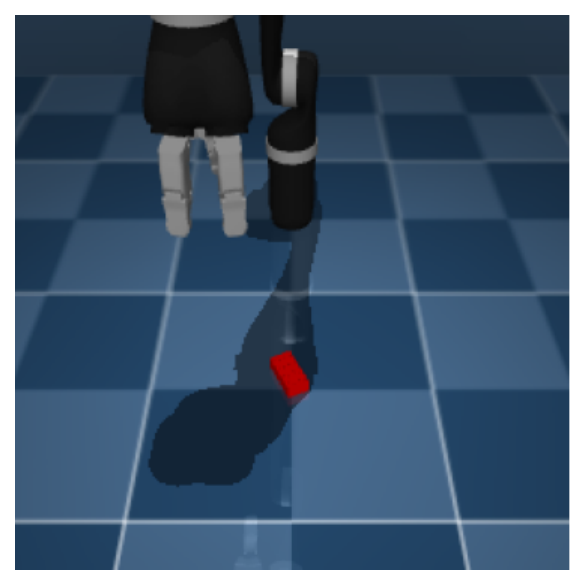}}
    \hspace{-0.1in}
    \subfigure[Quadruped Run]{\includegraphics[width=0.245\textwidth]{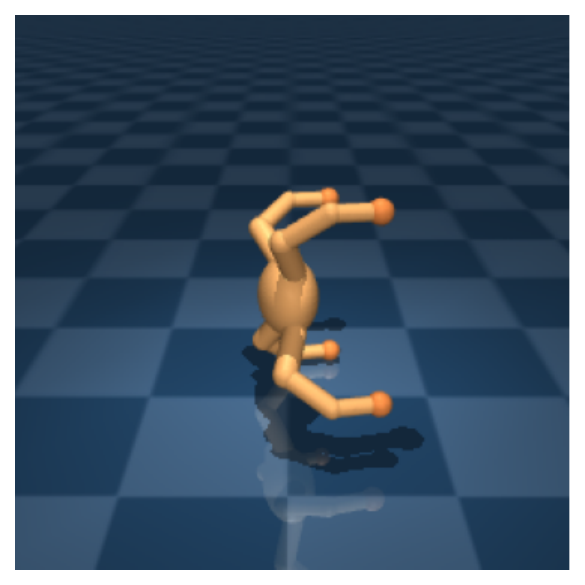}}
    \caption{\footnotesize
    Visualization of the visual control tasks in DeepMind Control Suites. 
    \normalsize}
    \label{fig:env dmc}
\end{center}
\end{figure}

\begin{figure}[htbp]
    \vskip -0.15in
\begin{center}
    \subfigure{\includegraphics[width=0.18\textwidth]{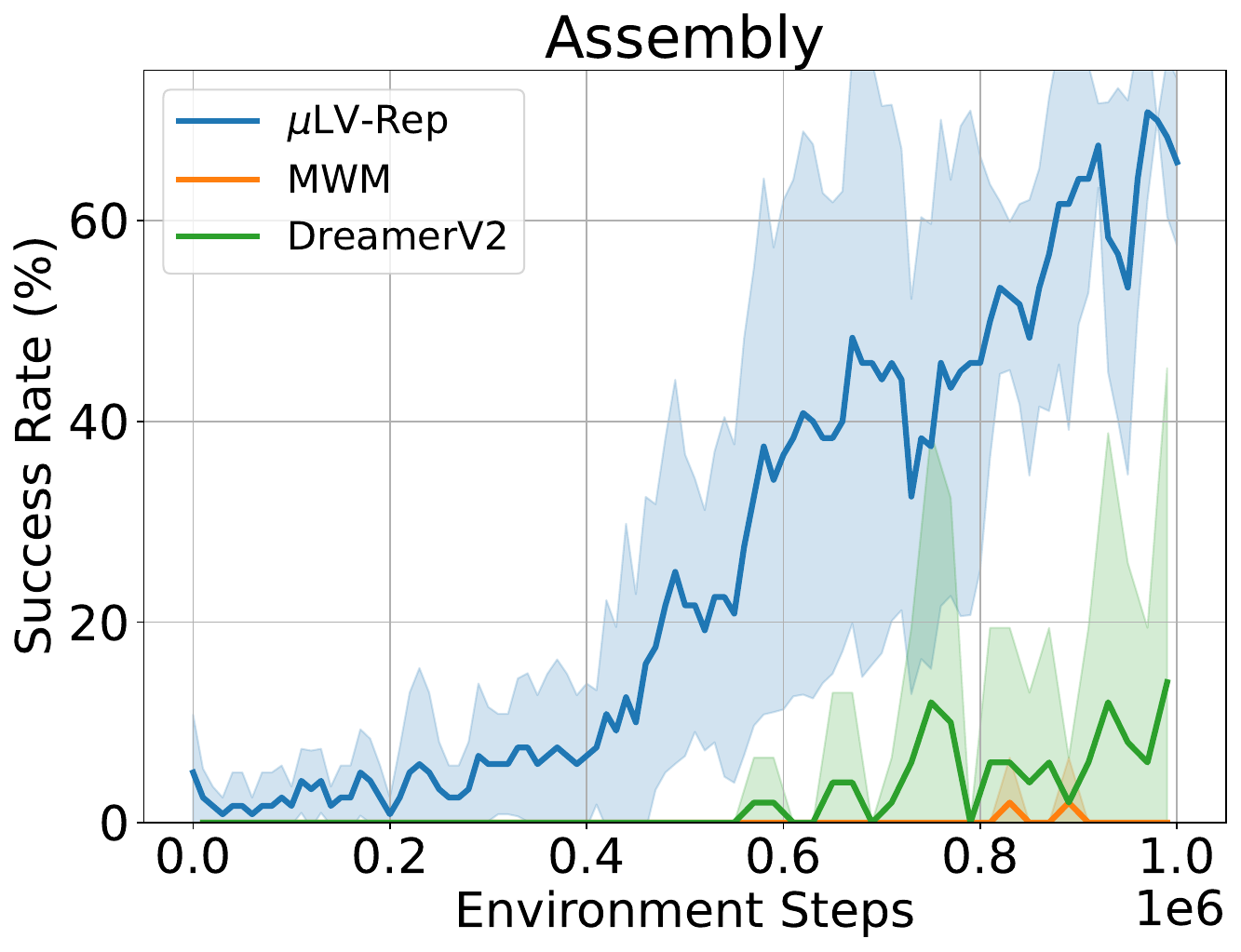}}
    \subfigure{\includegraphics[width=0.18\textwidth]{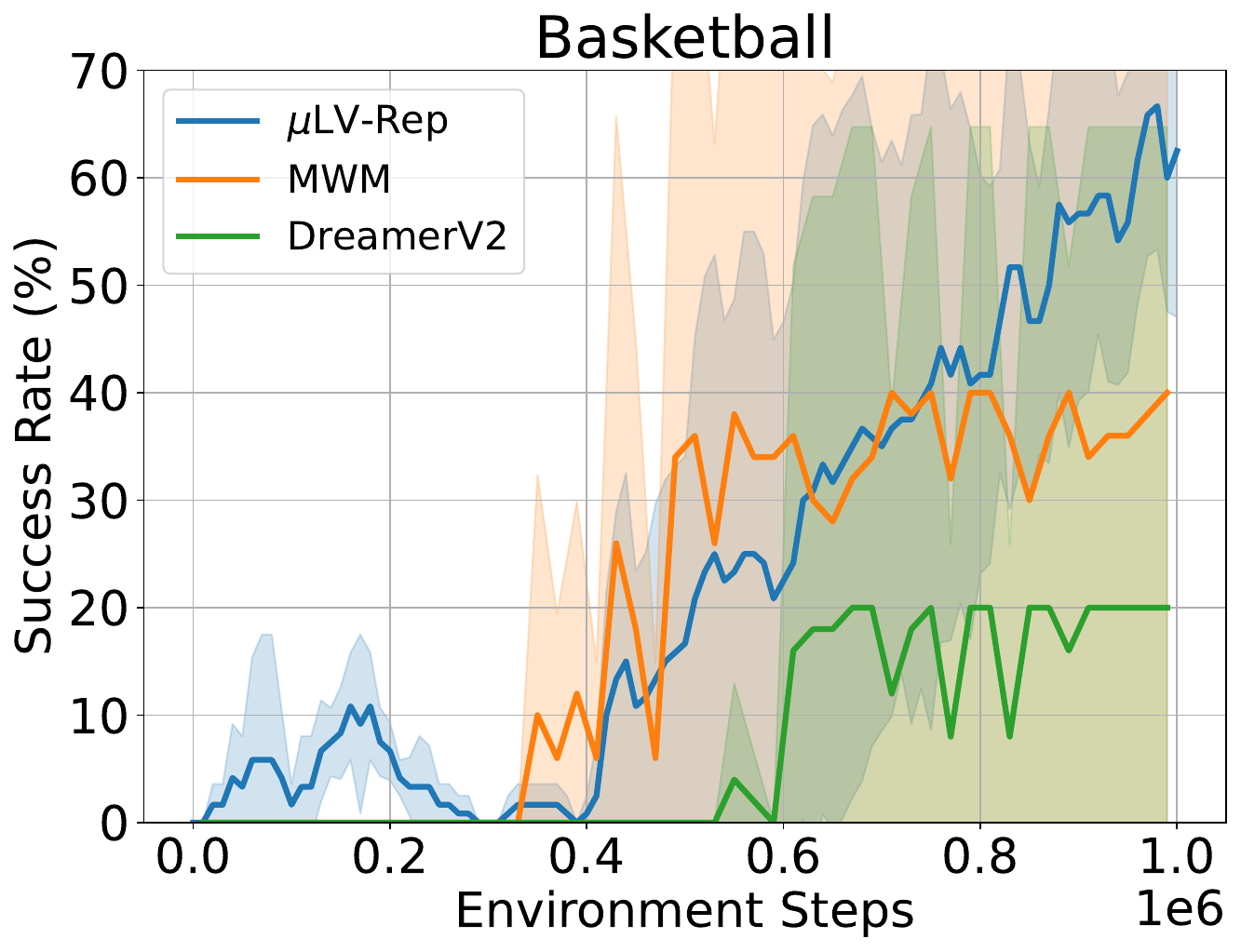}}
    \subfigure{\includegraphics[width=0.18\textwidth]{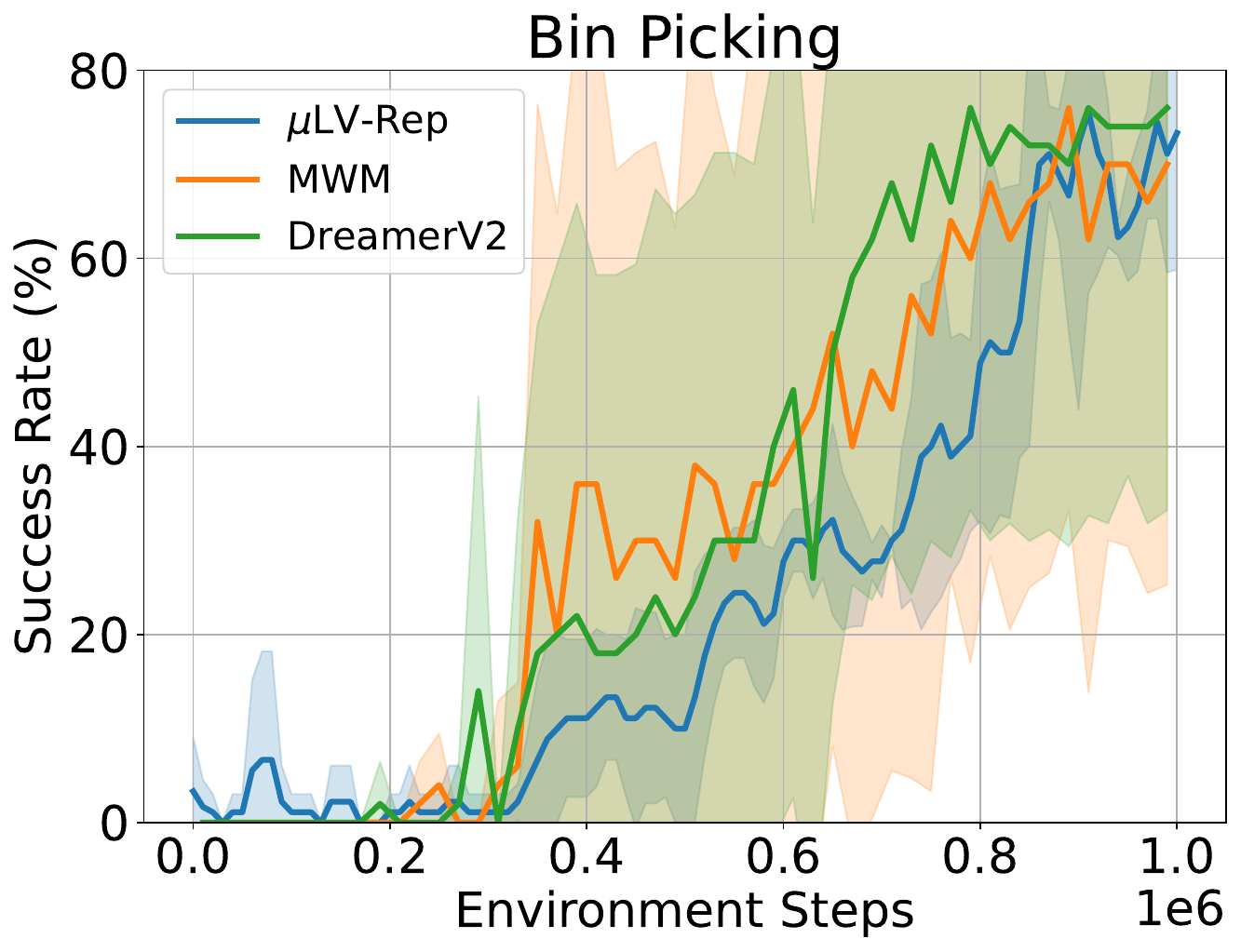}}
    \subfigure{\includegraphics[width=0.18\textwidth]{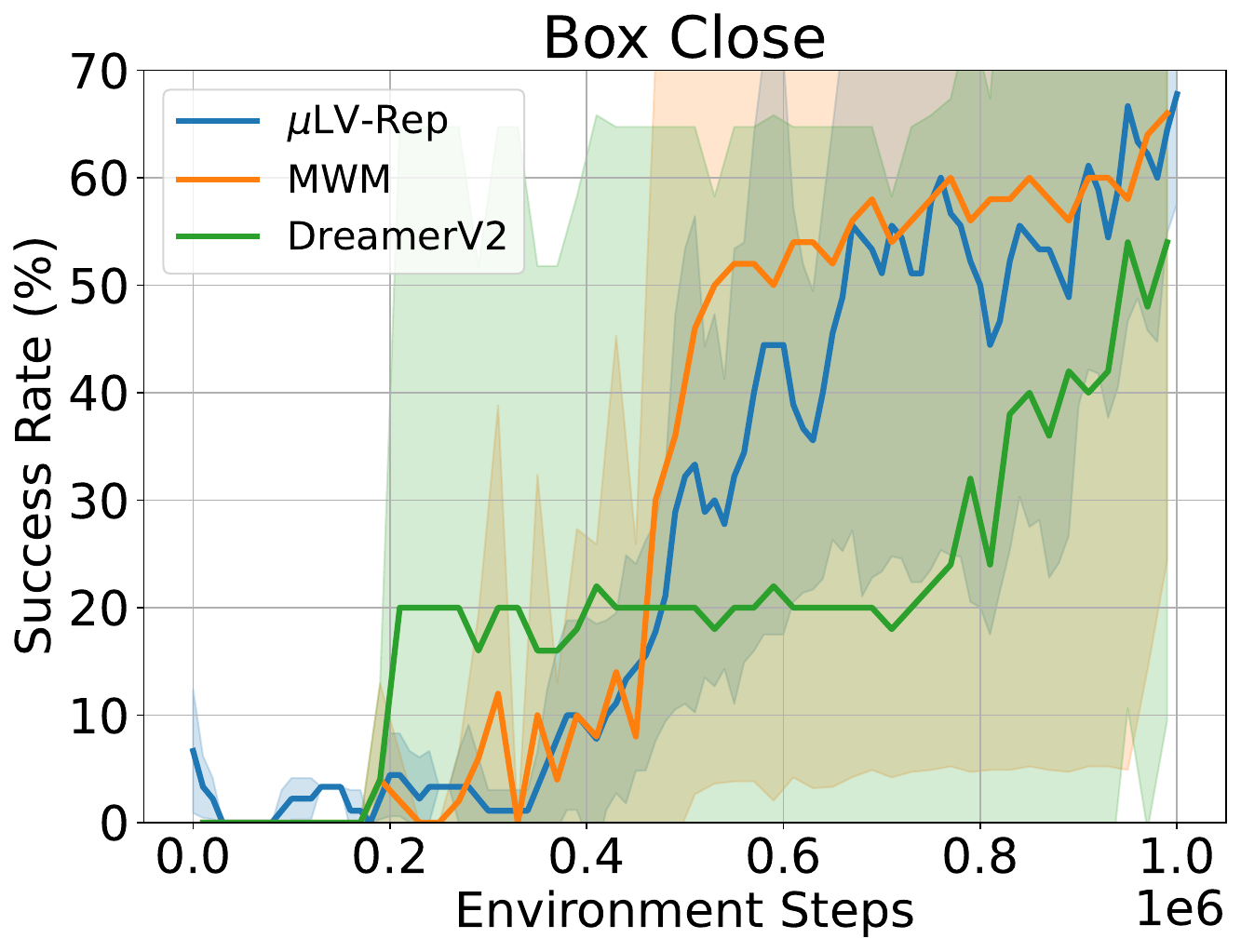}}
    \subfigure{\includegraphics[width=0.18\textwidth]{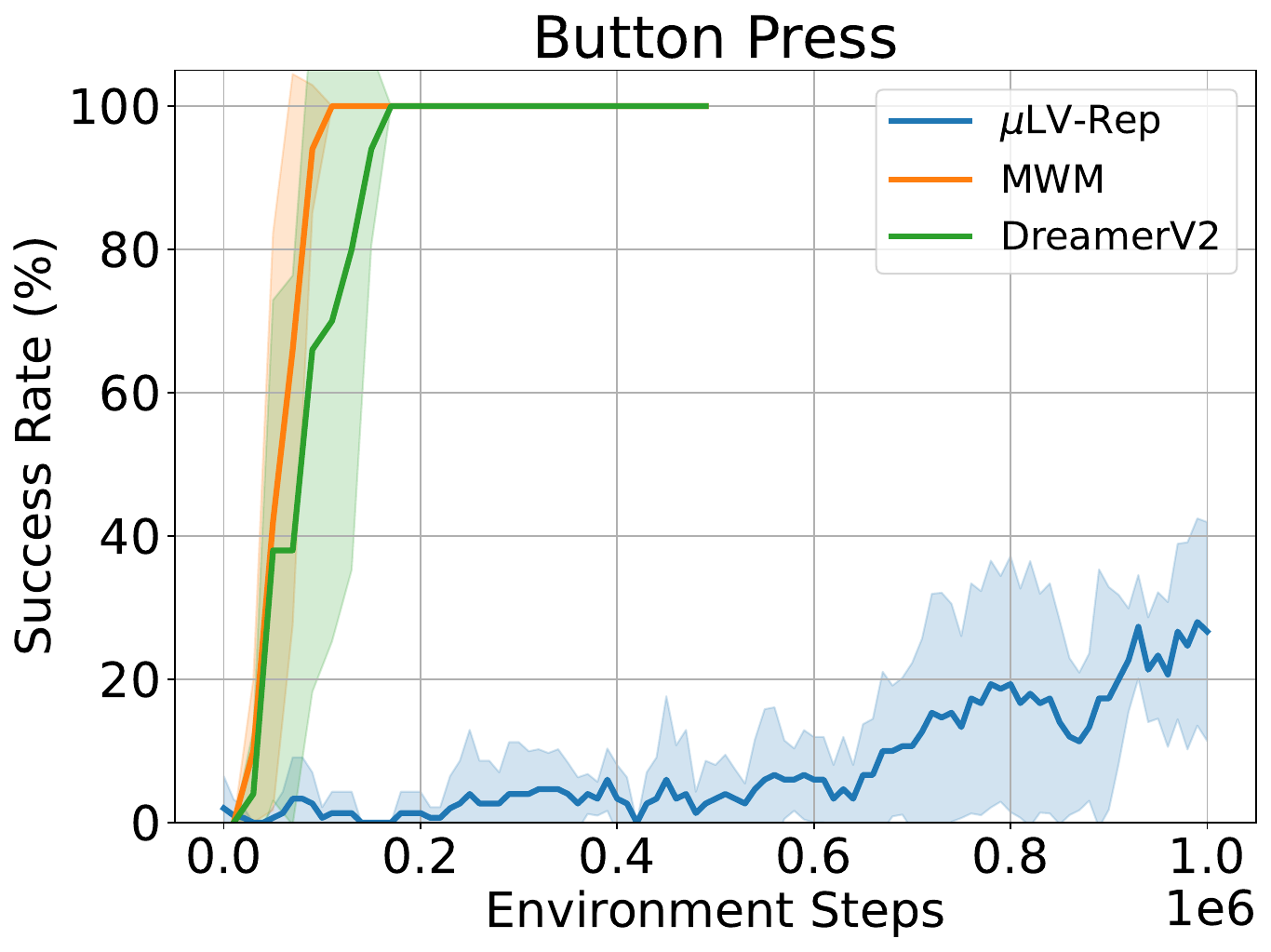}}
    \vskip -0.15in
    \subfigure{\includegraphics[width=0.18\textwidth]{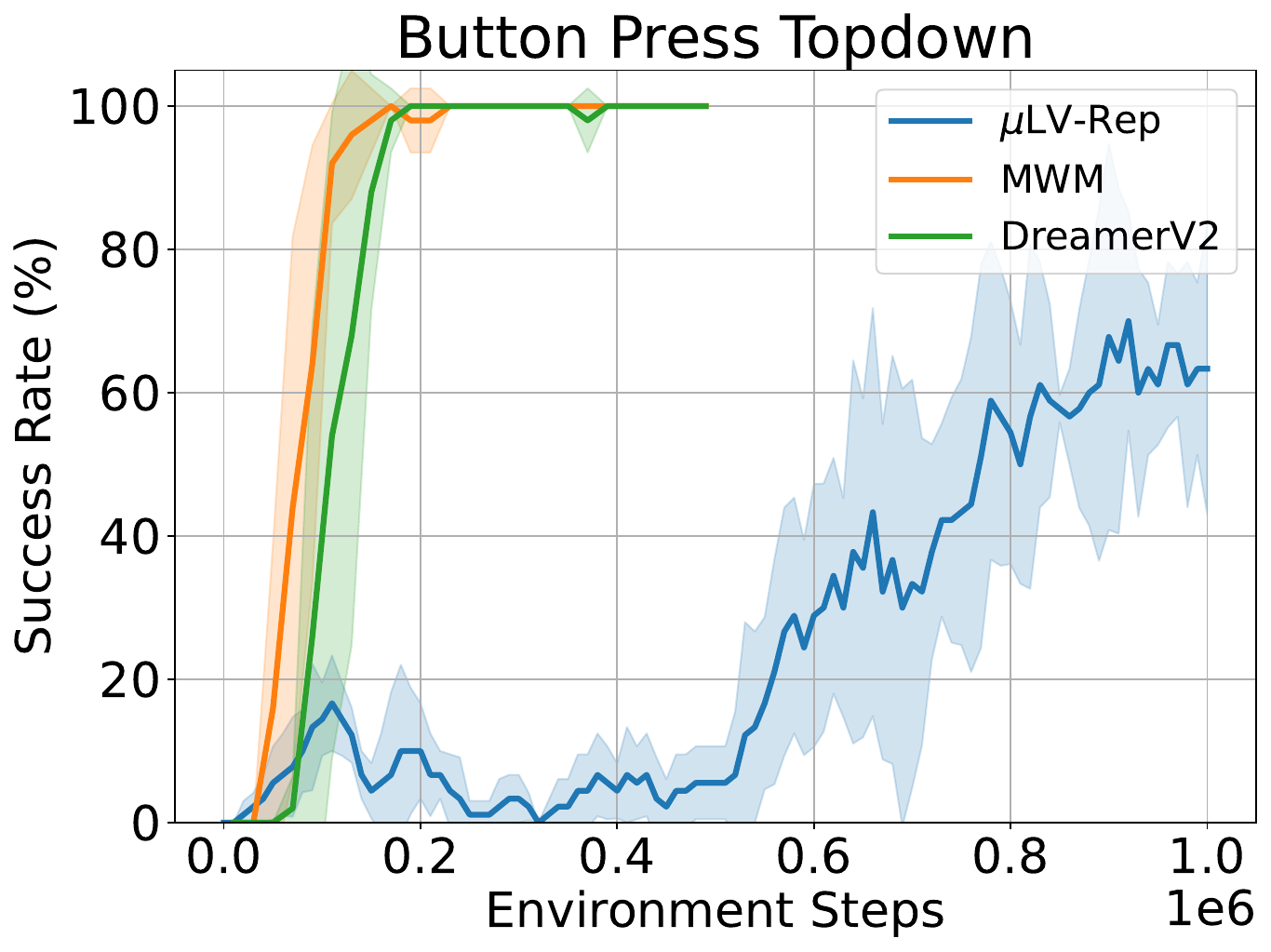}}
    \subfigure{\includegraphics[width=0.18\textwidth]{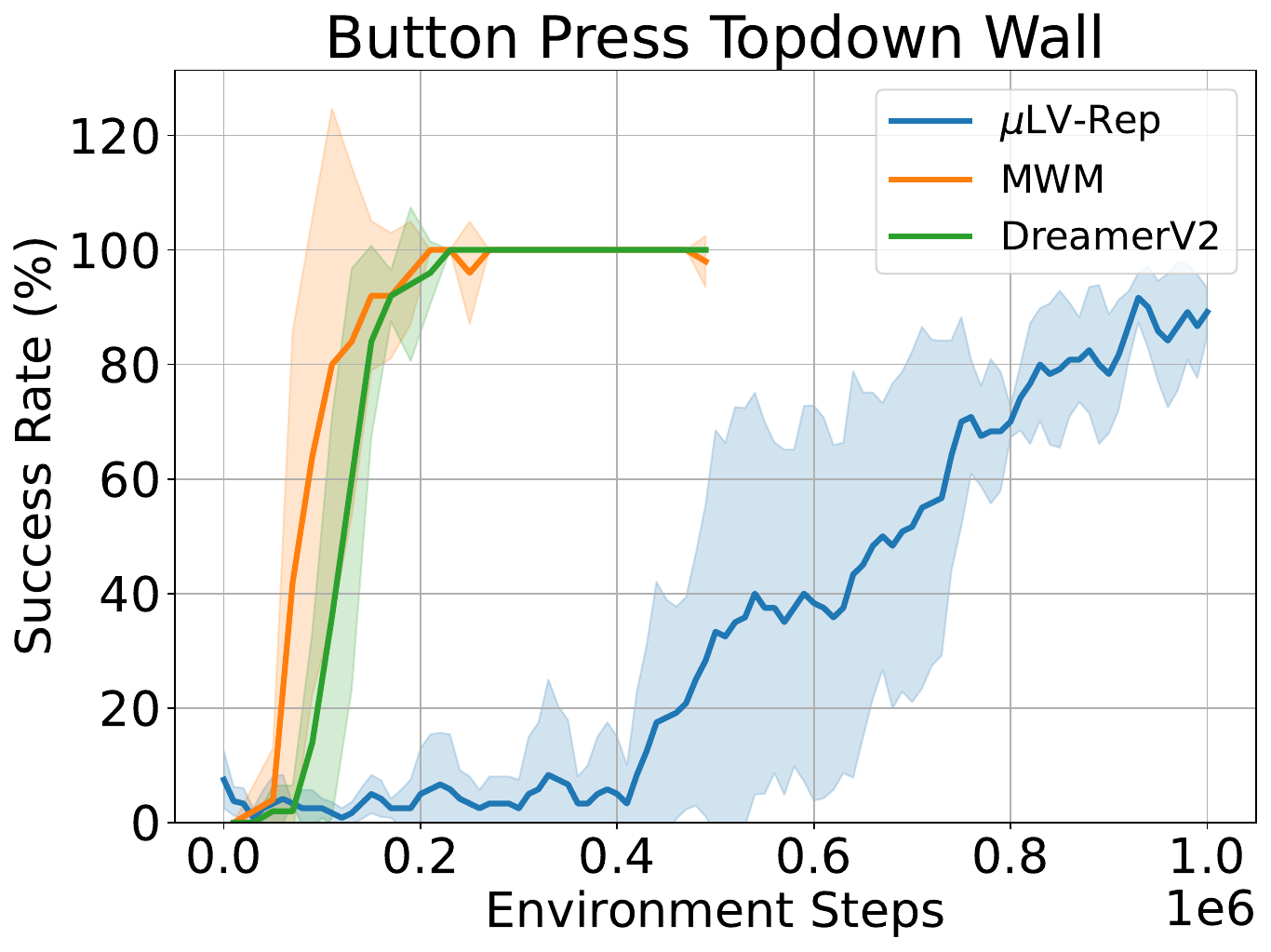}}
    \subfigure{\includegraphics[width=0.18\textwidth]{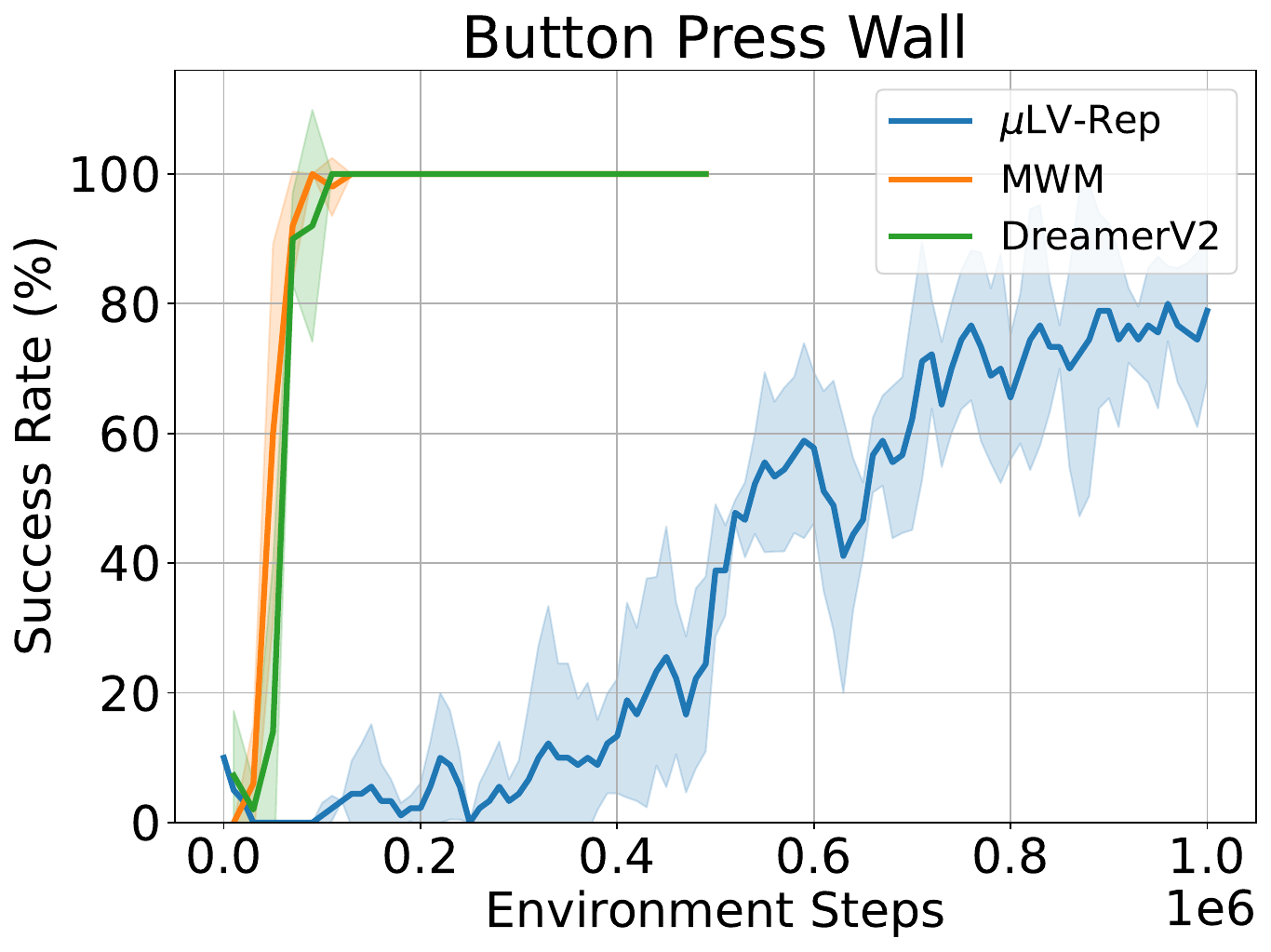}}
    \subfigure{\includegraphics[width=0.18\textwidth]{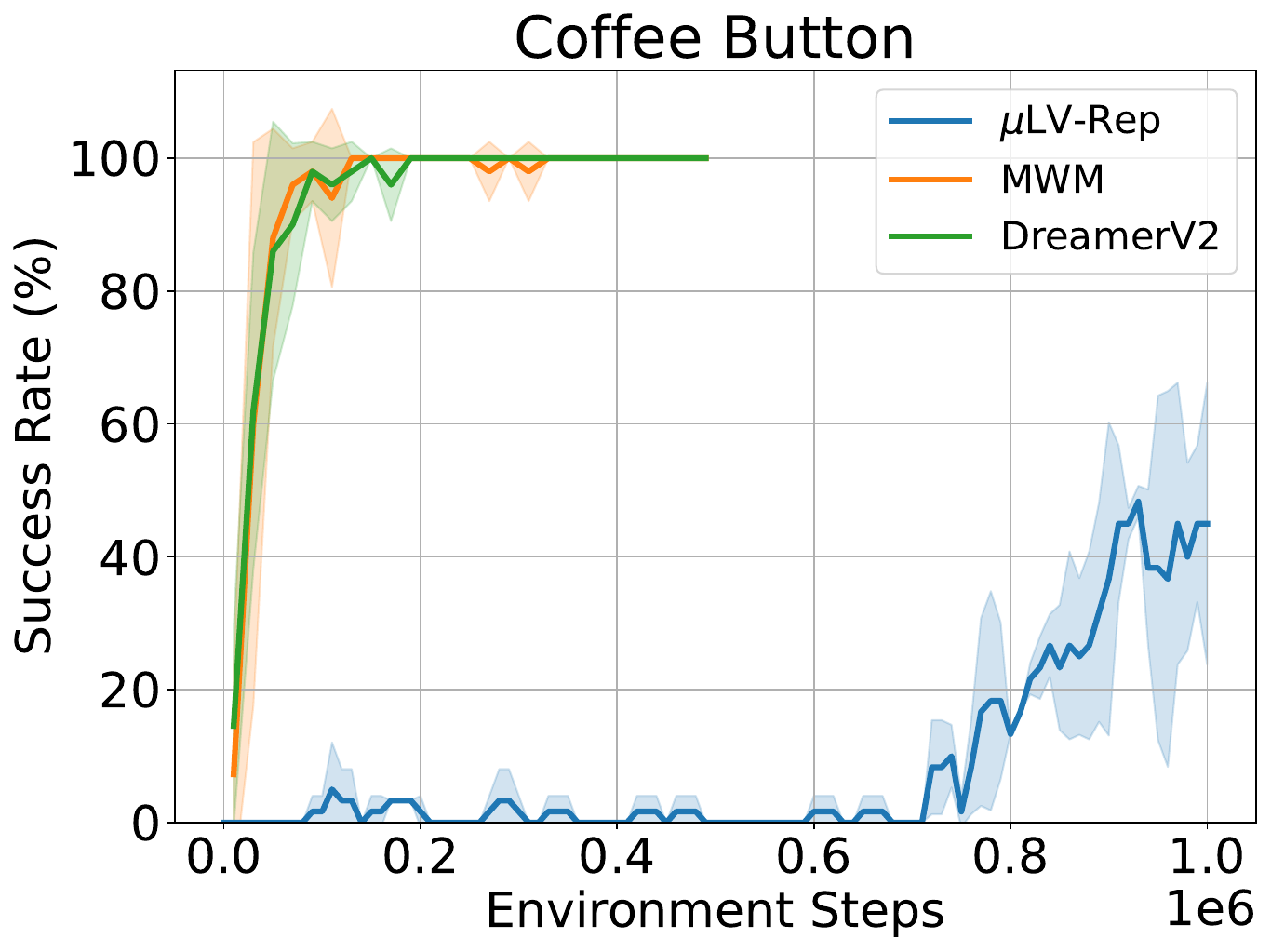}}
    \subfigure{\includegraphics[width=0.18\textwidth]{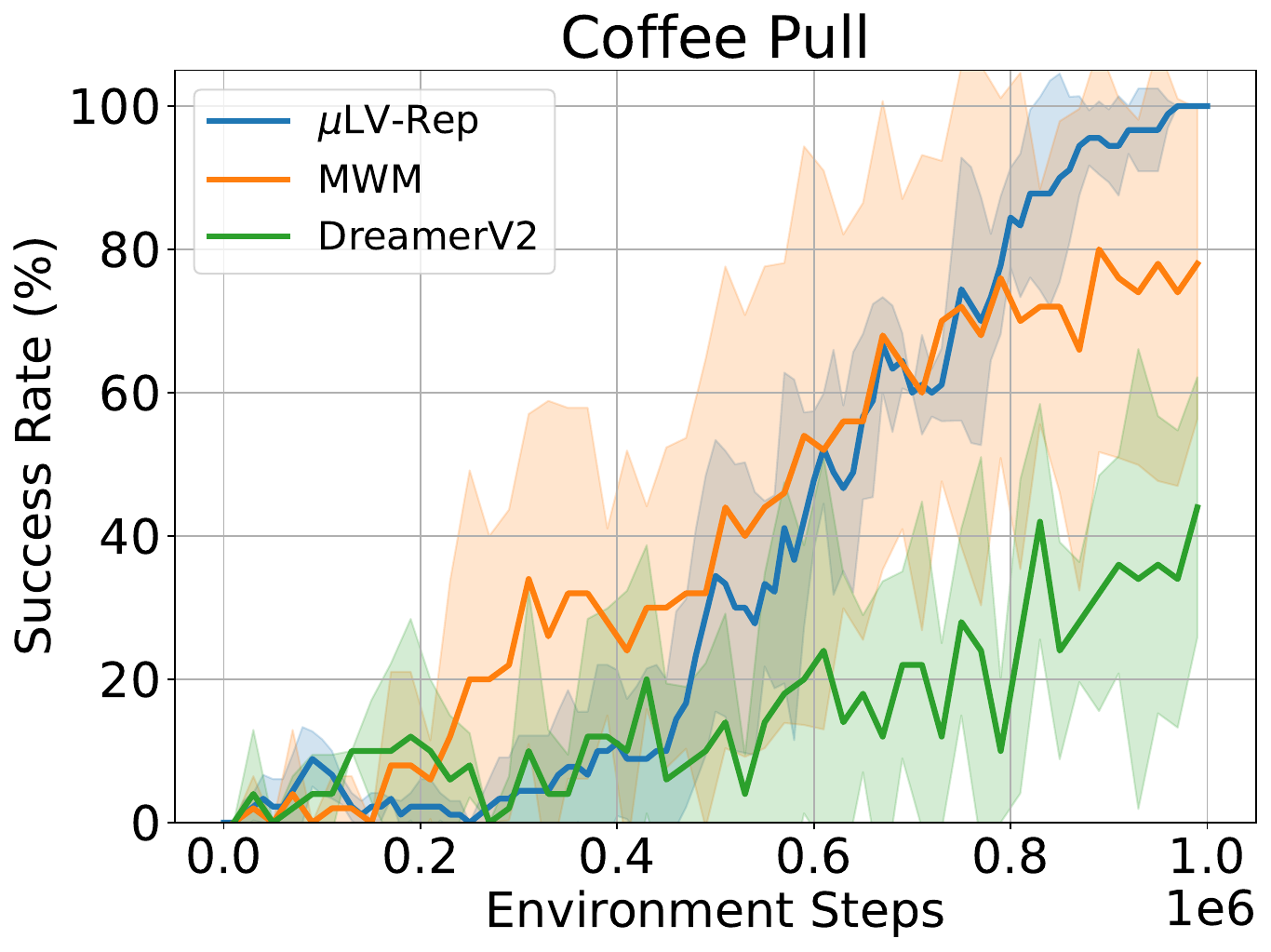}}
    \vskip -0.15in
    \subfigure{\includegraphics[width=0.18\textwidth]{pic/metaworld_coffee_push_performance.pdf}}
    \subfigure{\includegraphics[width=0.18\textwidth]{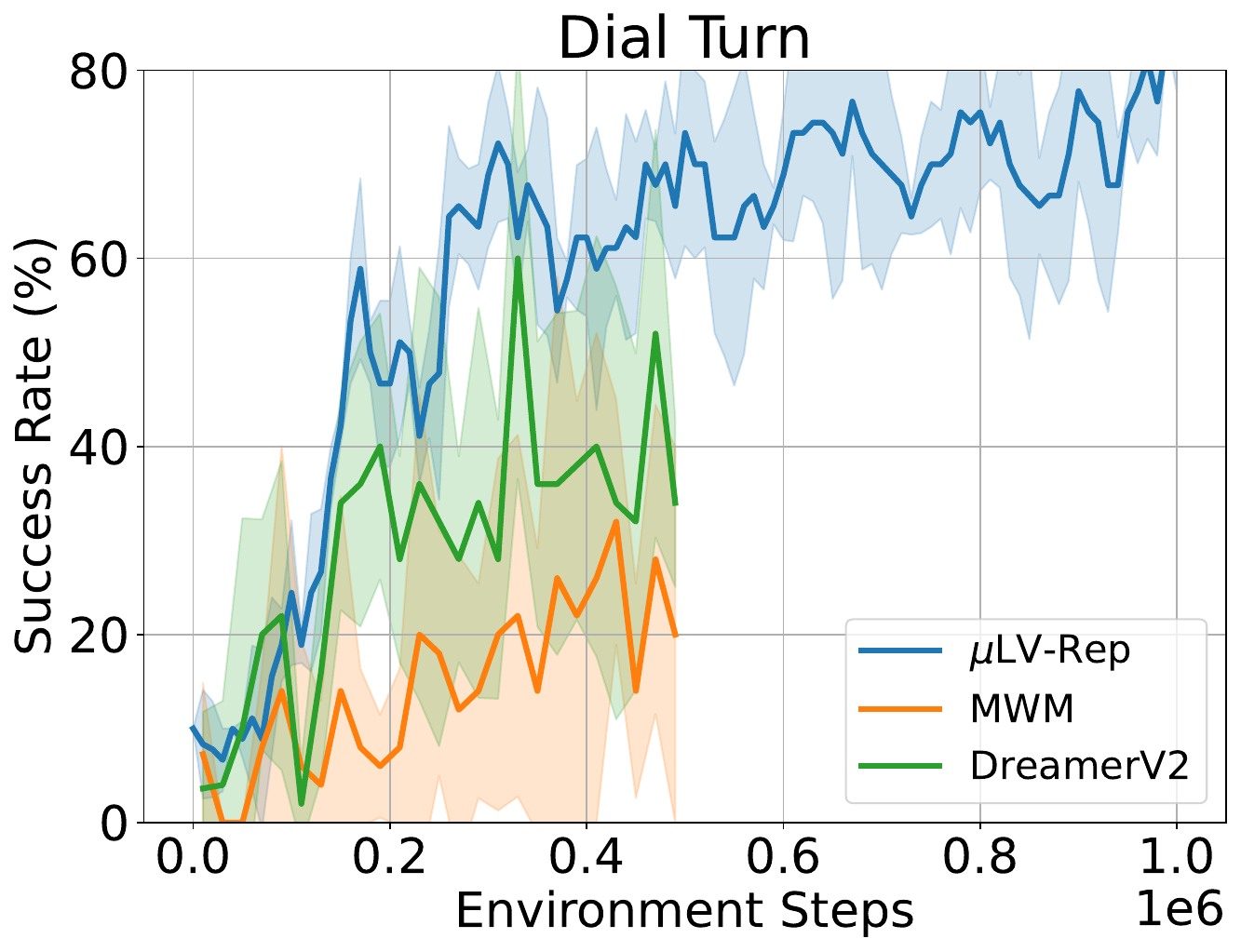}}    
    \subfigure{\includegraphics[width=0.18\textwidth]{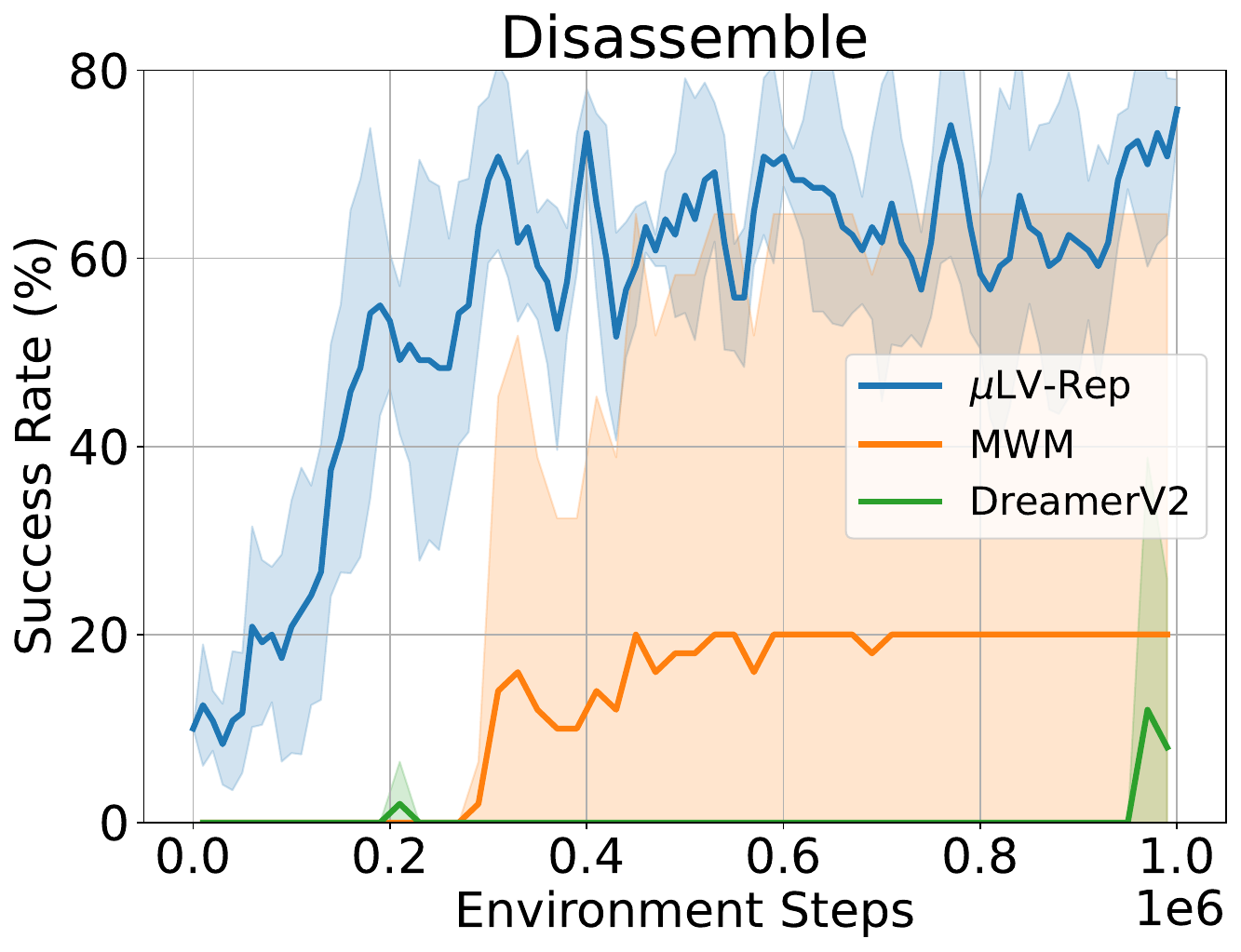}}
    \subfigure{\includegraphics[width=0.18\textwidth]{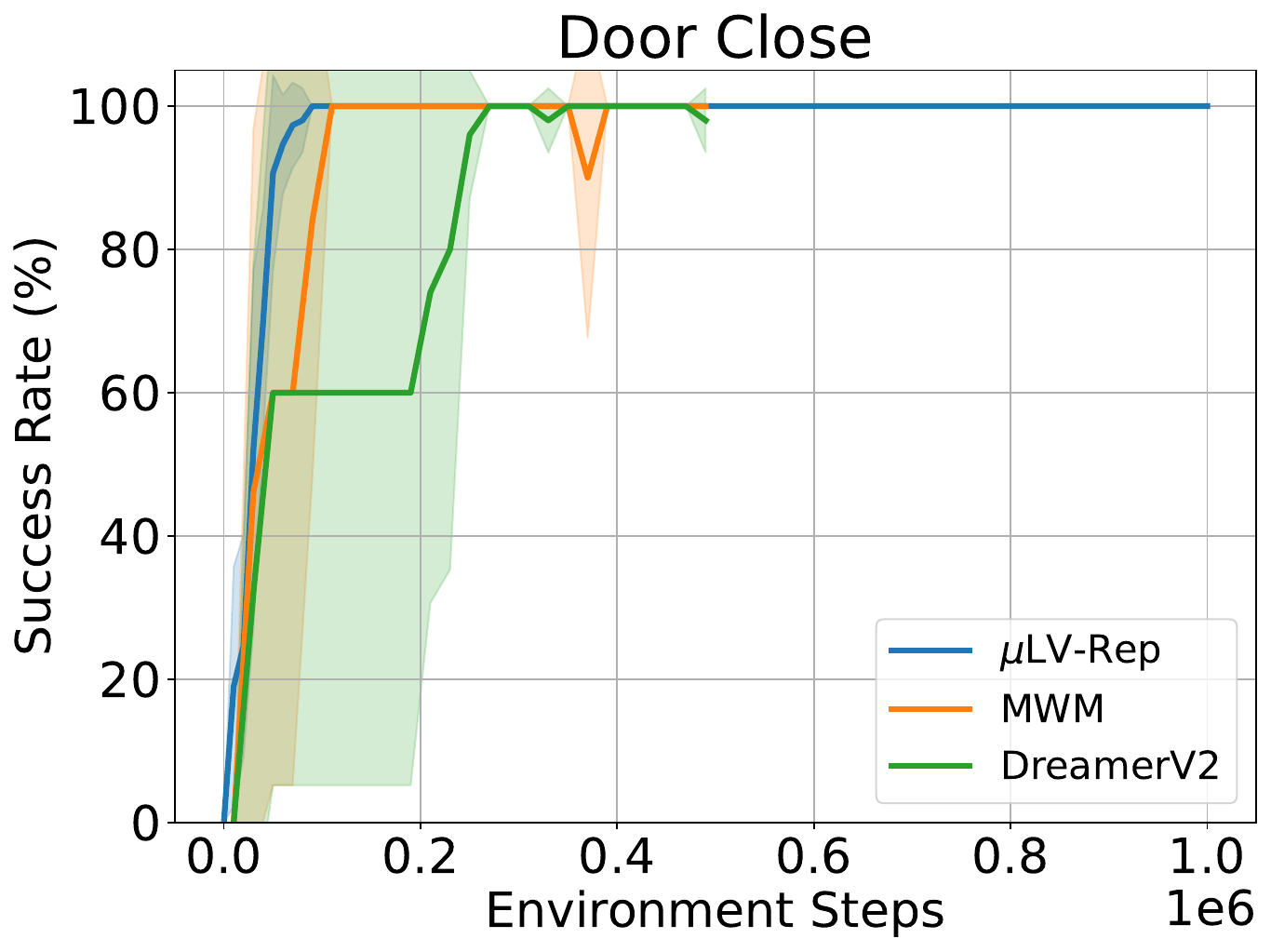}}
    \subfigure{\includegraphics[width=0.18\textwidth]{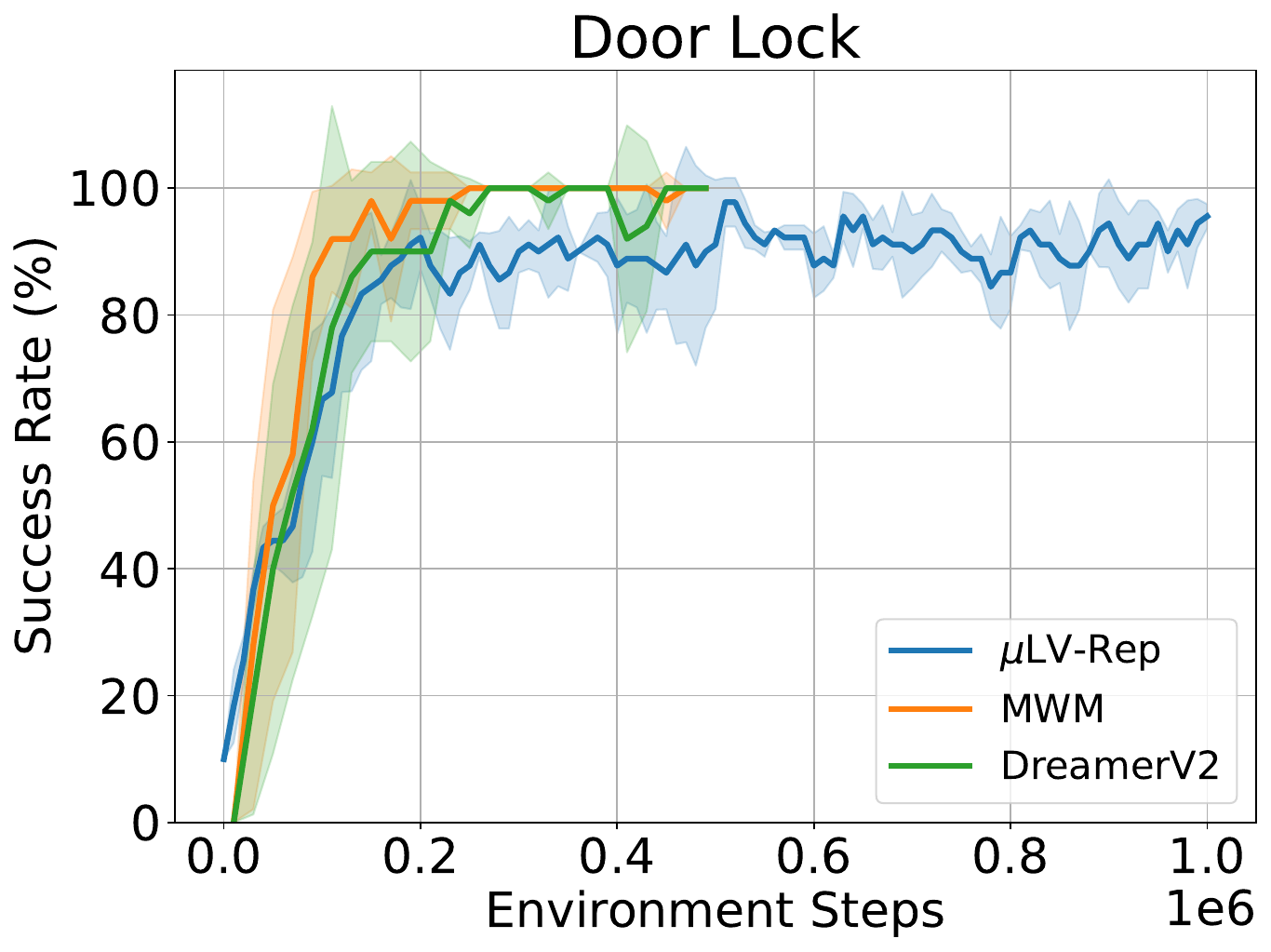}}
    \vskip -0.15in
    \subfigure{\includegraphics[width=0.18\textwidth]{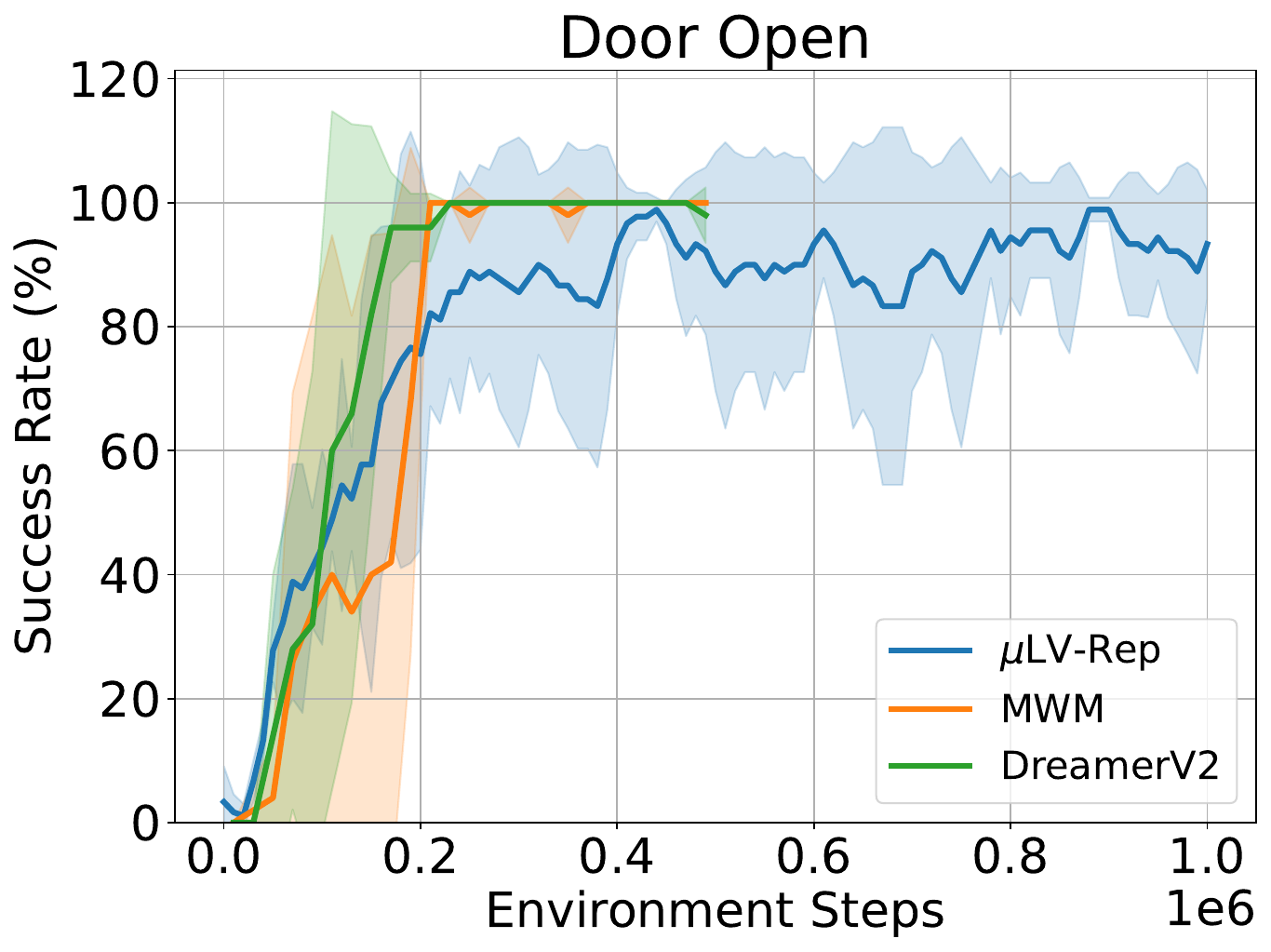}}
    \subfigure{\includegraphics[width=0.18\textwidth]{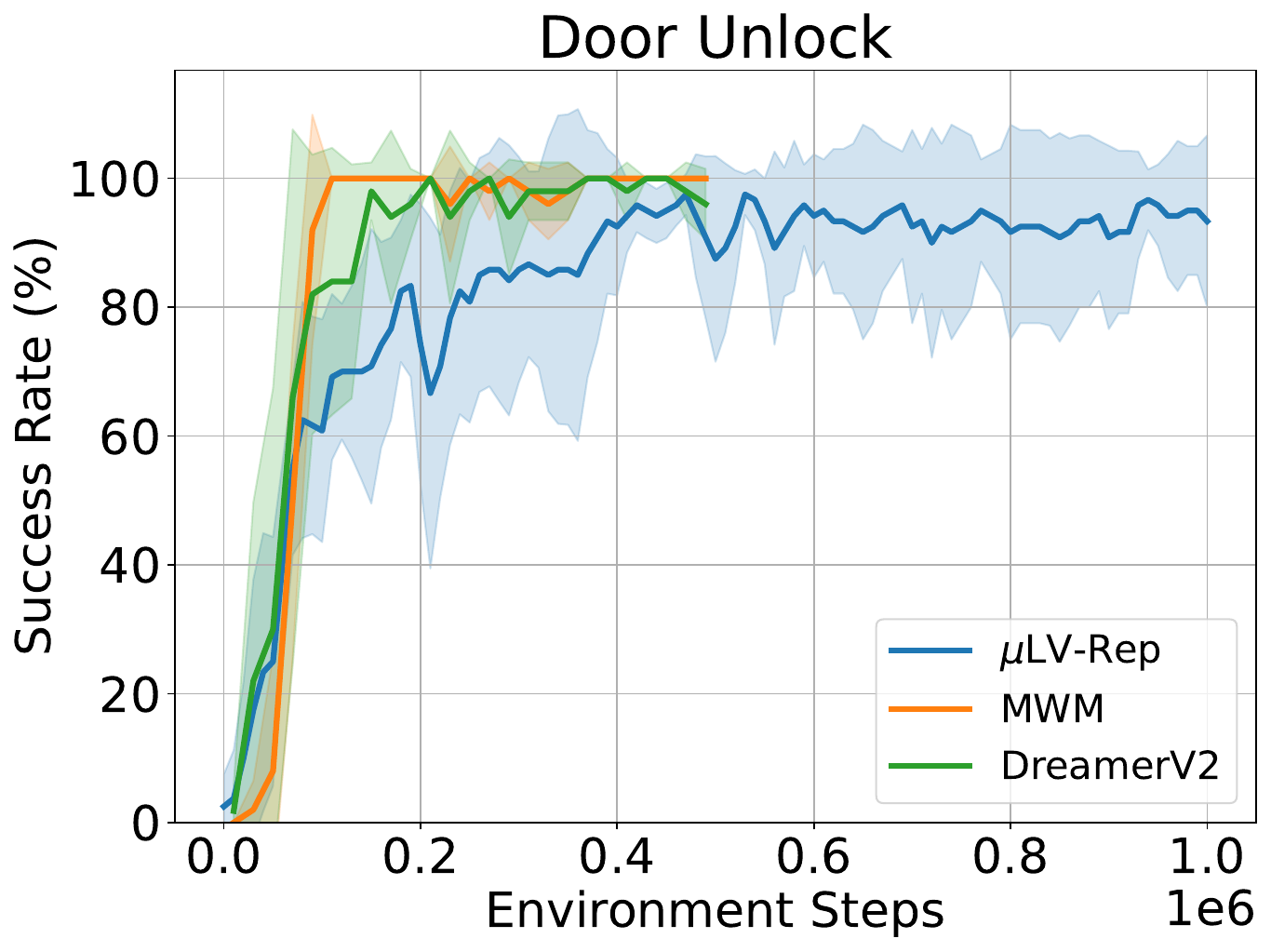}}
    \subfigure{\includegraphics[width=0.18\textwidth]{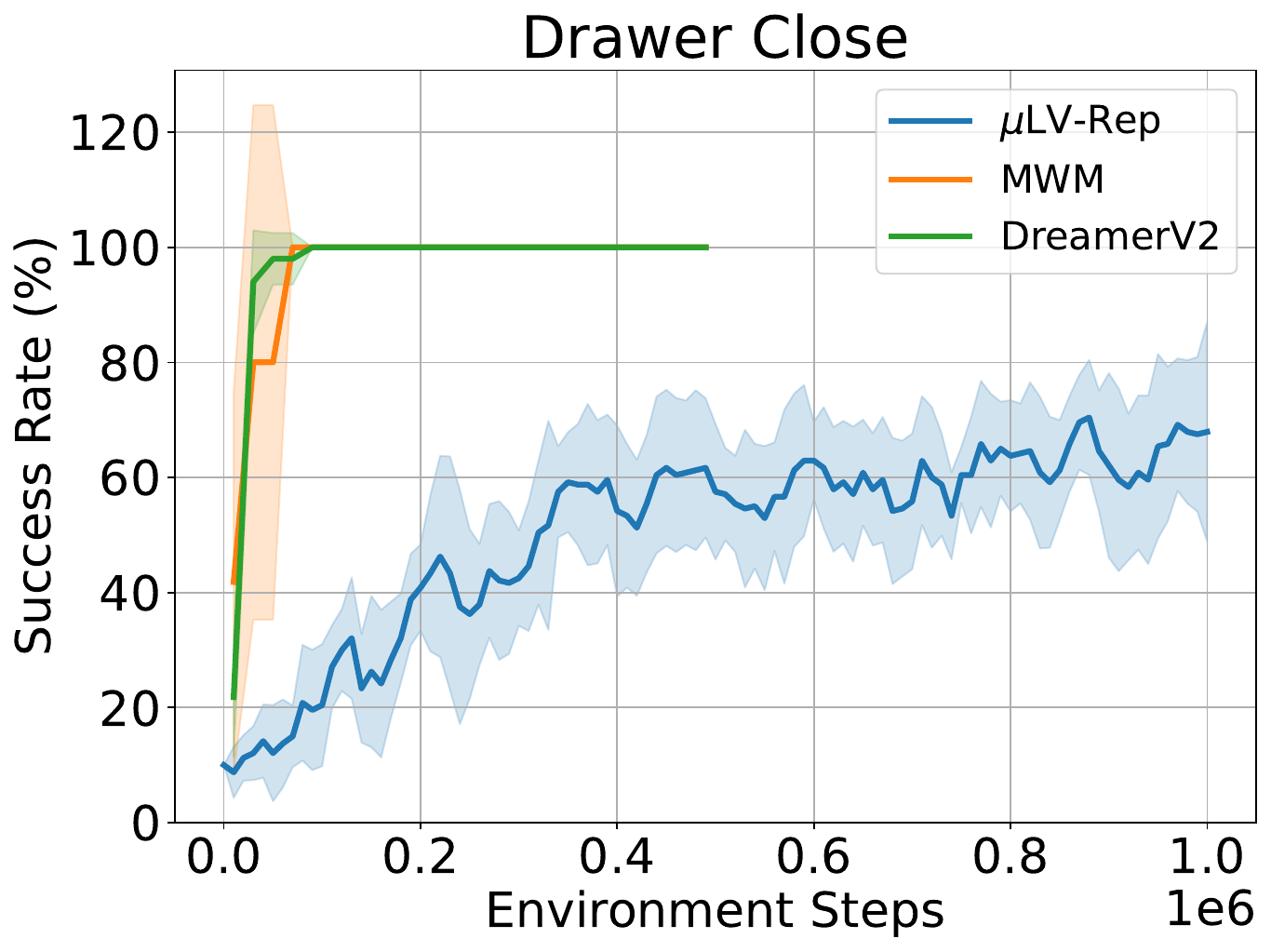}}
    \subfigure{\includegraphics[width=0.18\textwidth]{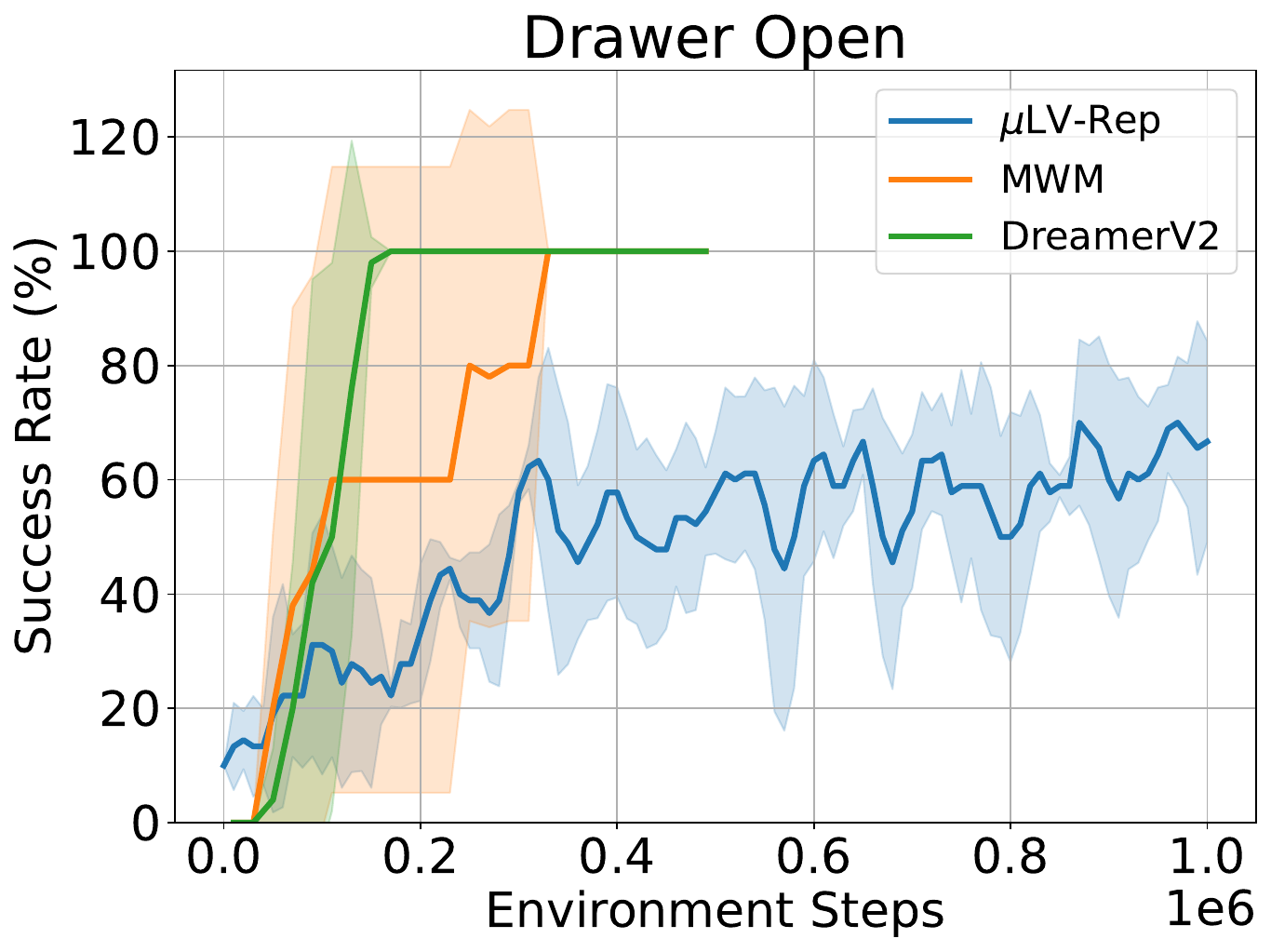}}
    \subfigure{\includegraphics[width=0.18\textwidth]{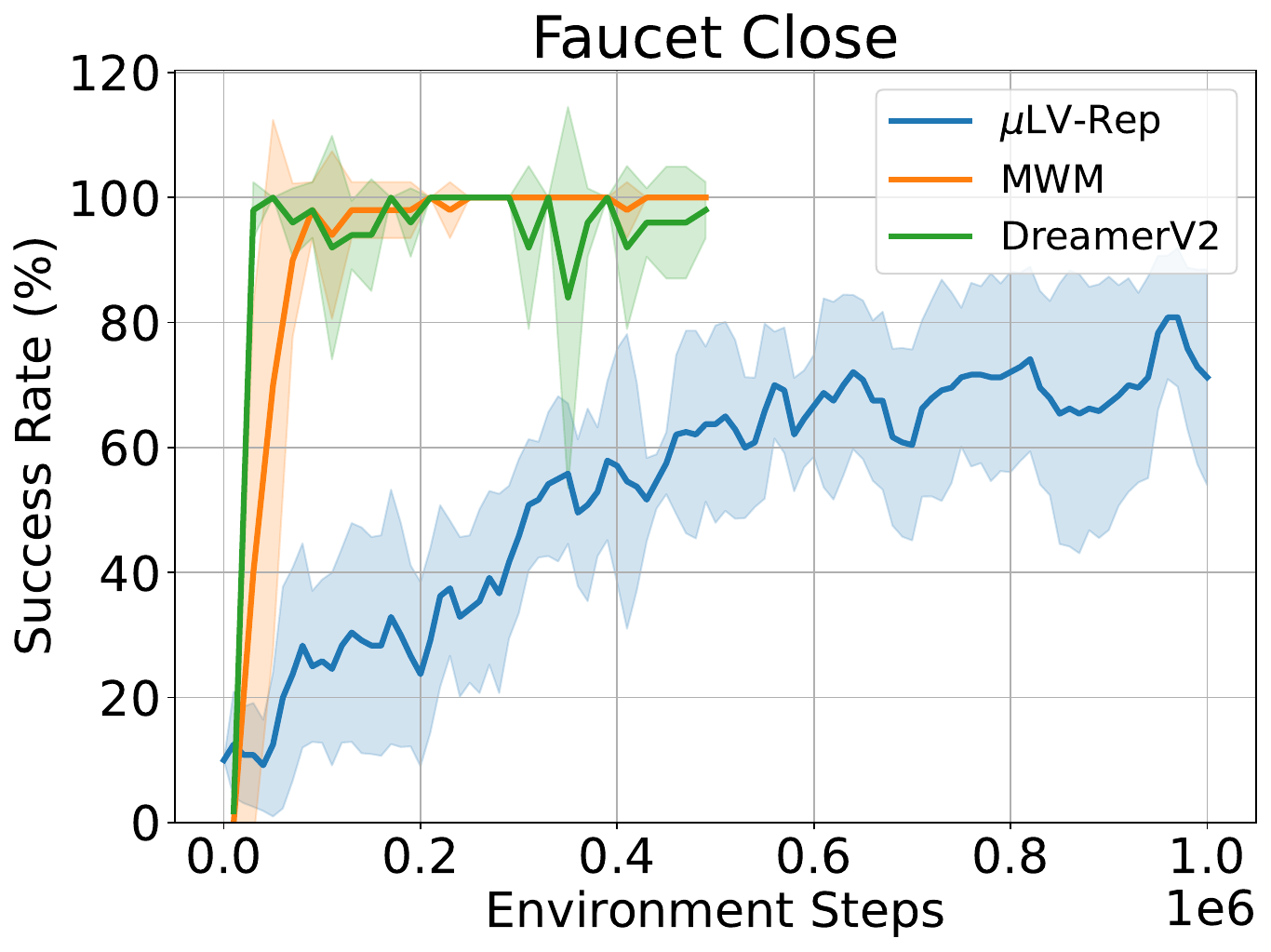}}
    \vskip -0.15in
    \subfigure{\includegraphics[width=0.18\textwidth]{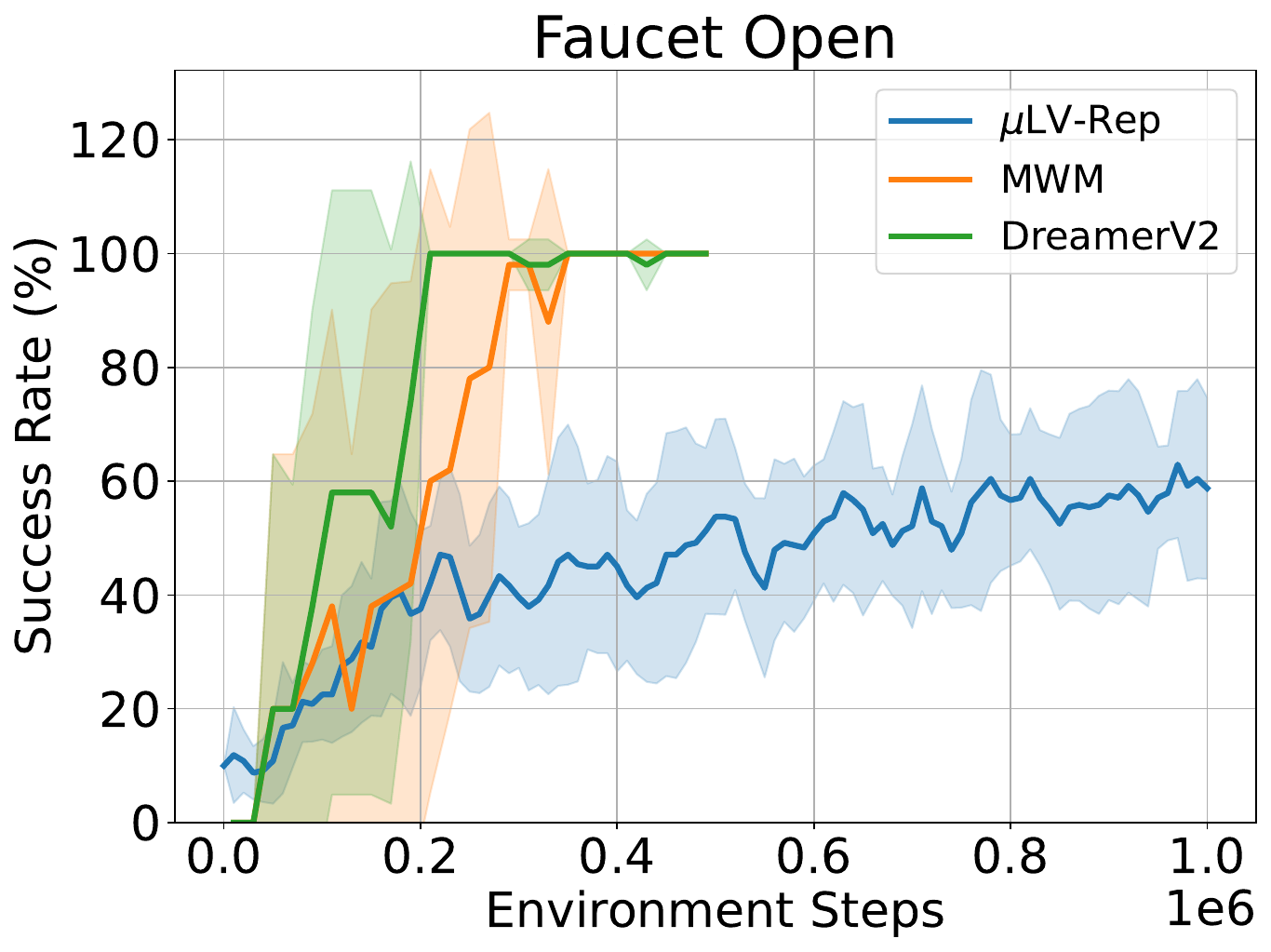}}
    \subfigure{\includegraphics[width=0.18\textwidth]{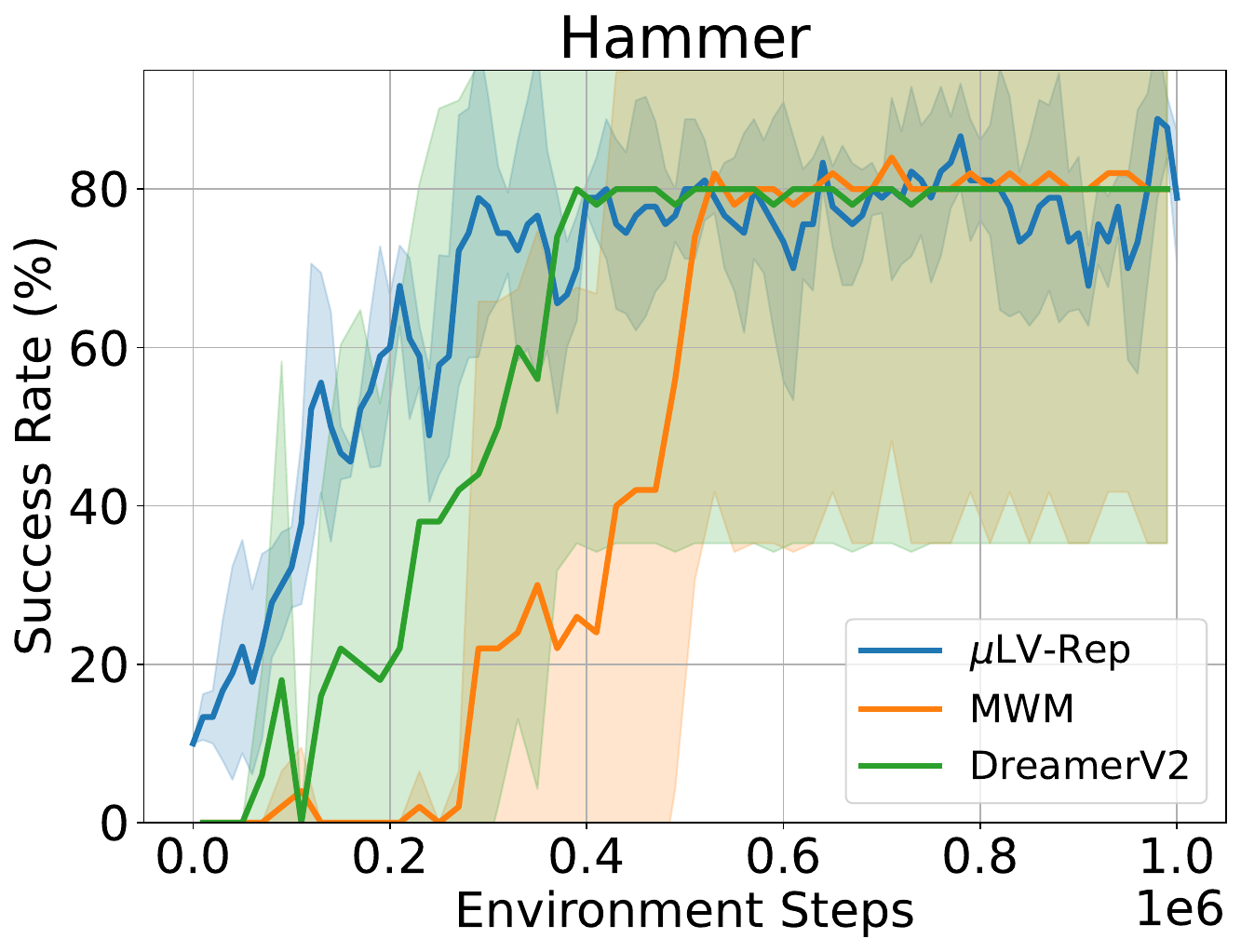}}
    \subfigure{\includegraphics[width=0.18\textwidth]{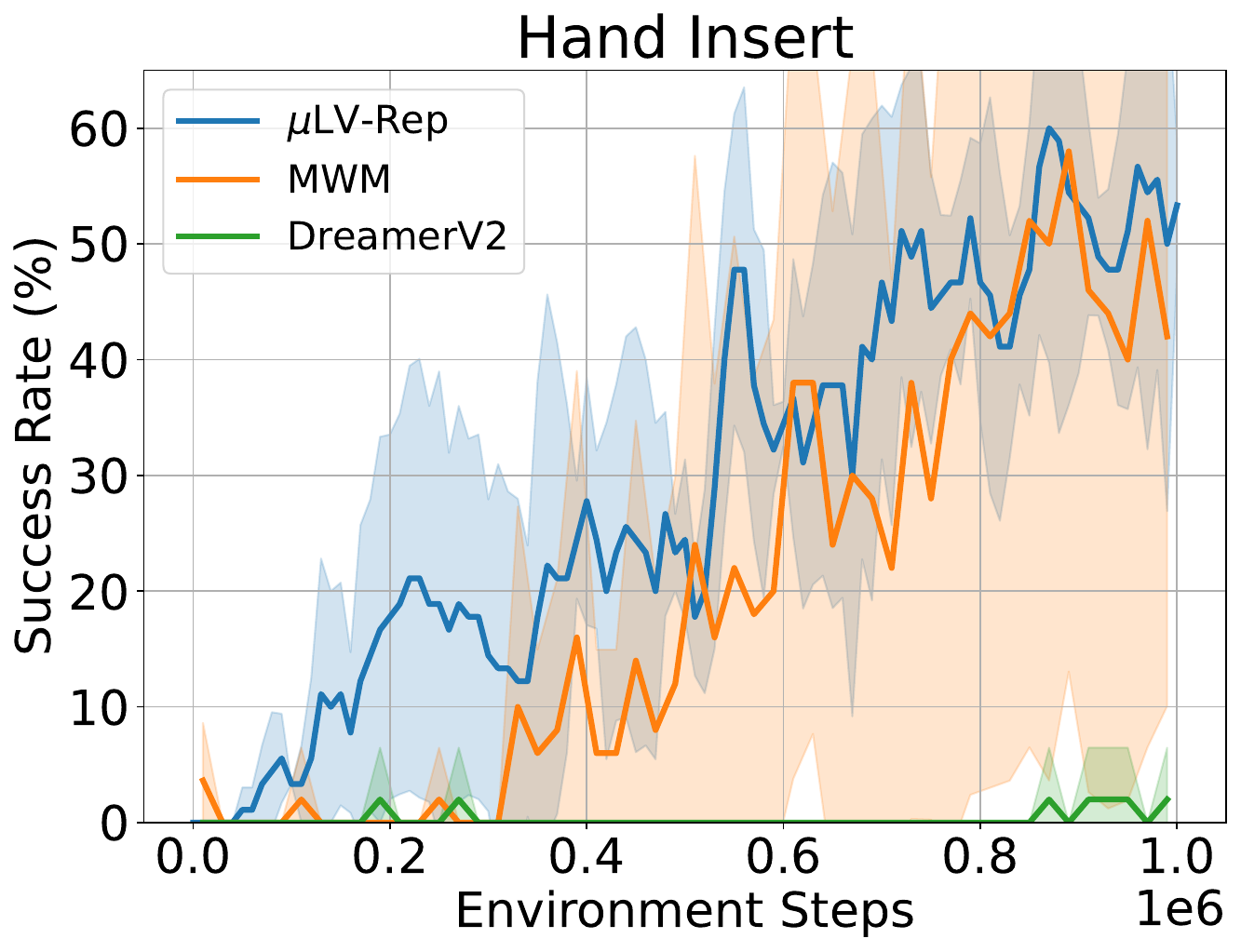}}
    \subfigure{\includegraphics[width=0.18\textwidth]{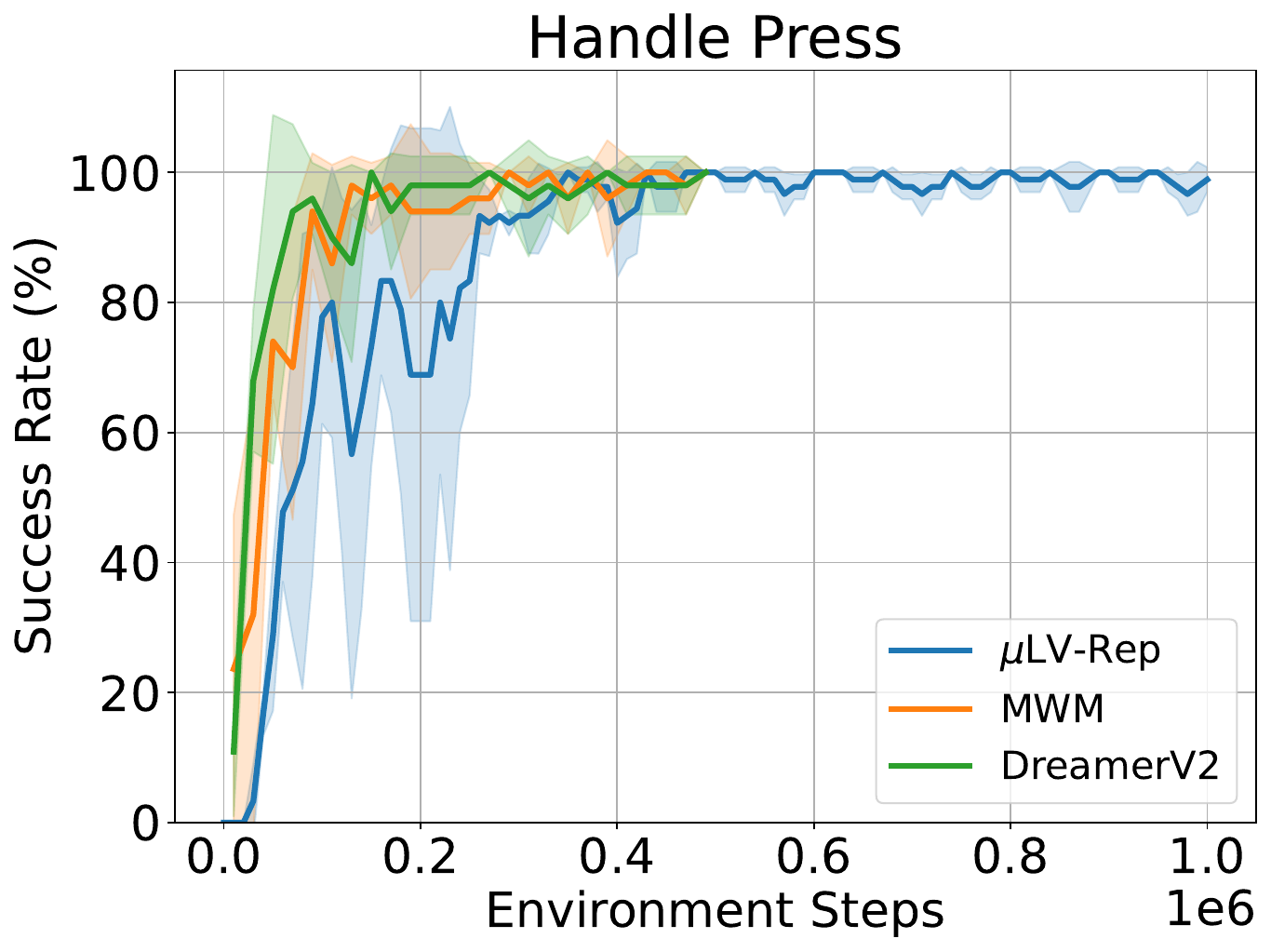}}
    \subfigure{\includegraphics[width=0.18\textwidth]{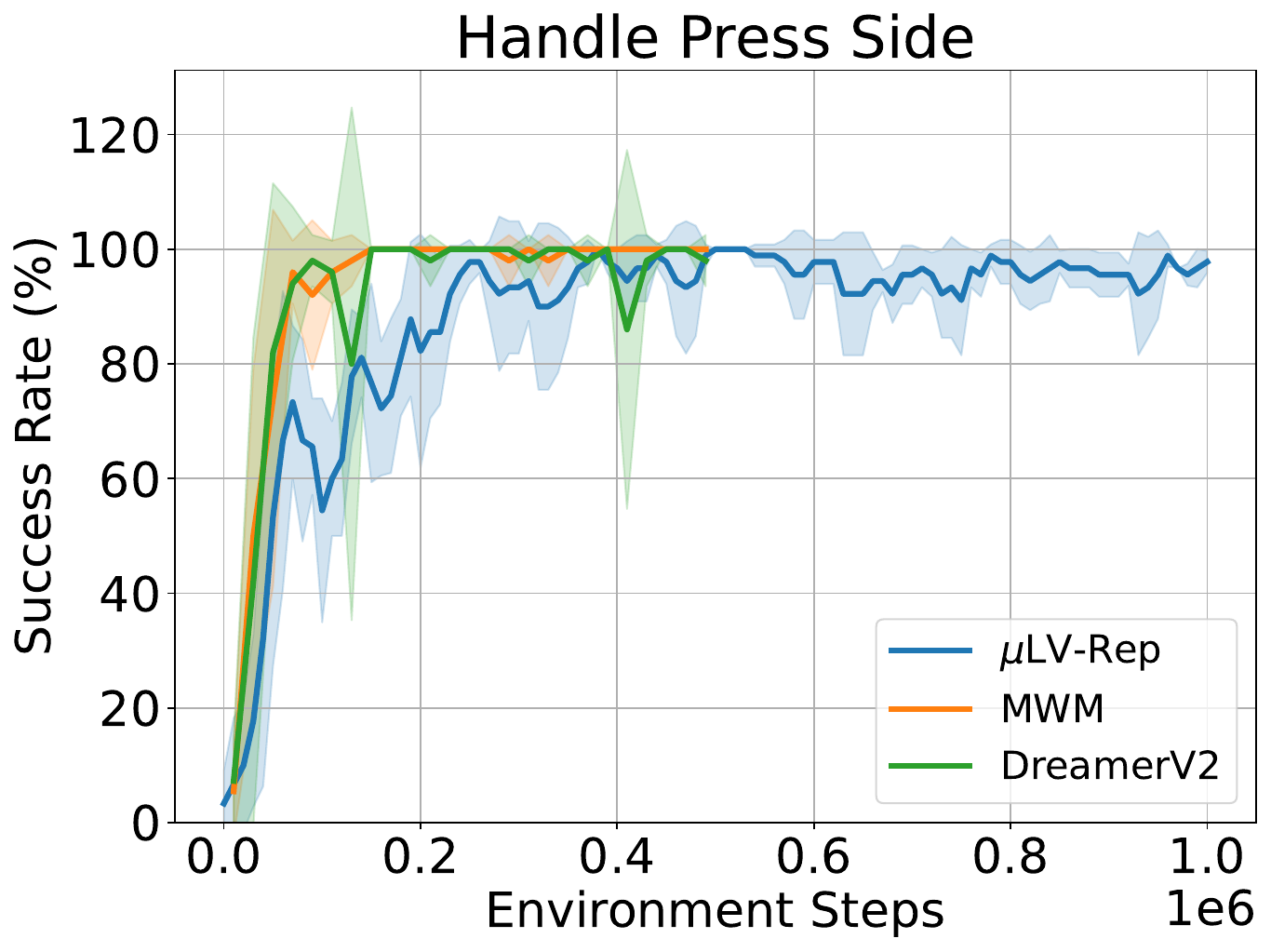}}
    \vskip -0.15in
    \subfigure{\includegraphics[width=0.18\textwidth]{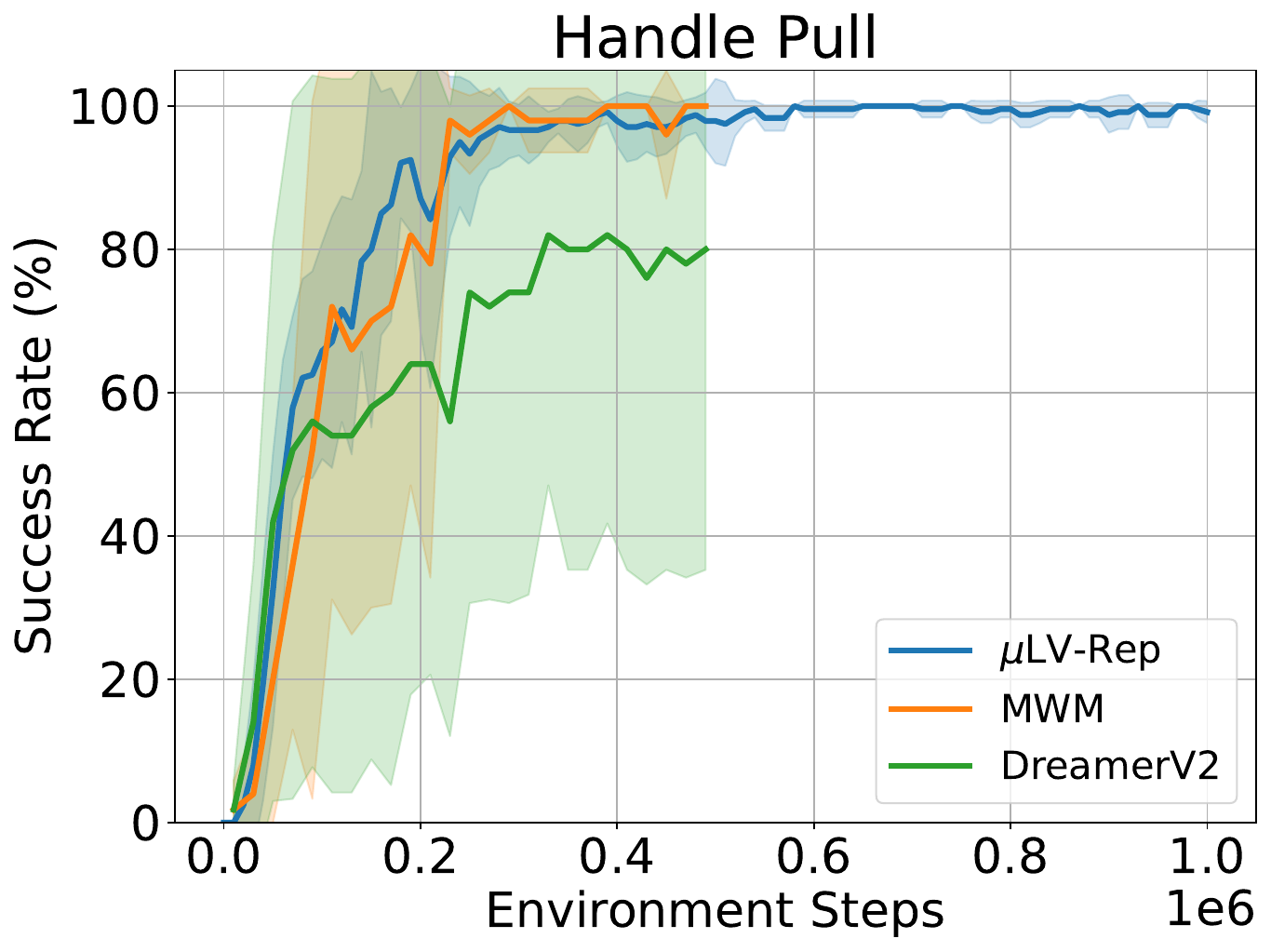}}
    \subfigure{\includegraphics[width=0.18\textwidth]{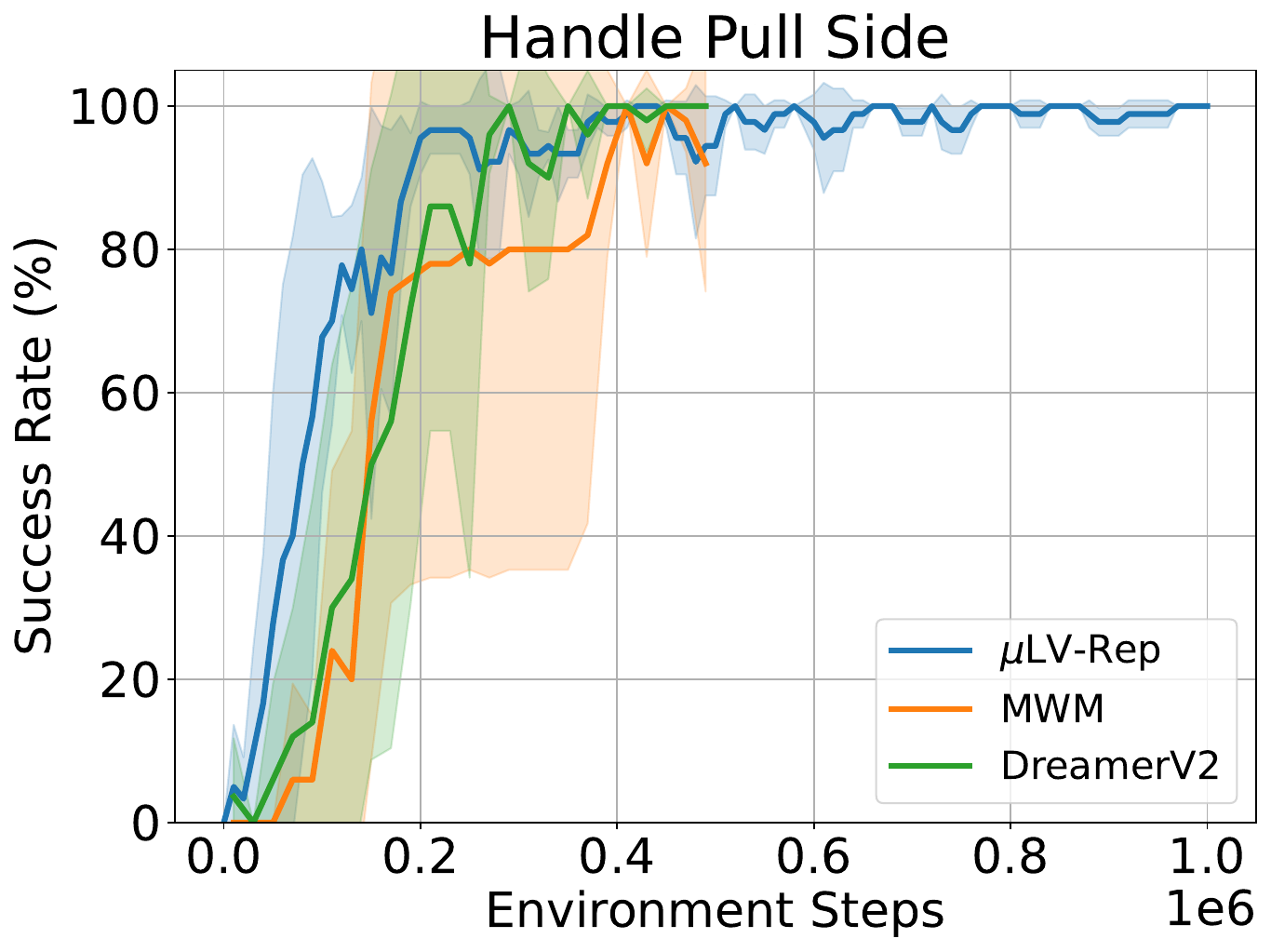}}
    \subfigure{\includegraphics[width=0.18\textwidth]{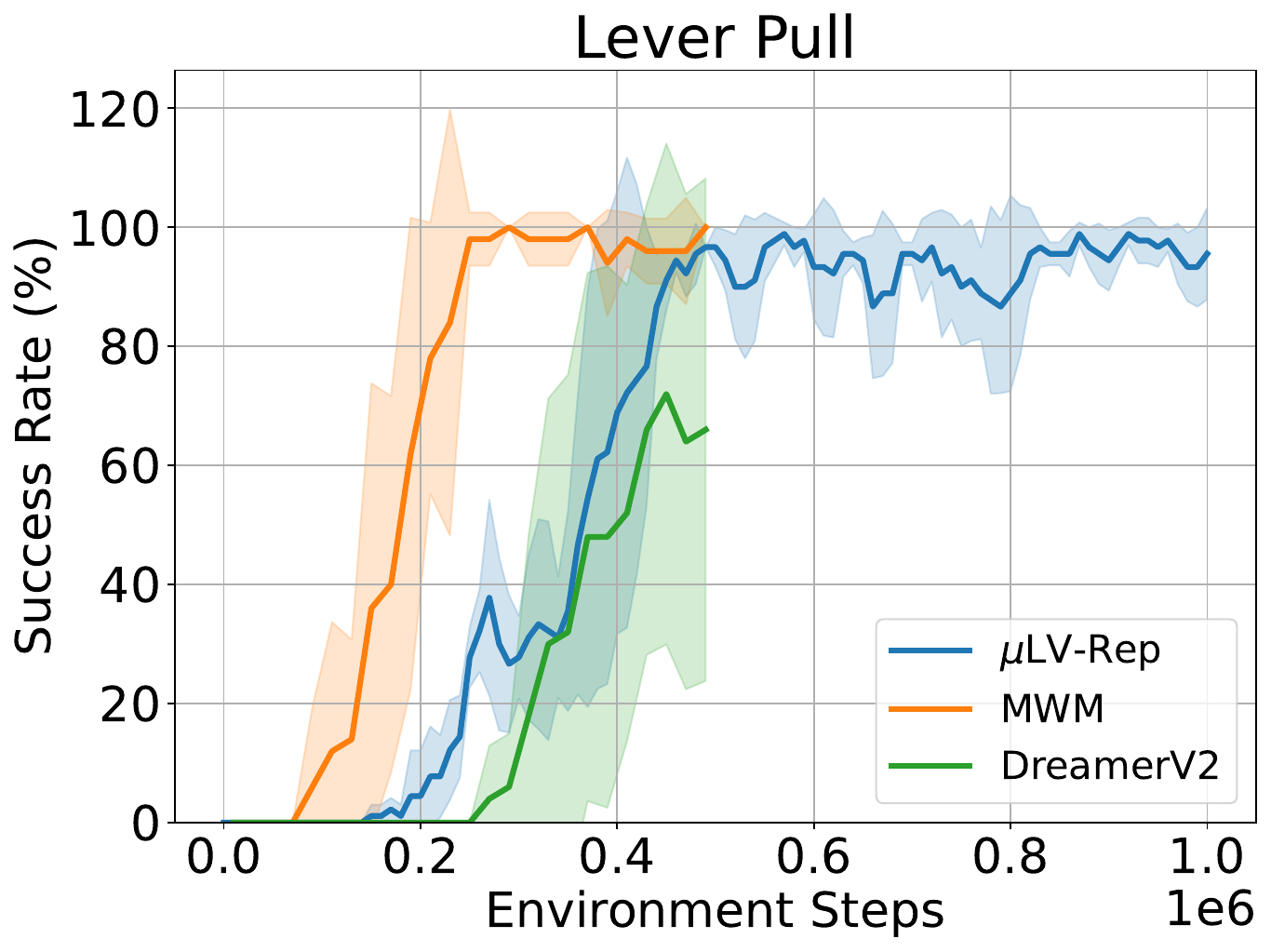}}
    \subfigure{\includegraphics[width=0.18\textwidth]{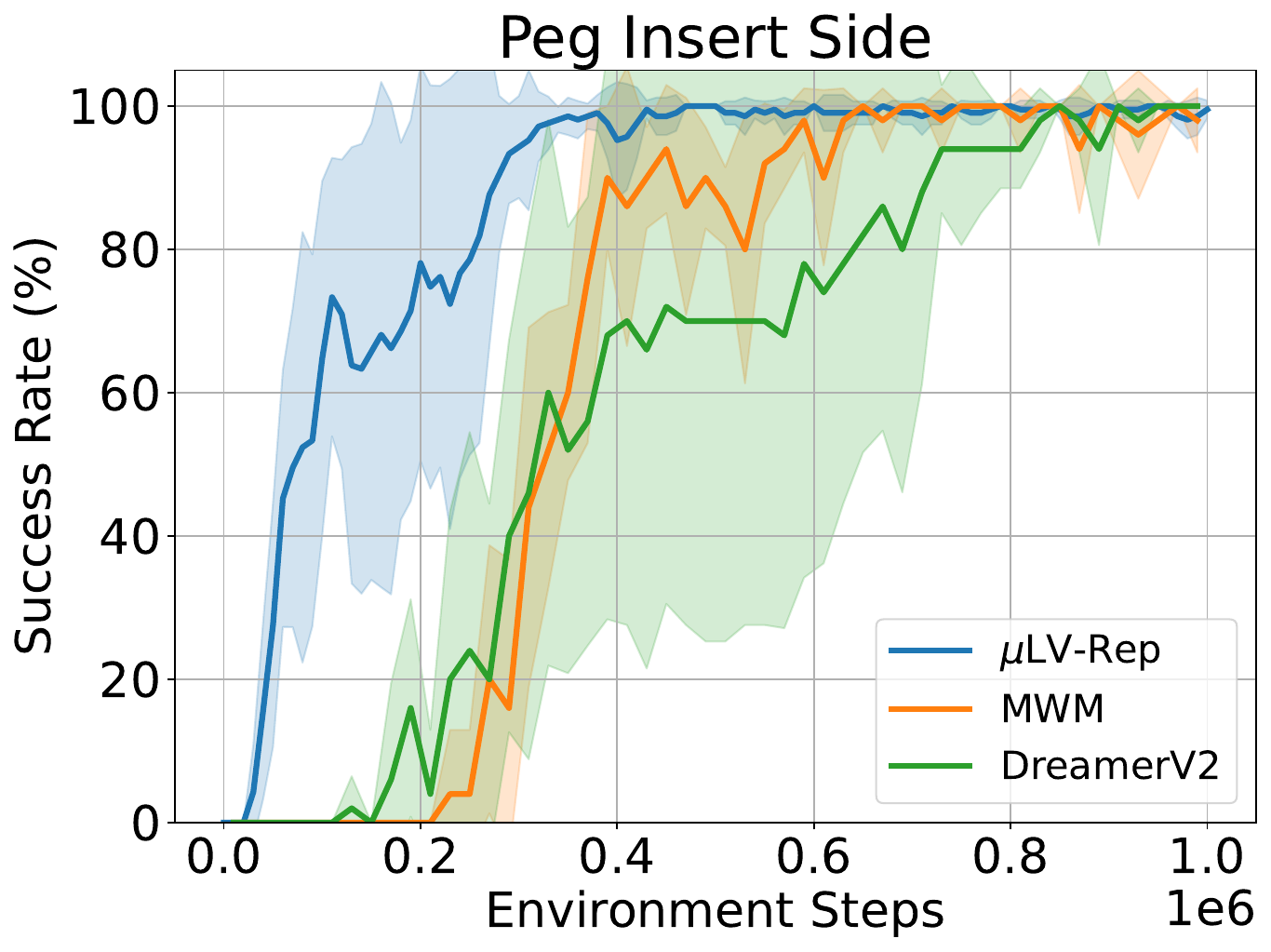}}
    \subfigure{\includegraphics[width=0.18\textwidth]{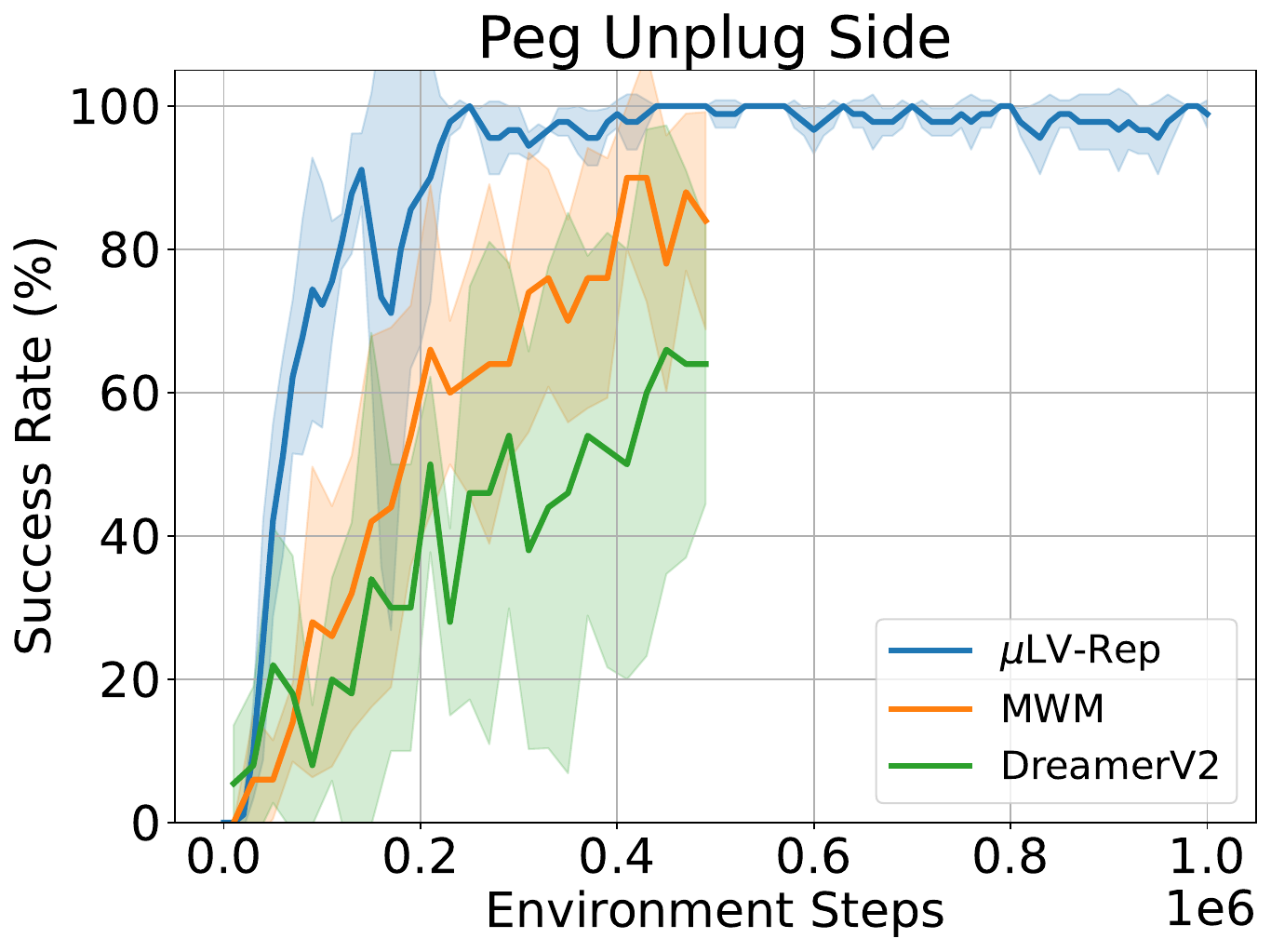}}
    \vskip -0.15in
    \subfigure{\includegraphics[width=0.18\textwidth]{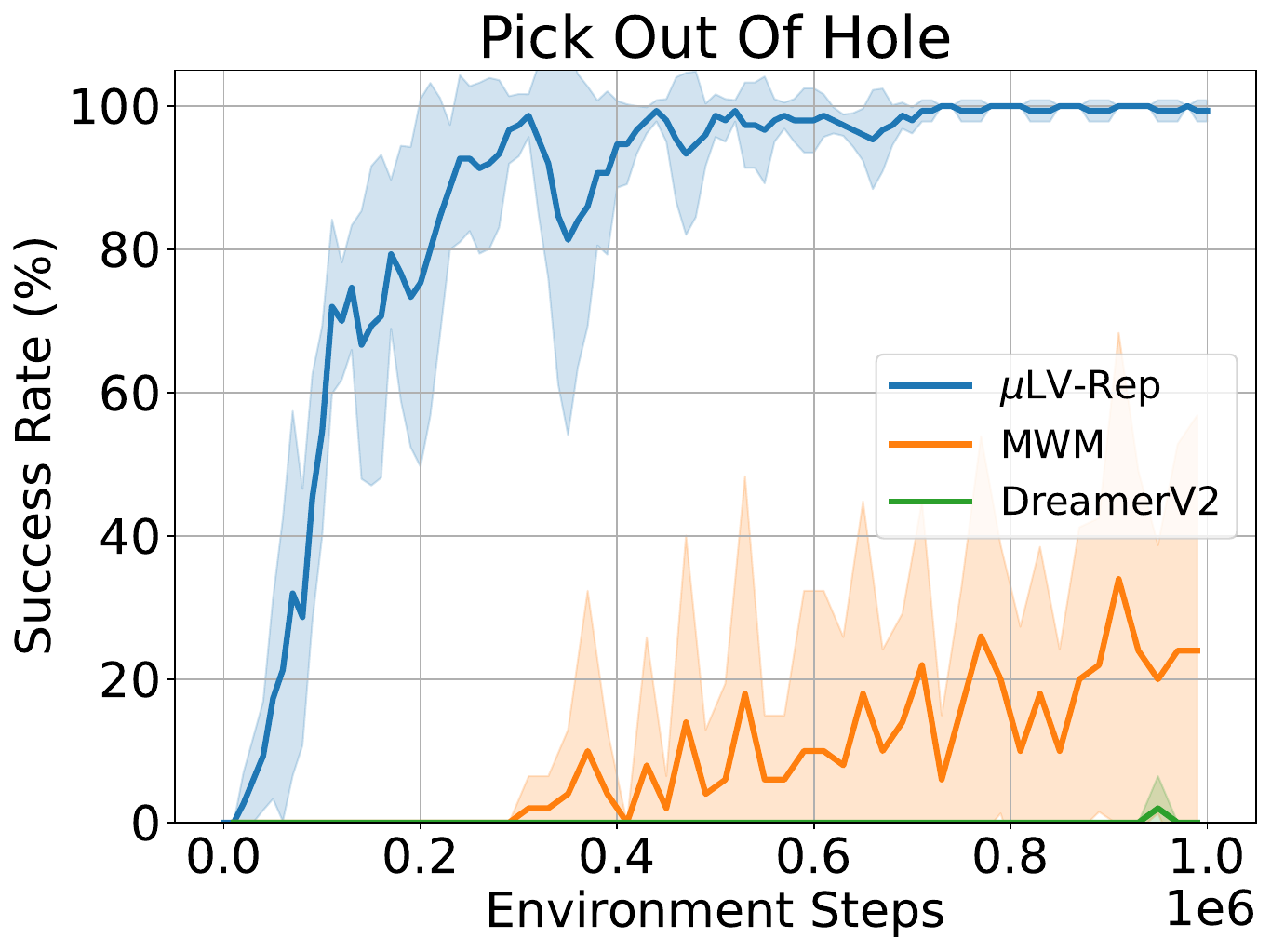}}
    \subfigure{\includegraphics[width=0.18\textwidth]{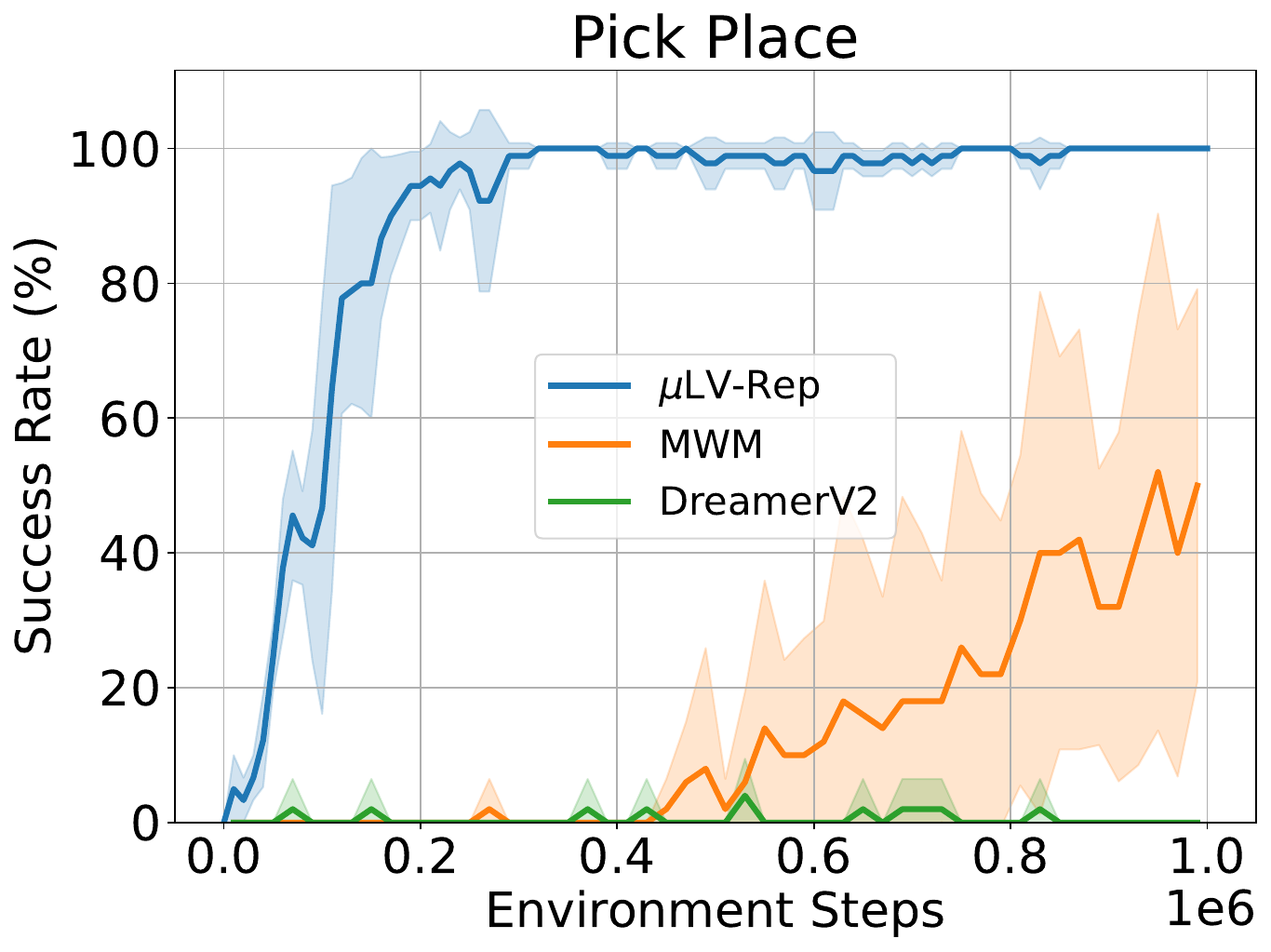}}
    \subfigure{\includegraphics[width=0.18\textwidth]{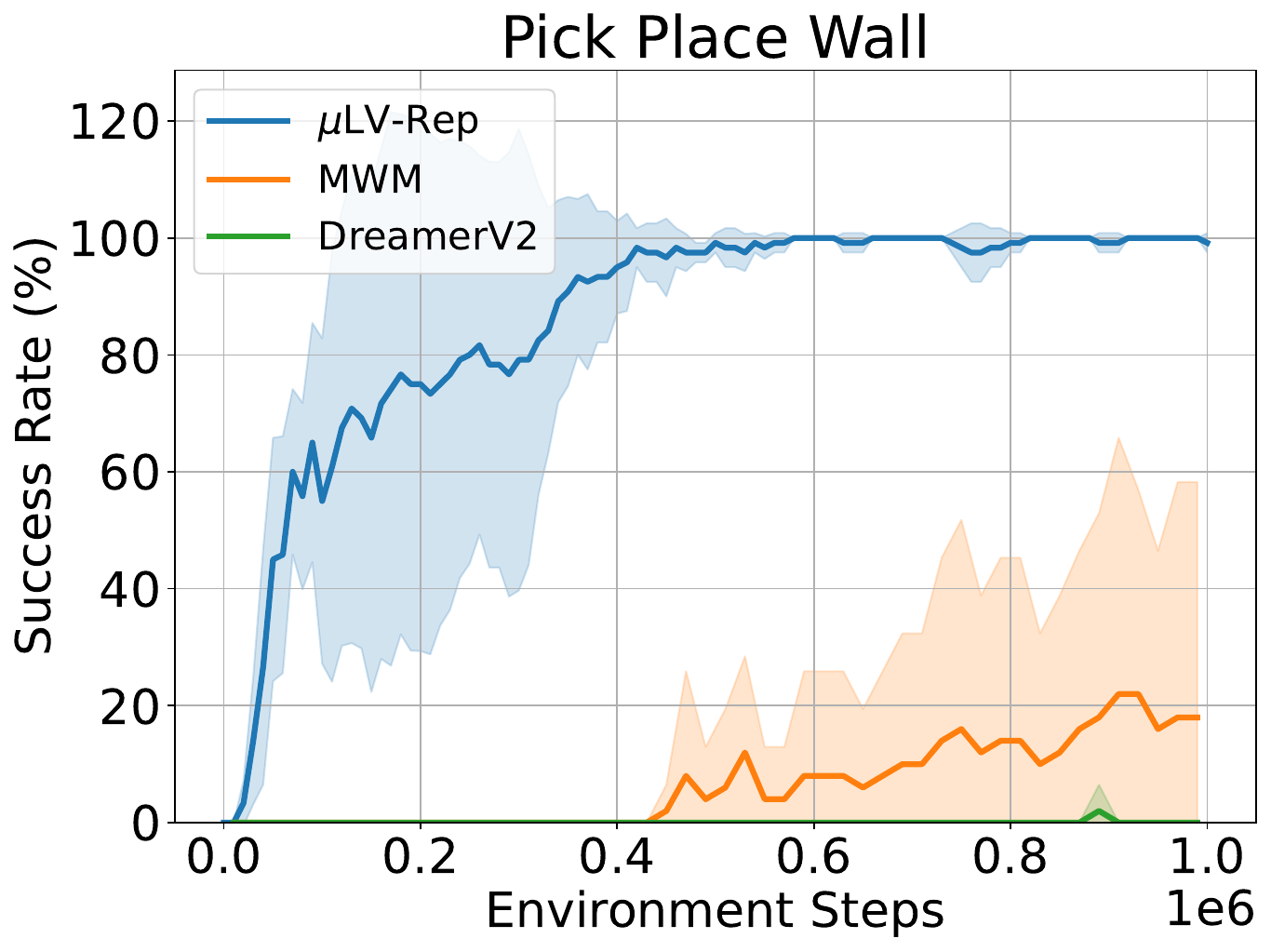}}
    \subfigure{\includegraphics[width=0.18\textwidth]{pic/metaworld_plate_slide_performance.pdf}}
    \subfigure{\includegraphics[width=0.18\textwidth]{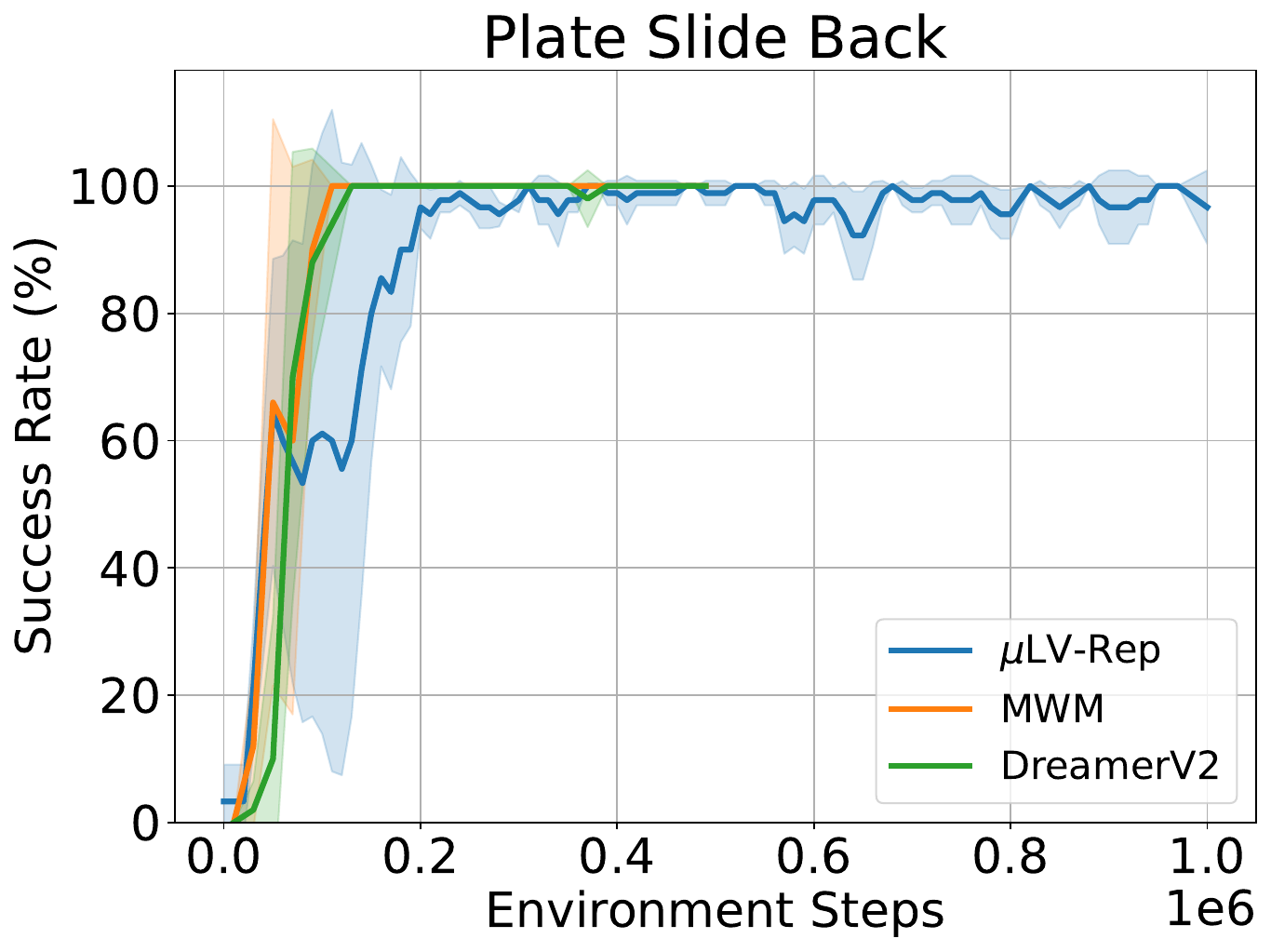}}
    \vskip -0.15in
    \subfigure{\includegraphics[width=0.18\textwidth]{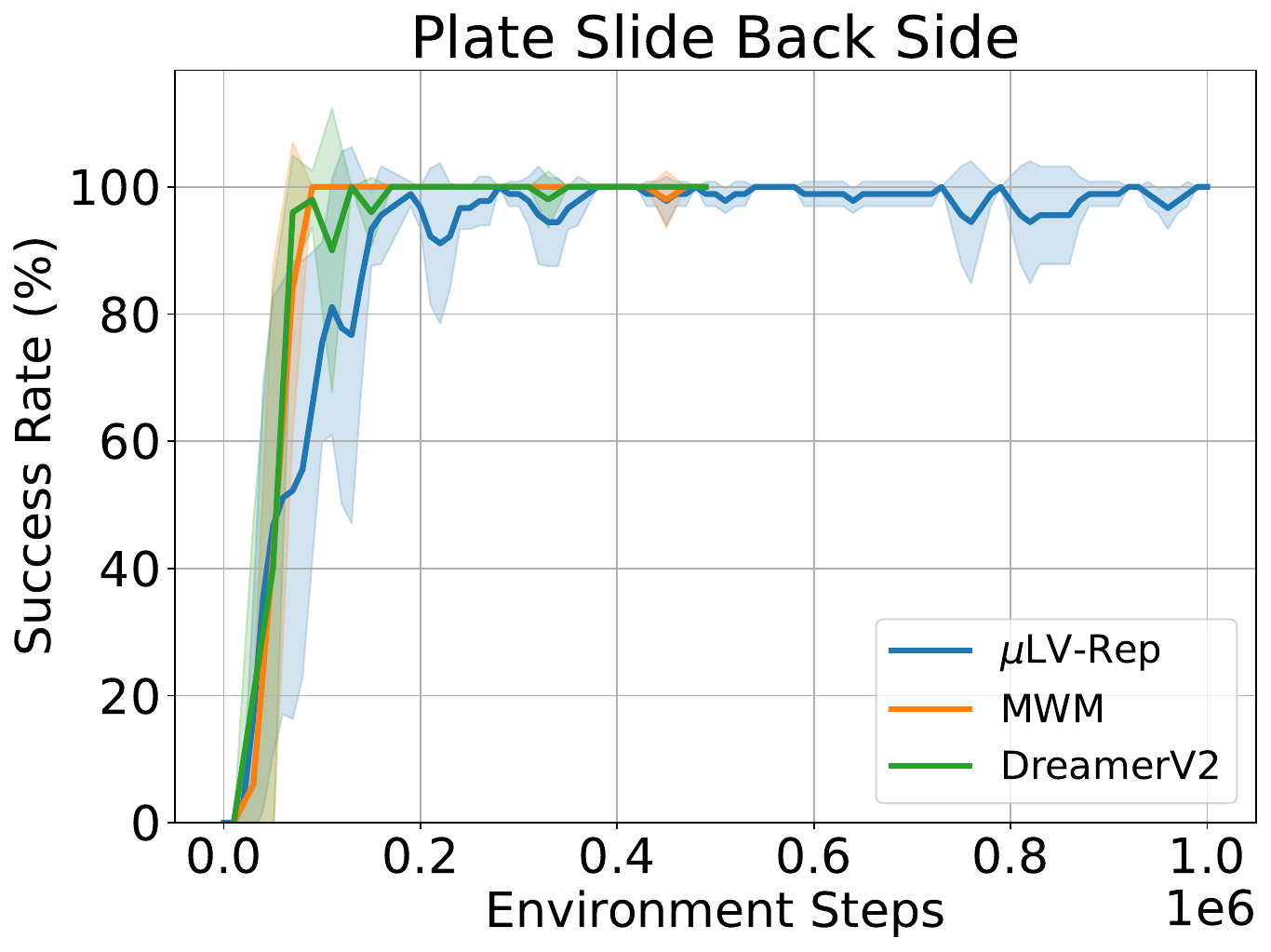}}
    \subfigure{\includegraphics[width=0.18\textwidth]{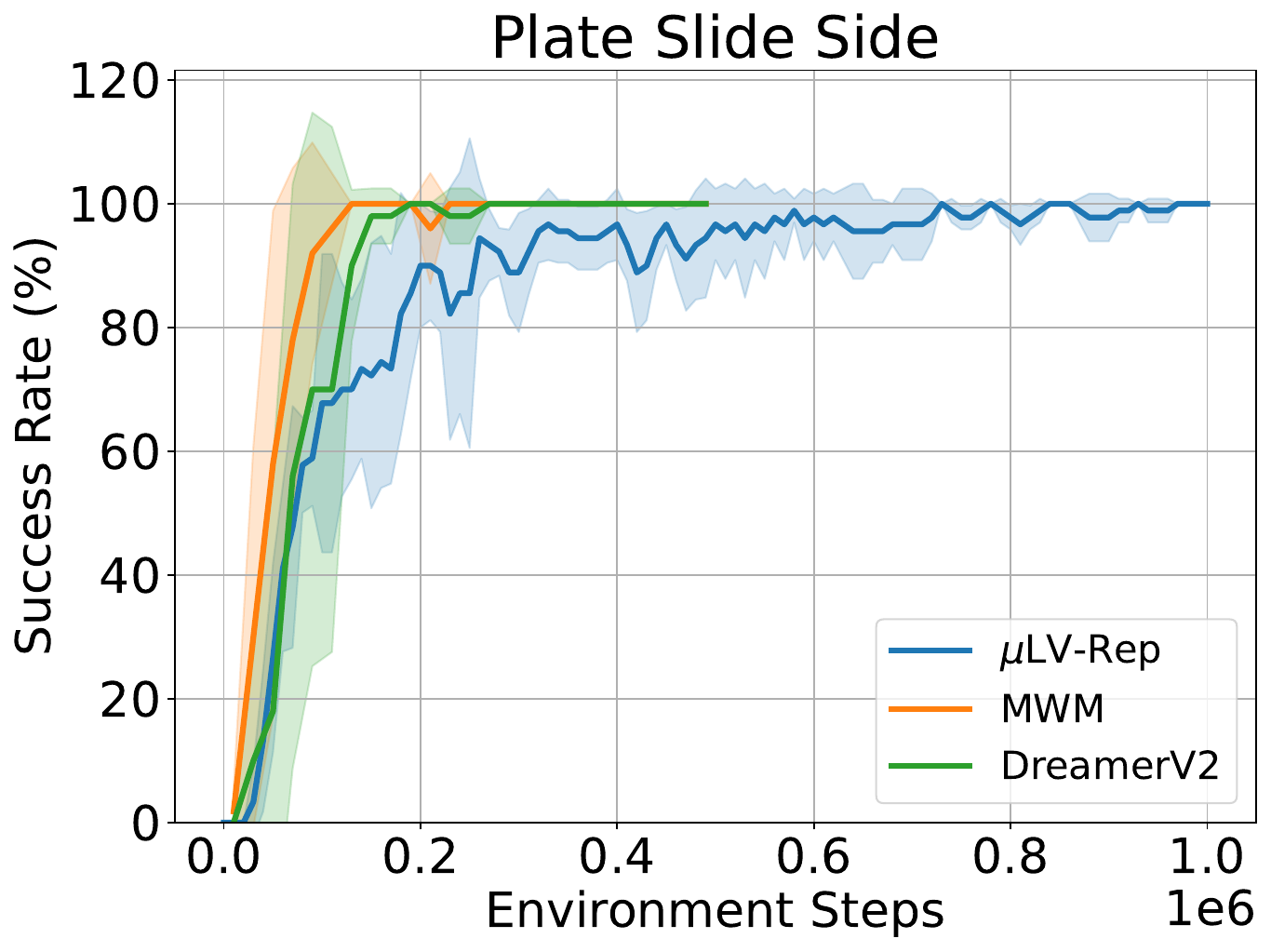}}
    \subfigure{\includegraphics[width=0.18\textwidth]{pic/metaworld_push_performance.pdf}}
    \subfigure{\includegraphics[width=0.18\textwidth]{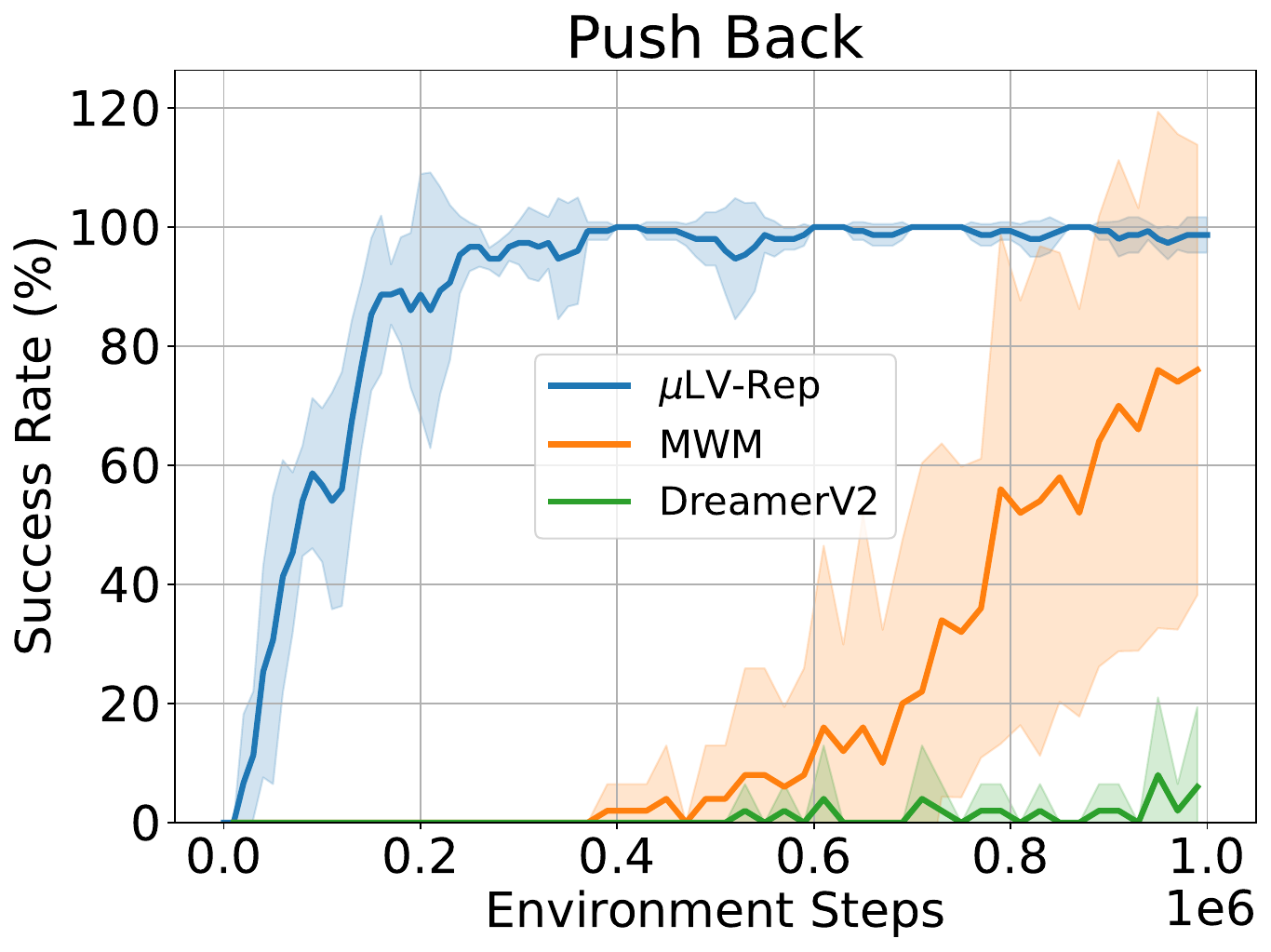}}
    \subfigure{\includegraphics[width=0.18\textwidth]{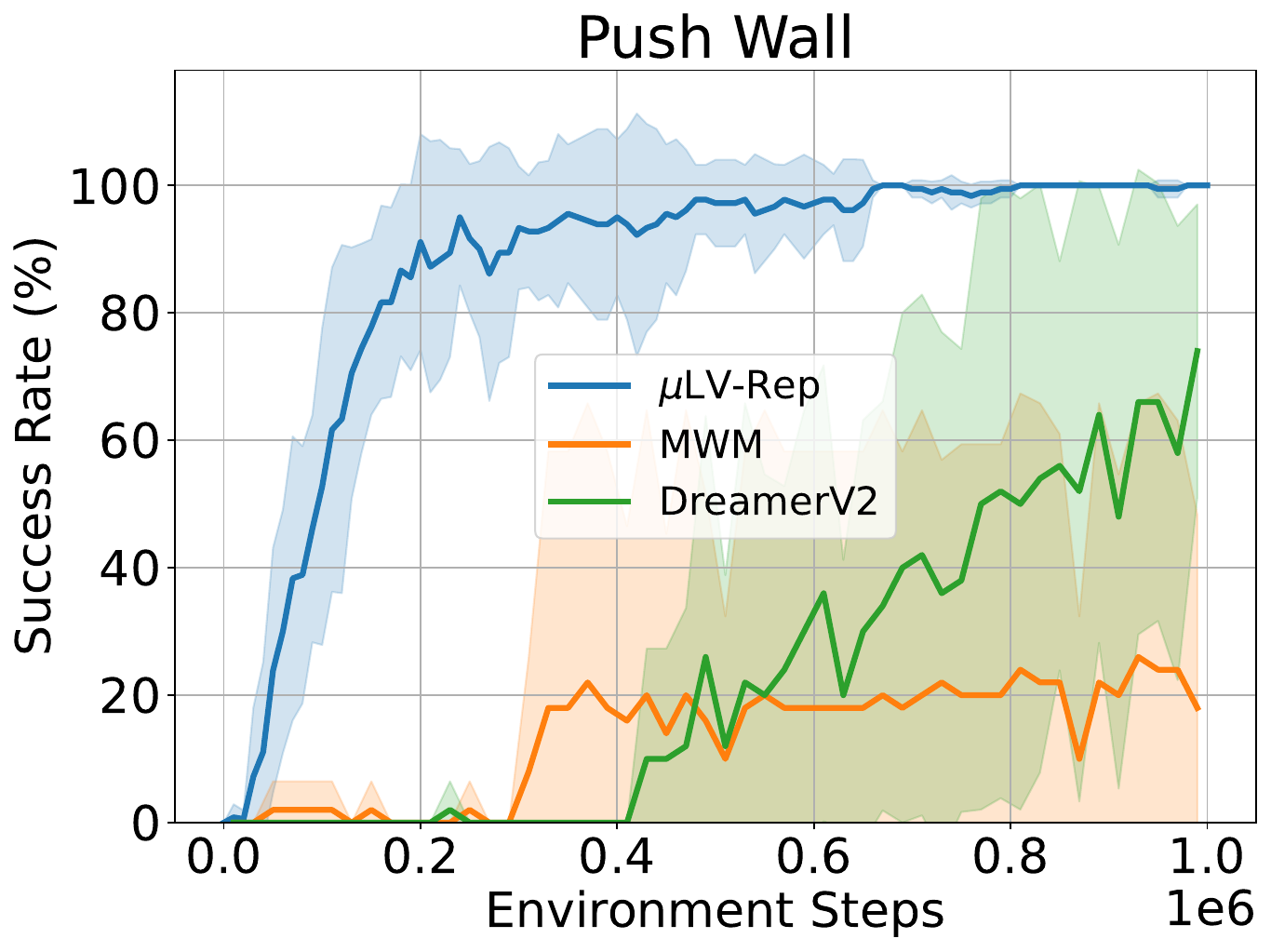}}
    \vskip -0.15in
    \subfigure{\includegraphics[width=0.18\textwidth]{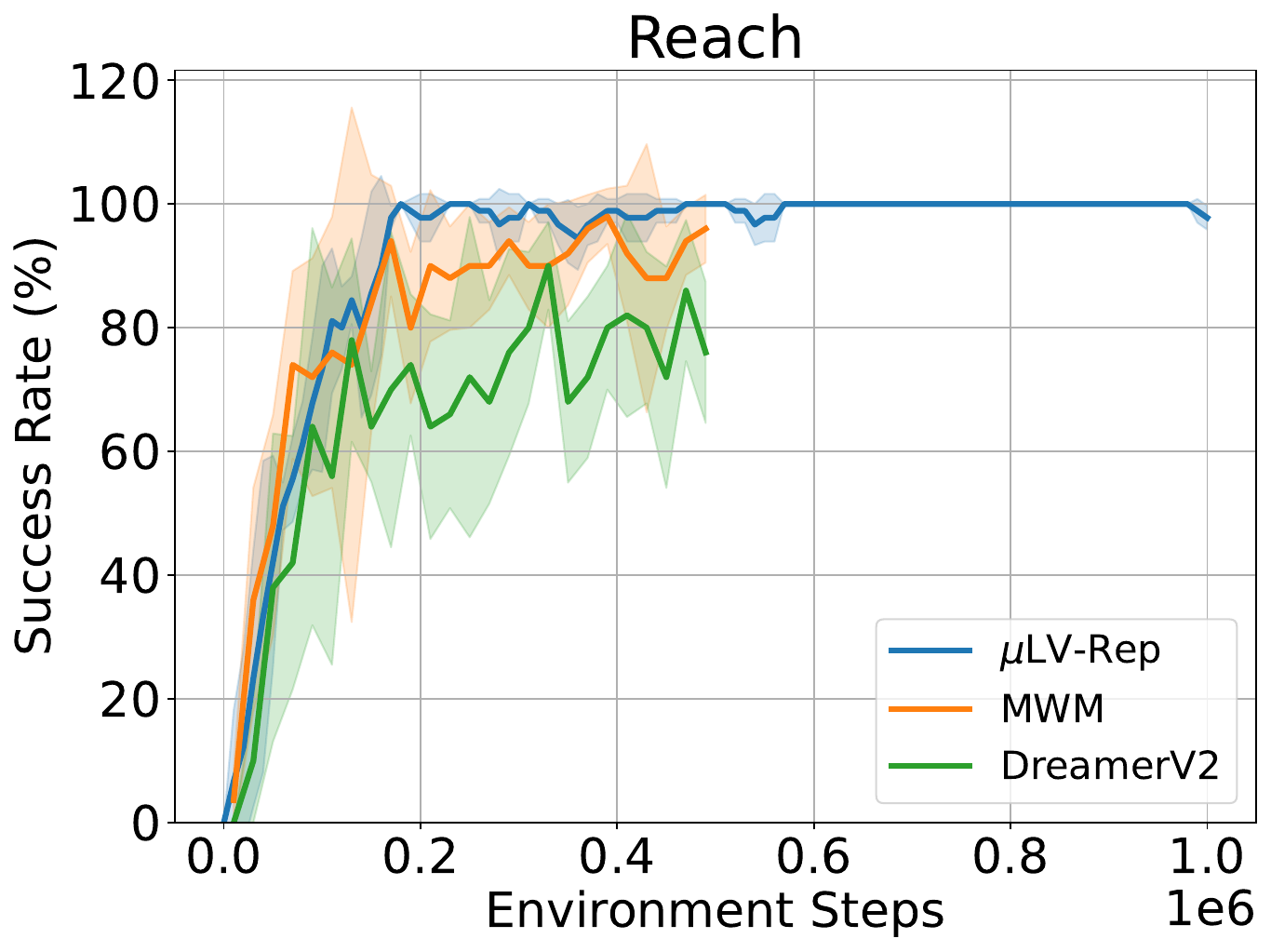}}
    \subfigure{\includegraphics[width=0.18\textwidth]{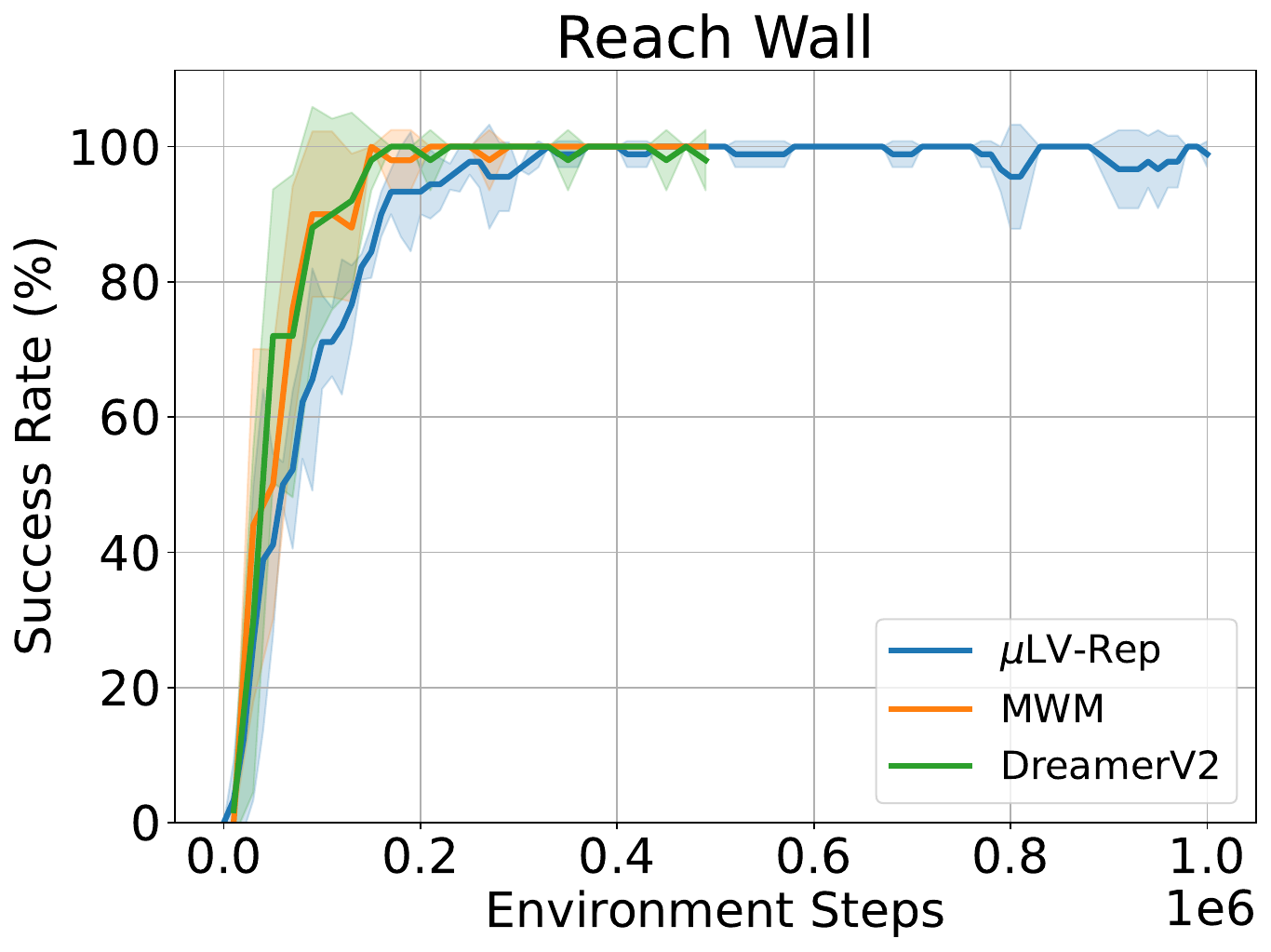}}
    \subfigure{\includegraphics[width=0.18\textwidth]{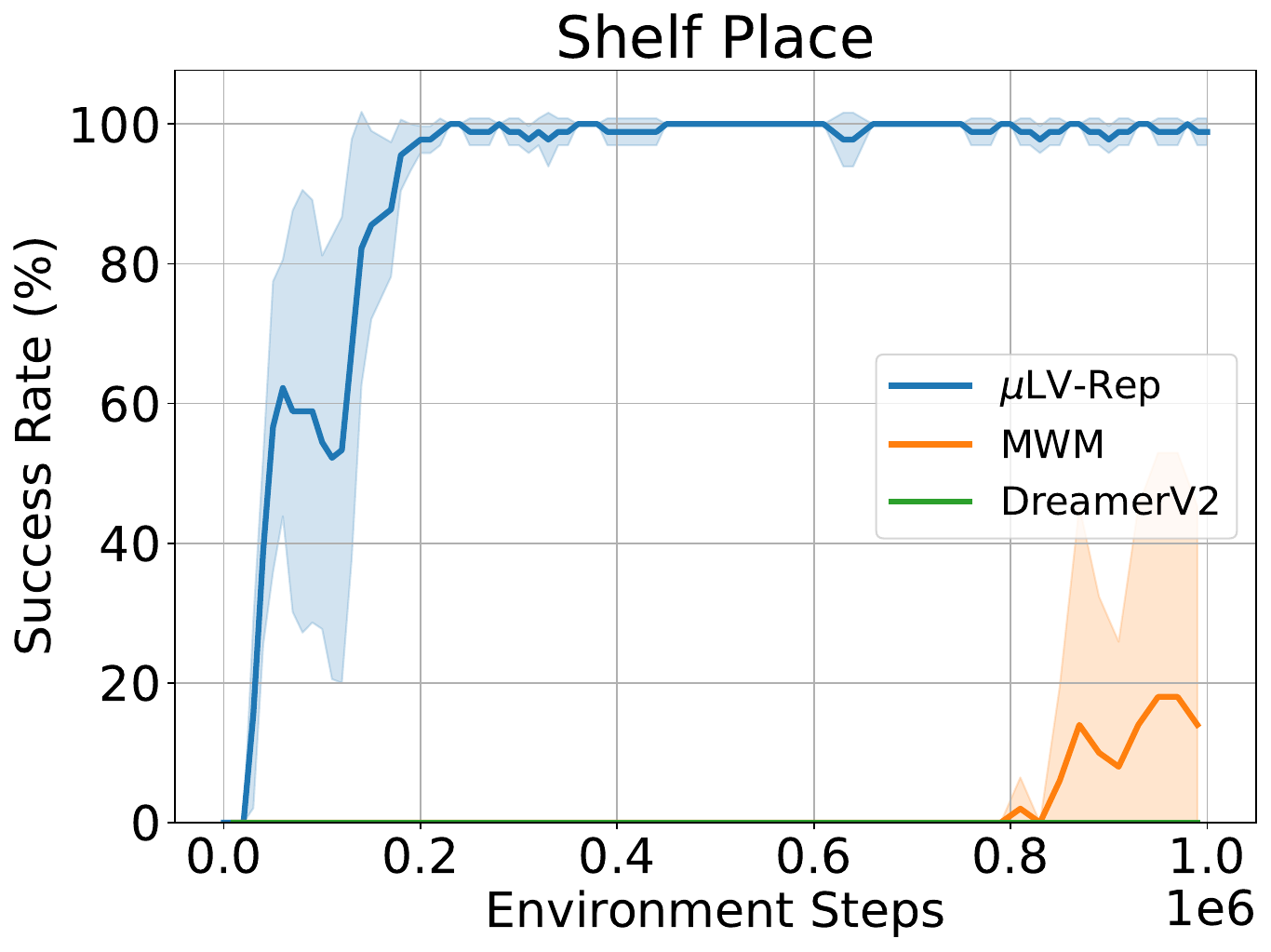}}
    \subfigure{\includegraphics[width=0.18\textwidth]{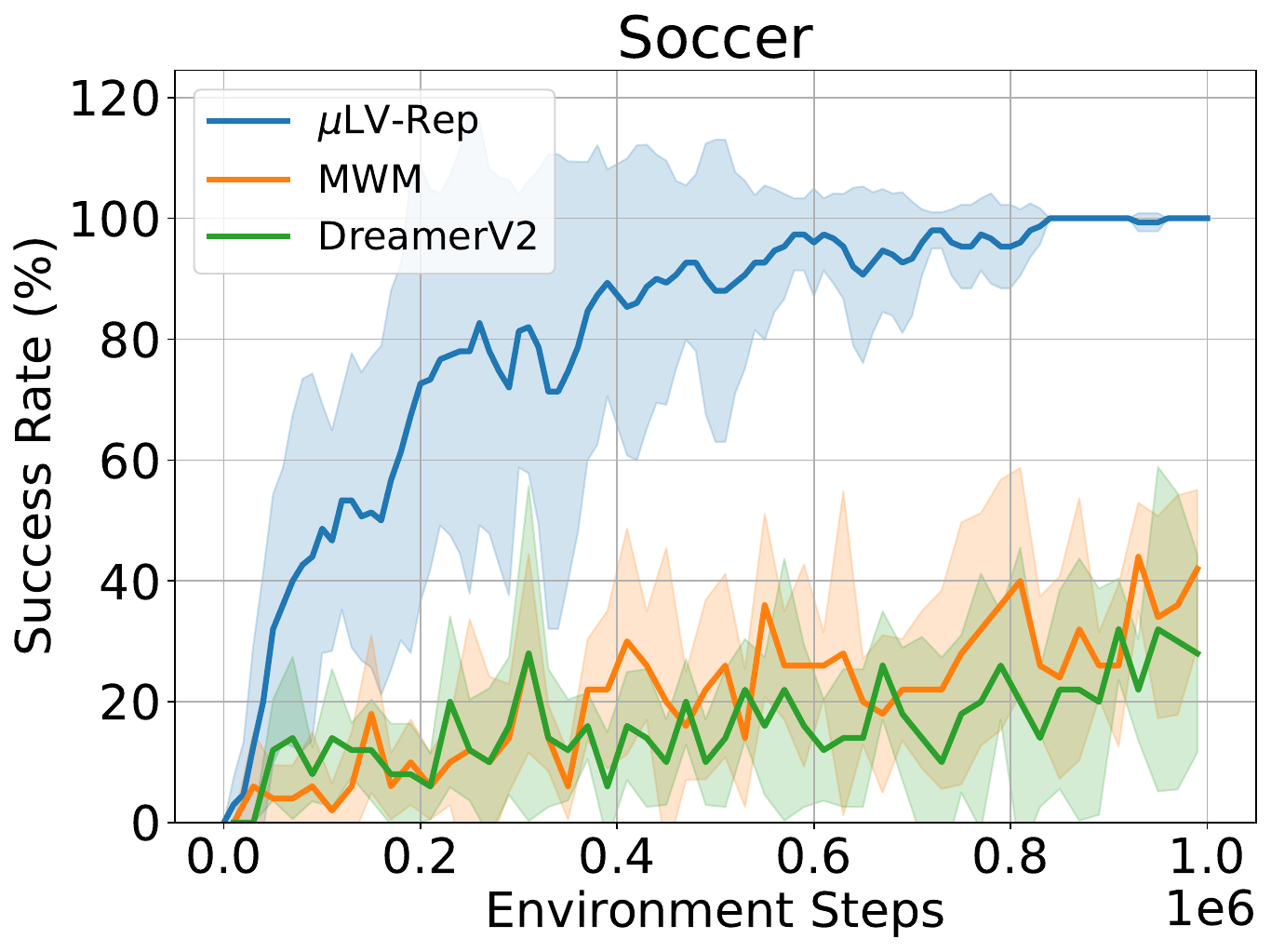}}
    \subfigure{\includegraphics[width=0.18\textwidth]{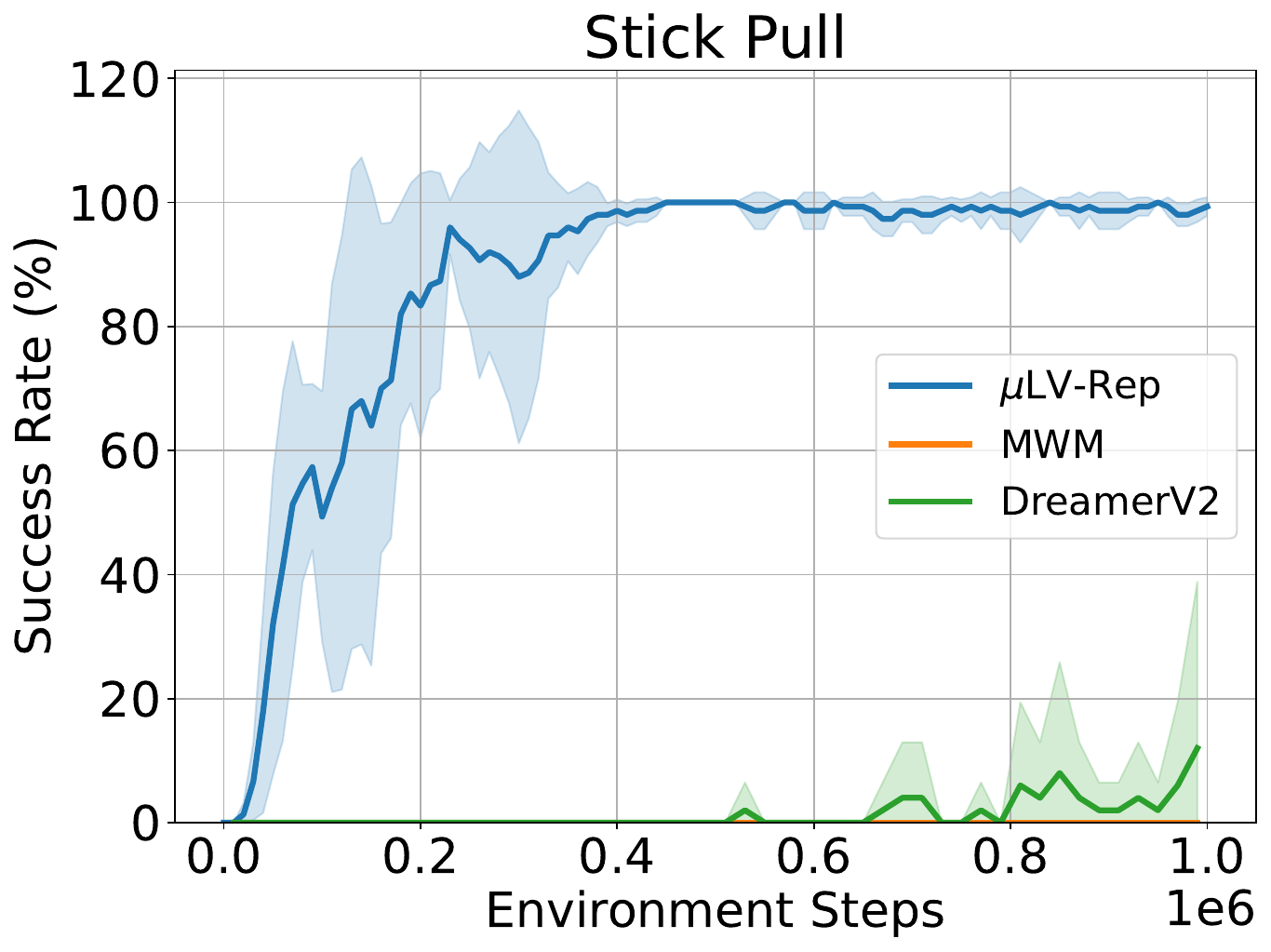}}
    \vskip -0.15in
    \subfigure{\includegraphics[width=0.18\textwidth]{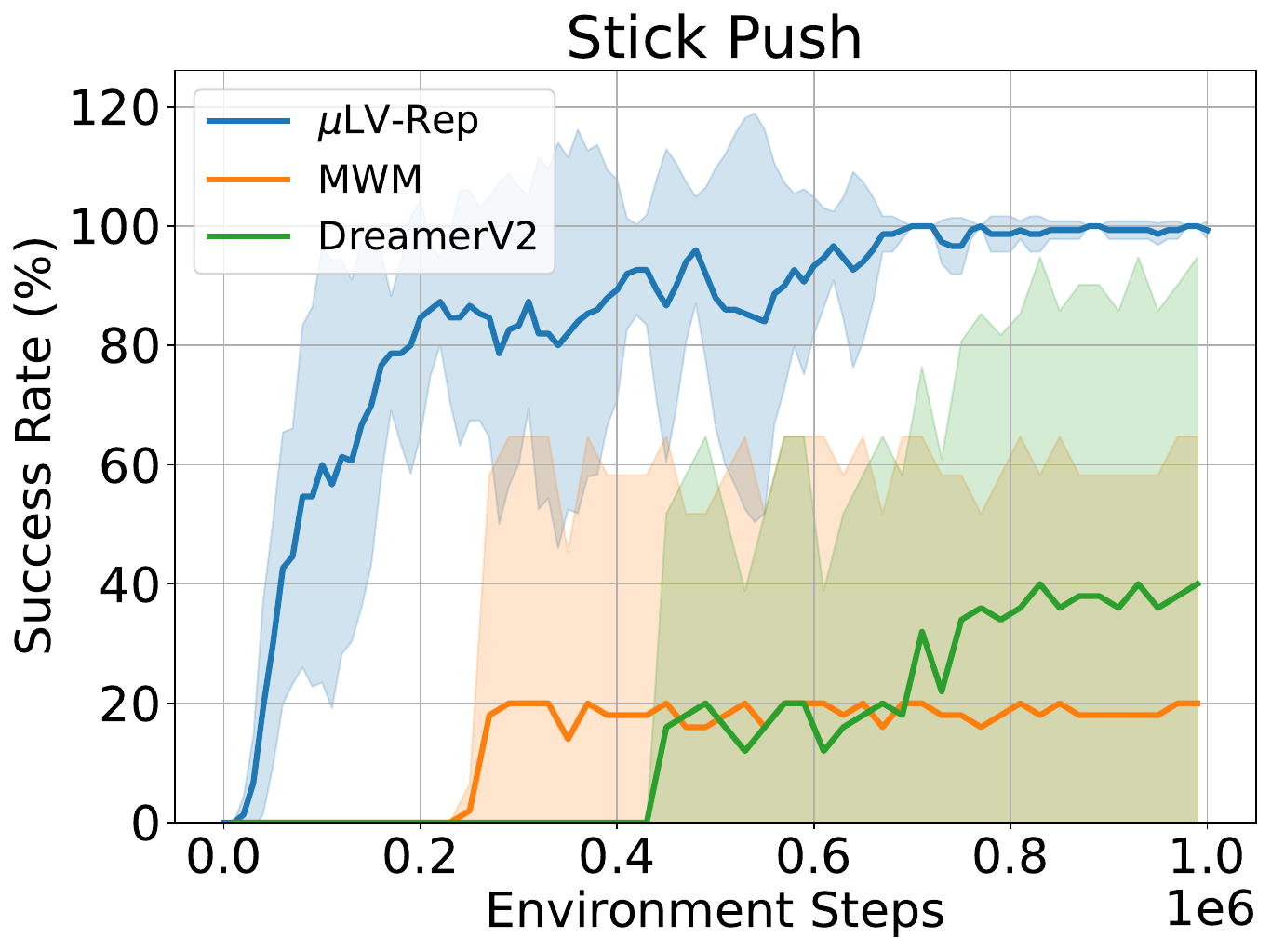}}
    \subfigure{\includegraphics[width=0.18\textwidth]{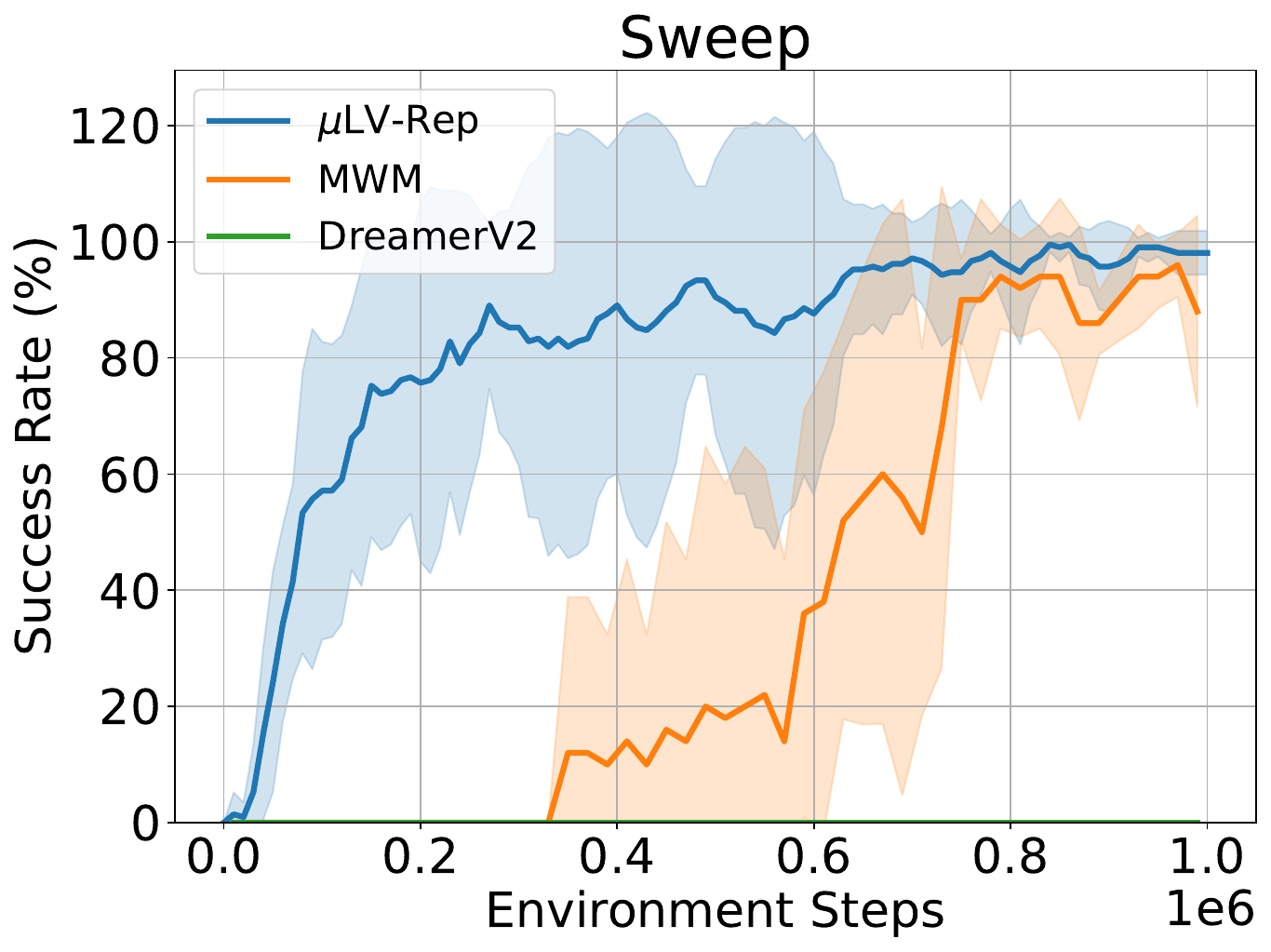}}
    \subfigure{\includegraphics[width=0.18\textwidth]{pic/metaworld_sweep_into_performance.pdf}}
    \subfigure{\includegraphics[width=0.18\textwidth]{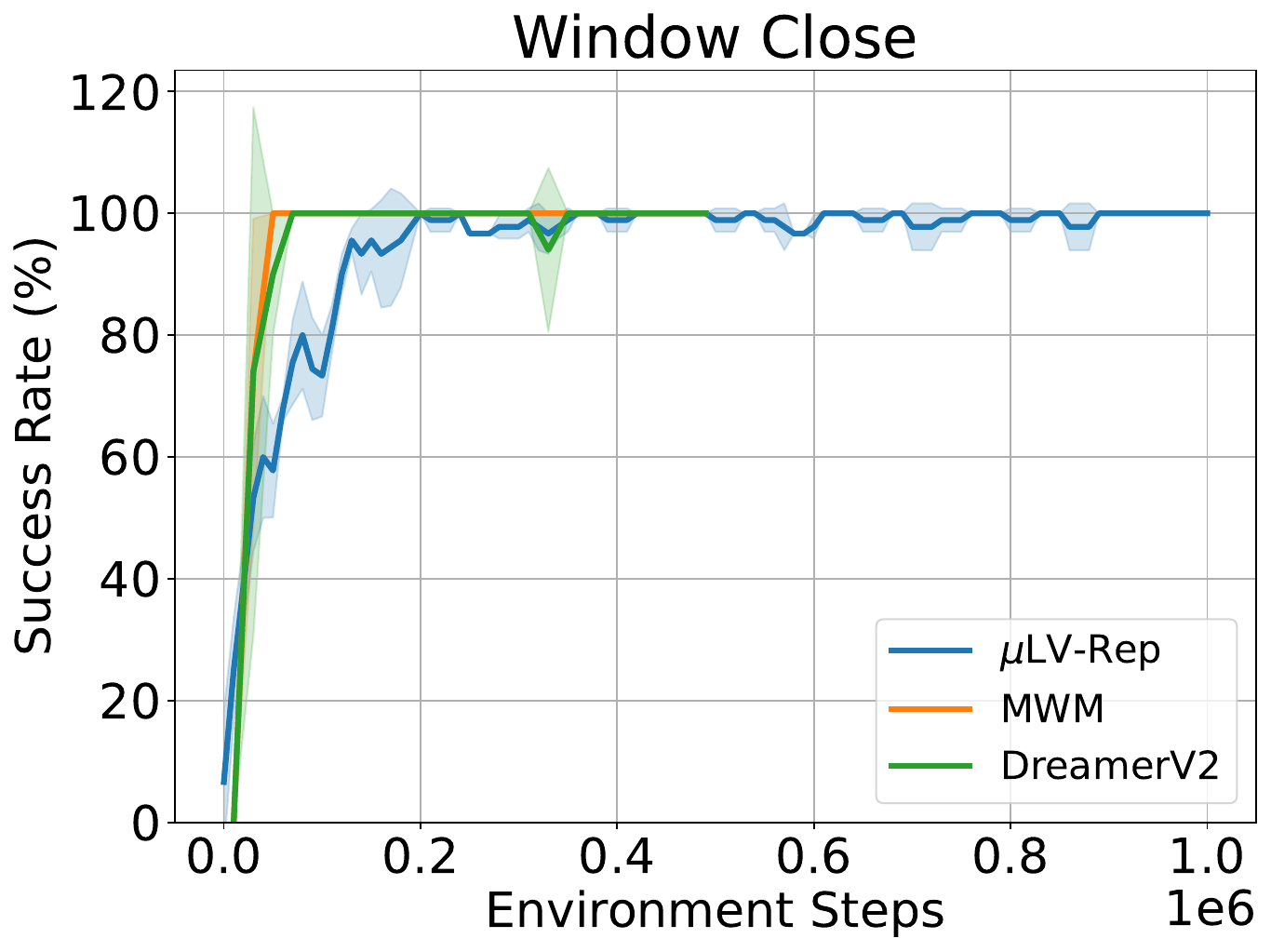}}
    \subfigure{\includegraphics[width=0.18\textwidth]{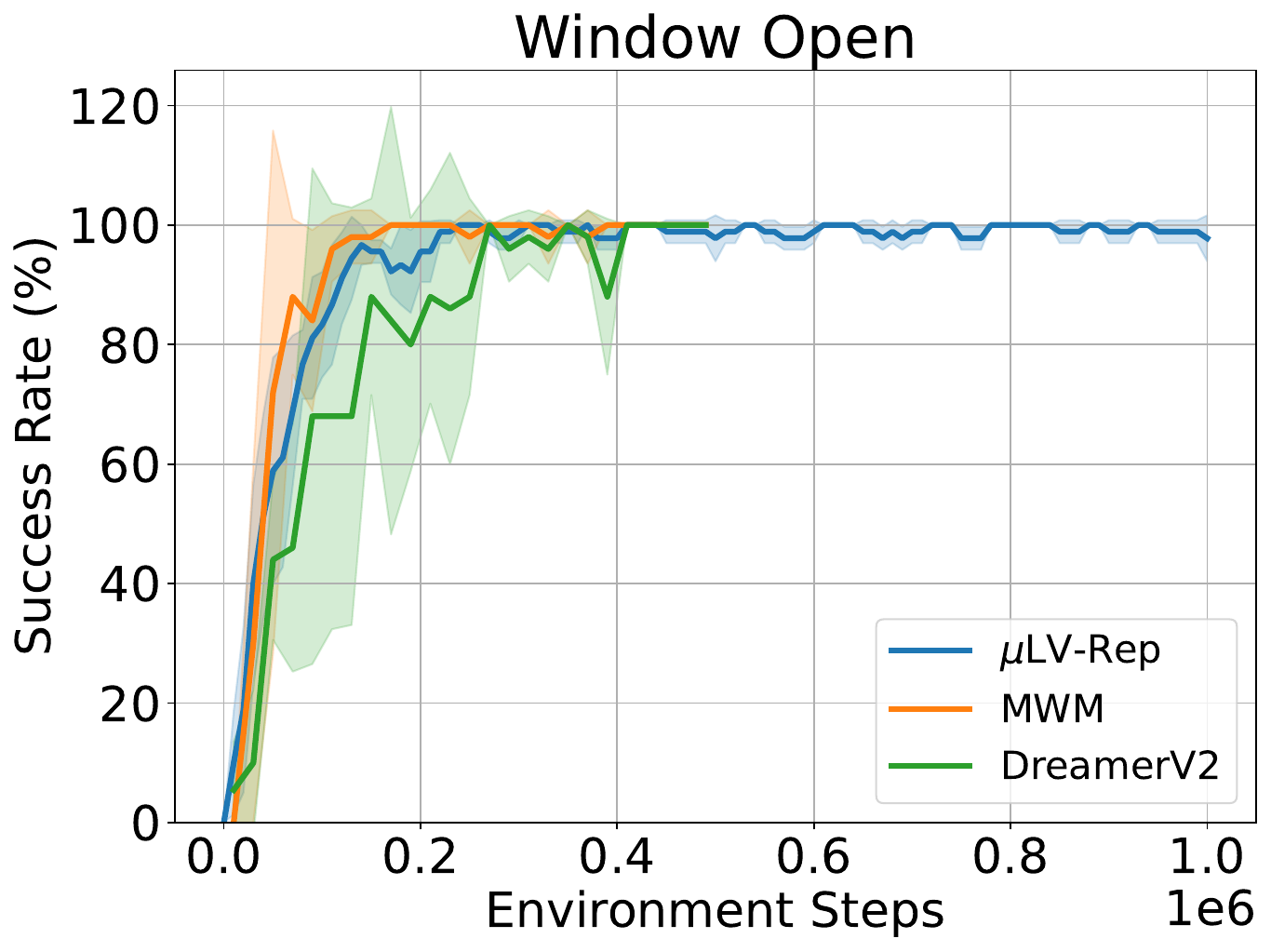}}
    \vskip -0.2in
    \caption{\footnotesize
    Overall performance on MetaWorld tasks.
    \normalsize}
    \label{fig: metaworld performance}
    \vskip -0.15in
\end{center}
\end{figure}

% \begin{figure*}[htbp]
%     % \vskip -0.1in
%     \centering
%     \subfigure{\includegraphics[width=0.245\textwidth]{ICML2024/pic/dmc_reacher_easy_performance.pdf}}
%     \subfigure{\includegraphics[width=0.245\textwidth]{ICML2024/pic/dmc_reacher_hard_performance.pdf}}
%     \subfigure{\includegraphics[width=0.245\textwidth]{ICML2024/pic/dmc_finger_turn_easy_performance.pdf}}
%     \subfigure{\includegraphics[width=0.245\textwidth]{ICML2024/pic/dmc_manip_reach_duplo_performance.pdf}}
%     % \vspace{-6mm}
%     % \\
%     % \subfigure{\includegraphics[width=0.245\textwidth]{figures/dmc_finger_turn_easy_performance.pdf}}
%     % \subfigure{\includegraphics[width=0.245\textwidth]{figures/dmc_manip_reach_duplo_performance.pdf}}
%     % \subfigure{\includegraphics[width=0.245\textwidth]{figures/dmc_quadruped_run_performance.pdf}}
%     % \subfigure{\includegraphics[width=0.245\textwidth]{figures/dmc_quadruped_walk_performance.pdf}}
%     % \\
%     % \subfigure{\includegraphics[width=0.245\textwidth]{figures/dmc_cartpole_swingup_sparse_performance.pdf}}
%     % \subfigure{\includegraphics[width=0.245\textwidth]{figures/dmc_hopper_hop_performance.pdf}}
%     % \subfigure{\includegraphics[width=0.245\textwidth]{figures/dmc_cheetah_run_performance.pdf}}
%     % \subfigure{\includegraphics[width=0.245\textwidth]{figures/dmc_finger_turn_hard_performance.pdf}}
%     % \vspace{-6mm}
%     \caption{
%     Learning curves on visual control tasks from DeepMind Control Suites measured by episodic return.
%     \normalsize}
%     \label{fig:dmc}
% \vspace{-2mm}
% \end{figure*}

In particular, we employ visual observations with dimensions of 64 × 64 × 3 and apply a variational autoencoder (VAE) to learn representations for these visual observations. 
The VAE is first pre-trained with random trajectories at the beginning and then fine-tuned during the online learning procedure. 
It produces compact vector representations for the images, which are then forwarded as input to our representation learning method. 
We apply actor-critic Learning based on the representation learned by VAE.
The configuration of used tasks are given in~\cref{table:env configuration}. The hyperparameters used in the RL agent are shown in~\cref{table:Hyperparameters}.

\begin{table}[htbp]
\caption{Configuration of environments.}
\vskip 0.1in
\label{table:env configuration}
\centering
\begin{tabular}{ll}\toprule[2pt]
\specialrule{0pt}{1pt}{1pt}
	Hyperparameter & Value  \\ 
	\hline\specialrule{0pt}{1pt}{1pt}
    Image observation & 64$\times$64$\times3$  \\\specialrule{0pt}{1pt}{1pt}
    Image normalization & Mean: $(0.485, 0.456, 0.406)$, Std: $(0.229, 0.224, 0.225)$ \\\specialrule{0pt}{1pt}{1pt}
    Action repeat & 2  \\\specialrule{0pt}{1pt}{1pt}
    Episode length  &  500 (Meta-world), 1000 (DMC) \\\specialrule{0pt}{1pt}{1pt}
    Normalize action & [-1,1] \\\specialrule{0pt}{1pt}{1pt}
    Camera  & corner2 (Meta-world), camera2 (DMC) \\\specialrule{0pt}{1pt}{1pt}
    Total steps in environment & 1M (Meta-world), 0.5M (DMC) \\\specialrule{0pt}{1pt}{1pt}
	 \specialrule{0pt}{1pt}{1pt}\bottomrule[2pt]
\end{tabular}
\end{table}

% \begin{table}[htbp]
% \caption{Hyperparameters in world model.}
% \vskip 0.1in
% \label{table:world model}
% \centering
% \begin{tabular}{ll}\toprule[2pt]
% \specialrule{0pt}{1pt}{1pt}
% 	Hyperparameter & Value  \\ 
% 	\hline\specialrule{0pt}{1pt}{1pt}
%     MAE \\\specialrule{0pt}{1pt}{1pt}
%     \hline 
%     ViT encoder size &  depth: 4, heads: 4, embedding dim: 256 \\\specialrule{0pt}{1pt}{1pt}
%     ViT decoder size & depth: 3, heads: 4, embedding dim: 128  \\\specialrule{0pt}{1pt}{1pt}
%     Patch size & 8$\times$8  \\\specialrule{0pt}{1pt}{1pt}
%     Mask ratio  &  0.75 \\\specialrule{0pt}{1pt}{1pt}
%     Batch size & 1024 \\\specialrule{0pt}{1pt}{1pt}
%     Optimizer  & Adam \\\specialrule{0pt}{1pt}{1pt}
%     Learning rate & 0.0003 \\\specialrule{0pt}{1pt}{1pt}
%     Pretrain step & 5000 \\\specialrule{0pt}{1pt}{1pt}
%     \hline
%     RSSM \\\specialrule{0pt}{1pt}{1pt}
%     \hline 
%     Deterministic state dim &  1024  \\\specialrule{0pt}{1pt}{1pt}
%     Stochastic state dim & 32  \\\specialrule{0pt}{1pt}{1pt}
%     Discrete latent dimensions & 32  \\\specialrule{0pt}{1pt}{1pt}
%     Batch size & 50 (Meta-world), 16 (DMC) \\\specialrule{0pt}{1pt}{1pt}
%     Sequence length & 50 \\\specialrule{0pt}{1pt}{1pt}
%     KL balance & 0.8 \\\specialrule{0pt}{1pt}{1pt}
%     Optimizer  & Adam \\\specialrule{0pt}{1pt}{1pt}
%     Learning rate & 0.0003 \\\specialrule{0pt}{1pt}{1pt}
%     Gradient clip  &  100 \\\specialrule{0pt}{1pt}{1pt}
% 	 \specialrule{0pt}{1pt}{1pt}\bottomrule[2pt]
% \end{tabular}
% \end{table}

\begin{table}[htbp]
\caption{Hyperparameters in $\mu$LV-Rep. The numbers in Conv and MLP denote the output channels and units.}
\vskip 0.1in
\label{table:Hyperparameters}
\centering
\begin{tabular}{ll}\toprule[2pt]
\specialrule{0pt}{1pt}{1pt}
    Hyperparameter & Value  \\ 
    \hline\specialrule{0pt}{1pt}{1pt}
    Buffer size & 1,000,000 \\\specialrule{0pt}{1pt}{1pt}
    Batch size & 256 \\\specialrule{0pt}{1pt}{1pt}
    Random steps & 4000 \\\specialrule{0pt}{1pt}{1pt}
    Pretrain step & 10000 \\\specialrule{0pt}{1pt}{1pt}
    Features dim. & 100 \\\specialrule{0pt}{1pt}{1pt}
    Hidden dim. & 1024 \\\specialrule{0pt}{1pt}{1pt} 
    Encoder & (Conv(32), Conv(32), Conv(32), Conv(32), MLP(100)) \\\specialrule{0pt}{1pt}{1pt}
    Image Decoder & (MLP(100), MLP(1024), ConvT(32), ConvT(32), ConvT(32), ConvT(32), Conv(3)) \\\specialrule{0pt}{1pt}{1pt}
    Actor Network & (MLP(1024),MLP(1024),MLP(Action Space)) \\\specialrule{0pt}{1pt}{1pt}
    Critic Network & (MLP(1024),MLP(1024),MLP(1)) \\\specialrule{0pt}{1pt}{1pt}
    Optimizer  & Adam \\\specialrule{0pt}{1pt}{1pt}
    Learning rate & 0.0001 \\\specialrule{0pt}{1pt}{1pt}
    Discount & 0.99 \\\specialrule{0pt}{1pt}{1pt}    
    Critic soft-update rate & 0.01 \\\specialrule{0pt}{1pt}{1pt}
    Evaluate interval & 10,000 \\\specialrule{0pt}{1pt}{1pt}
    Evaluate episodes  &  10 (Meta-world), 5 (DMC) \\\specialrule{0pt}{1pt}{1pt}
	 \specialrule{0pt}{1pt}{1pt}\bottomrule[2pt]
\end{tabular}
\end{table}

% \begin{table}[htbp]
% \caption{Hyperparameters used in Actor Critic.}
% \vskip 0.1in
% \label{table:ac}
% \centering
% \begin{tabular}{ll}\toprule[2pt]
% \specialrule{0pt}{1pt}{1pt}
% 	Hyperparameter & Value  \\ 
% 	\hline\specialrule{0pt}{1pt}{1pt}
%     Replay buffer & 1,000,000 \\\specialrule{0pt}{1pt}{1pt}
%     Batch size & 50 \\\specialrule{0pt}{1pt}{1pt}
%     Trajectory length & 50 \\\specialrule{0pt}{1pt}{1pt}
%     Network size &  [512, 512, 512, 512] \\\specialrule{0pt}{1pt}{1pt}
%     Optimizer & Adam  \\\specialrule{0pt}{1pt}{1pt}
%     Learning rate & 0.0001 \\\specialrule{0pt}{1pt}{1pt}
%     Gradient clip  &  100 \\\specialrule{0pt}{1pt}{1pt}
%     Entropy weight & 0.0001 \\\specialrule{0pt}{1pt}{1pt}
%     Discount & 0.99 \\\specialrule{0pt}{1pt}{1pt}
%     $\lambda$ return discount & 0.95 \\\specialrule{0pt}{1pt}{1pt}
%     Random steps & 5000 \\\specialrule{0pt}{1pt}{1pt}
%     Evaluate interval & 10,000 \\\specialrule{0pt}{1pt}{1pt}
%     Evaluate episodes  &  10 (Meta-world), 5 (DMC) \\\specialrule{0pt}{1pt}{1pt}
% 	\specialrule{0pt}{1pt}{1pt}\bottomrule[2pt]
% \end{tabular}
% \end{table}

\subsection{Ablation Studies}\label{appendix:ablation}

The importance of the exploration has been demonstrated in~\citep{zhang2022making}. We perform ablation studies to demonstrate the effects of the major components, including representation dimension and window size, as illustrated below. 
\cref{fig:ablate-rep} presents an ablation study on representation dimension, where we compare \algabb with latent representation dimensions 2048, 512, and 128. 
We also ablate the effect of window size $L$.  
In \cref{fig:ablate-l-step}, we compare \algabb with window size $L=1,3,5$. 
We also compare DrQ-v2~\citep{yarats2022mastering} with $L=1,3,5$ to show the effect of $L$ on other algorithms. 
The results show that for $L=1$, both \algabb and DrQ-v2 struggle with learning, which confirms the Non-Markovian property of the DMC control problems. 
We can also find that $L=3$ is sufficient for learning in both test domains. 

\begin{figure}[htbp]
    \vskip -0.1in
    % \vspace{-4mm}
\begin{center}
    \subfigure{\includegraphics[width=0.25\textwidth]{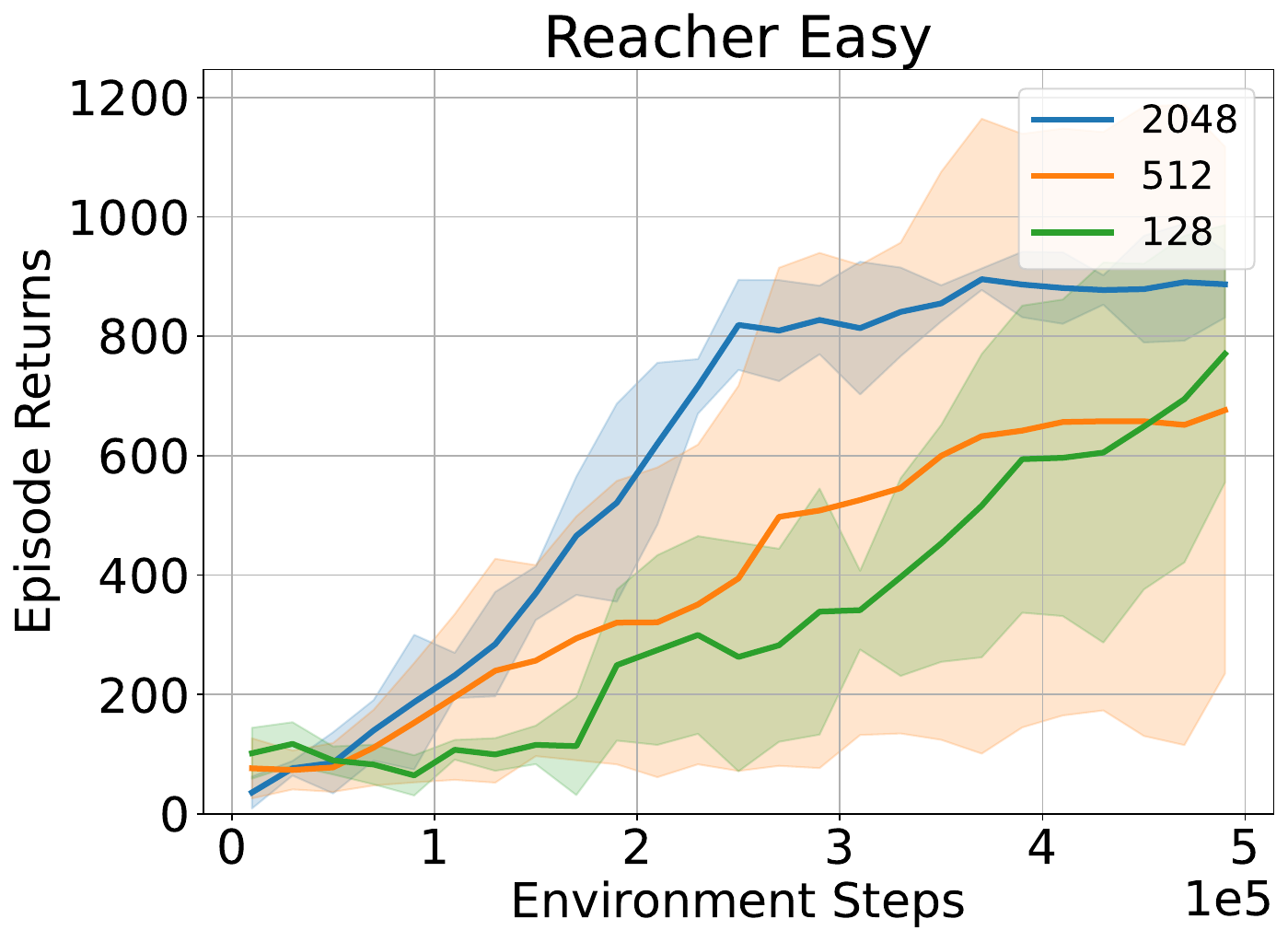}}
    \subfigure{\includegraphics[width=0.25\textwidth]{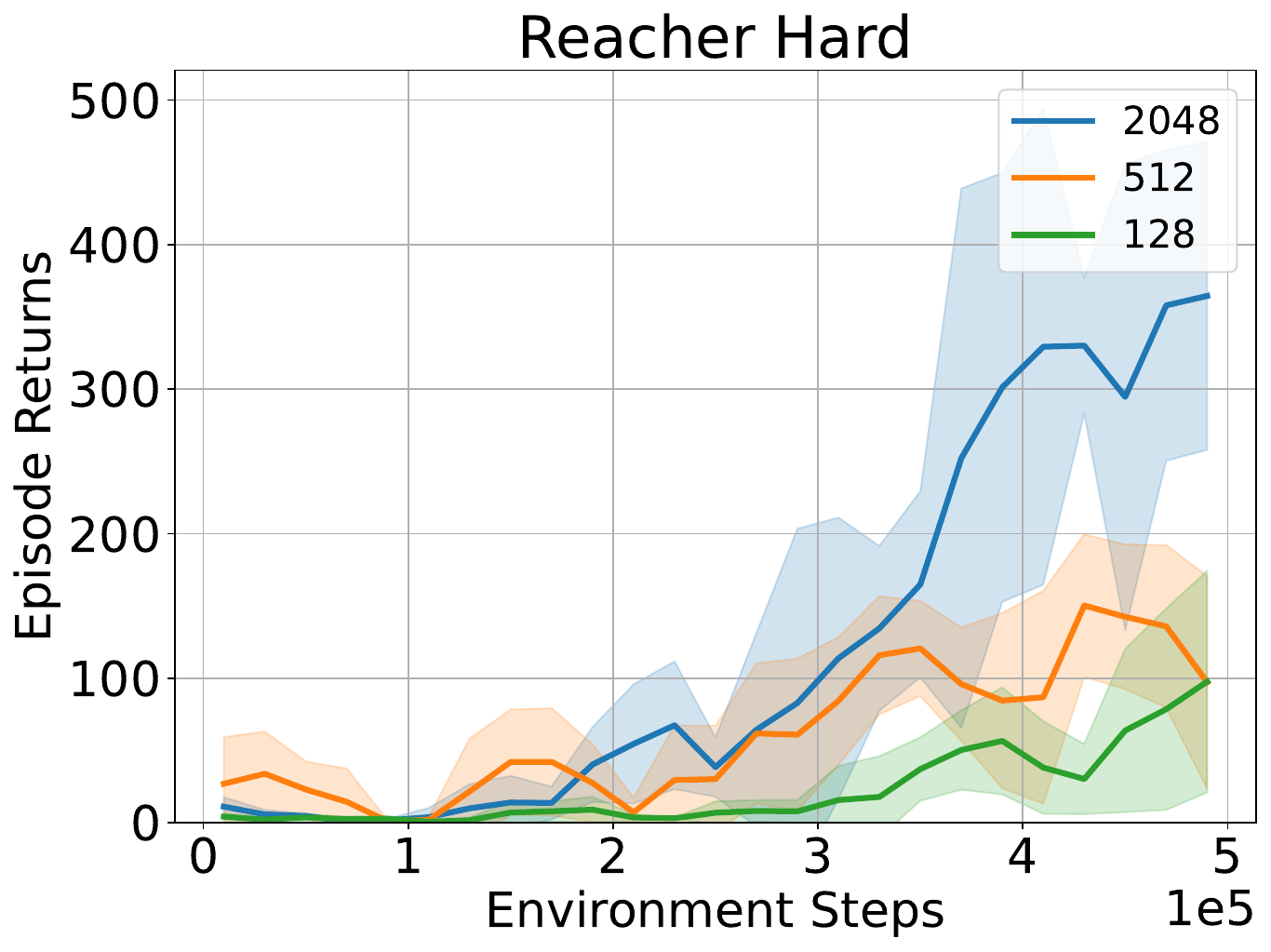}}
    \caption{
    % \color{blue}
    Ablation of feature dimension on visual control tasks from DeepMind Control Suites.
    Increasing the dimension of the feature gets better performance.
    \normalsize}
    \label{fig:ablate-rep}
\end{center}
\end{figure}

\begin{figure}[tbp]
    % \vspace{-4mm}
\begin{center}
    \subfigure{\includegraphics[width=0.25\textwidth]{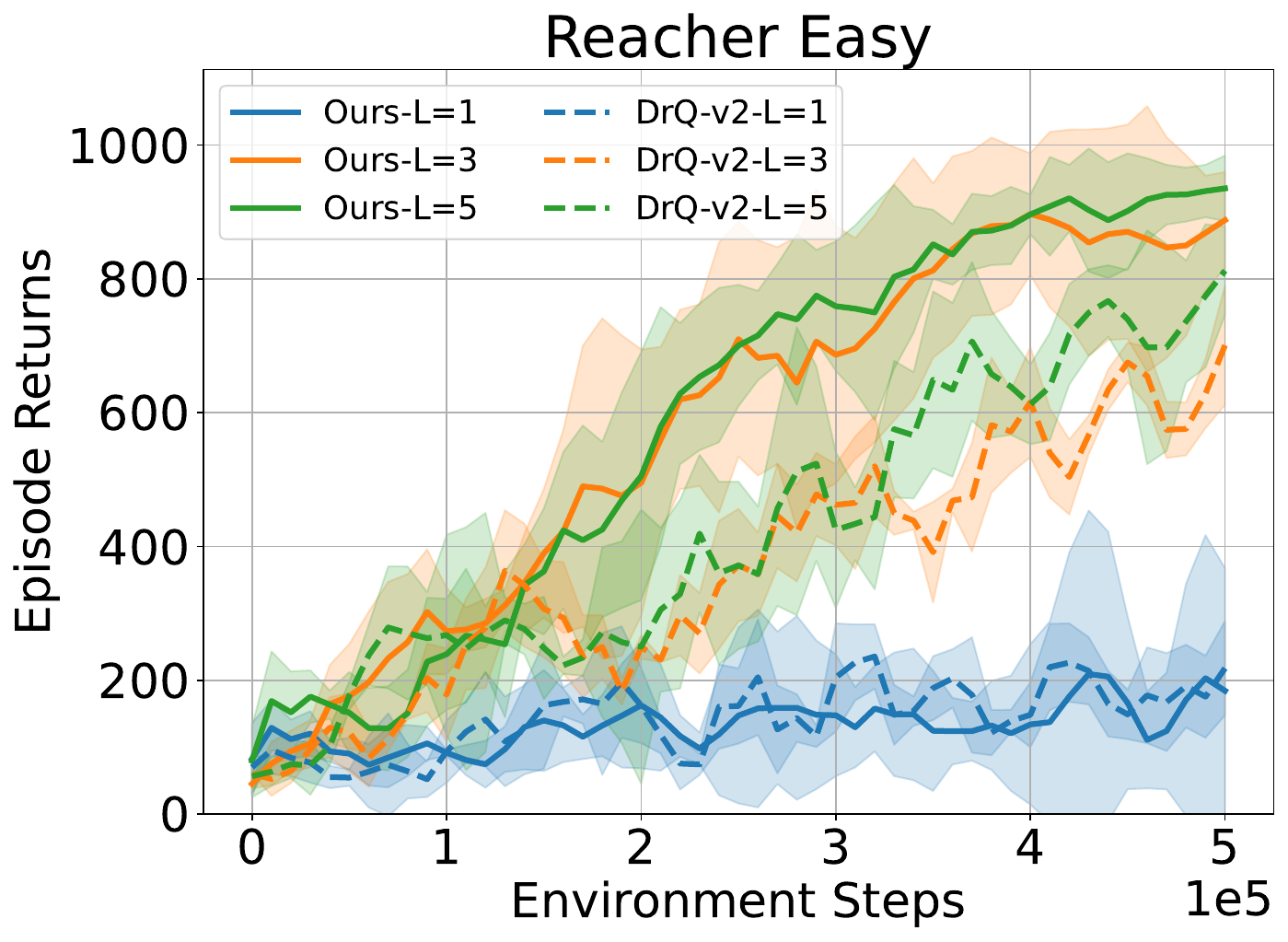}}
    \subfigure{\includegraphics[width=0.25\textwidth]{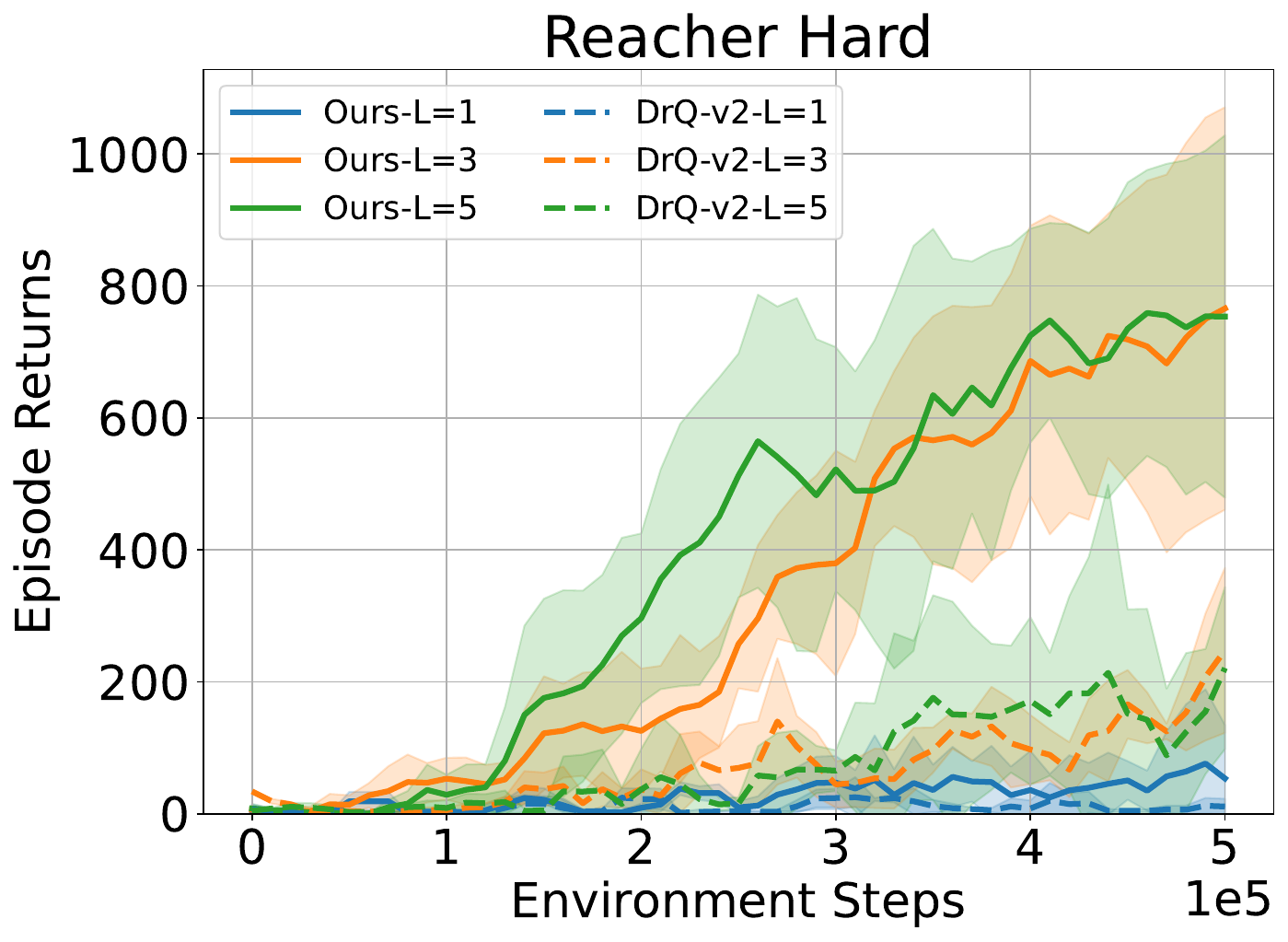}}
    \caption{
    % \color{blue}
    Ablation of window size $L$ on visual control tasks from DeepMind Control Suites. $L=3$ is sufficient for learning in both test domains.
    \normalsize}
    \label{fig:ablate-l-step}
\end{center}
\end{figure}

\subsection{Experiment of Partially Observable Continuous Control}
\label{sec:extra-experiments}

We also evaluate the proposed approach on  RL tasks with partial observations, constructed based on the OpenAI gym MuJoCo~\citep{todorov2012mujoco}. 
%\Bo{@Chenjun, is it possible we have the curves for experiments in this setting?}
Standard MuJoCo tasks from the OpenAI gym and DeepMind Control Suites are not partially observable. 
To generate partially observable problems based on these tasks, we adopt a widely employed approach of masking velocities within the observations~\citep{ni2021recurrent,weigand2021reinforcement,gangwani2020learning}.  
In this way, it becomes impossible to extract complete decision-making information from a single environment observation, yet the ability to reconstruct the missing observation remains achievable by aggregating past observations. 
We provide the best performance when using the original fully observable states (without velocity masking) as input, denoted by \emph{Best-FO} (Best result with Full Observations). This gives a reference for the best result an algorithm is expected to achieve in our tests.

We consider four baselines in the experiments, including two model-based methods
Dreamer \citep{Hafner2020Dream,hafner2021mastering} and Stochastic Latent Actor-Critic (SLAC) \citep{lee2020stochastic}, and a model-free baseline, SAC-MLP, that concatenates  history sequences (past four observations) as  input to an MLP layer for both the critic and policy. 
This simple baseline can be viewed as an analogue to how DQN processes observations in Atari games \citep{mnih2013playing} as a sanity check. 
We also compare to the neural PSR~\citep{guo2018neural}. 
We compare all algorithms after running 200K environment steps. 
This setup exactly follows the benchmark \citep{wang2019benchmarking}, which has been widely adopted in \citep{zhang2022making,ren2023spectral,ren2023latent} for fairness. 
All results are averaged across four random seeds. 
\cref{tab:MuJoCo_results_POMDP} presents all the experimental
results,  averaged over four random seeds. 
The results clearly demonstrate that the proposed method consistently delivers either competitive or superior outcomes across all domains  compared to both the model-based and model-free baselines. 
We note that in most domains, \algabb nearly matches the performance of Best-FO, further confirming that the proposed method is able to extract useful representations for decision-making in partially observable environments. 

\begin{table*}[t]
\caption{
 \footnotesize Performance on various continuous control problems {\color{blue}with partial observation}. All results are averaged across 4 random seeds and a window size of 10K. 
%Results marked with $^*$ is adopted from MBBL. 
{\algabb} achieves the best performance compared to the baselines. Here, Best-FO denotes the performance of LV-Rep using {\color{red}full observations} as inputs, providing a reference on how well an algorithm can achieve most in our tests. \normalsize
}
\vspace{2mm}
% \footnotesize
\scriptsize
\setlength\tabcolsep{3.5pt}
\label{tab:MuJoCo_results_POMDP}
\centering
% \footnotesize
\scriptsize
\resizebox{0.9\textwidth}{!}{%
\begin{tabular}{p{2cm}p{2cm}p{2cm}p{2cm}p{2cm}p{2cm}p{2cm}}
% {lcccccccccccc}
\toprule
& HalfCheetah & Humanoid & Walker & Ant & Hopper\\ 
\midrule  
{\bf \algabb} & \textbf{3596.2 $\pm$ 874.5} & \textbf{806.7 $\pm$ 120.7} & \textbf{1298.1$\pm$ 276.3} &\textbf{1621.4 $\pm$ 472.3} & \textbf{1096.4 $\pm$ 130.4} \\
{Dreamer-v2} & 2863.8 $\pm$ 386 & 672.5 $\pm$ 36.6 & \textbf{1305.8 $\pm$ 234.2} & 1252.1 $\pm$ 284.2    & 758.3 $\pm$ 115.8\\
SAC-MLP & 1612.0 $\pm$ 223 & 242.1 $\pm$ 43.6 & 736.5 $\pm$ 65.6 & \textbf{1612.0 $\pm$ 223} & 614.15 $\pm$ 67.6\\
SLAC & \textbf{3012.4} $\pm$ 724.6 & 387.4 $\pm$ 69.2 & 536.5 $\pm$ 123.2 & {1134.8 $\pm$ 326.2} & 739.3 $\pm$ 98.2\\
PSR & 2679.75$\pm$386 & 534.4 $\pm$ 36.6 & 862.4 $\pm$ 355.3 & {1128.3 $\pm$ 166.6} & 818.8 $\pm$ 87.2\\
\midrule
%Best-ML & 5557.6$\pm$439.5 & 1086$\pm$278.2 & 2523.5$\pm$333.9 &  2511.8$\pm$460.0 & 2204.8$\pm$496.0\\
Best-FO & 5557.6$\pm$439.5 & 1086$\pm$278.2 & 2523.5$\pm$333.9 &  2511.8$\pm$460.0 & 2204.8$\pm$496.0\\
\midrule  
& Cheetah-run & Walker-run & Hopper-run &  Humanoid-run & Pendulum\\ 
\midrule  
{\bf \algabb} & {525.3 $\pm$ 89.2} & \textbf{702.3 $\pm$ 124.3} & \textbf{69.3$\pm$ 12.8} &\textbf{9.8 $\pm$ 6.4} & \textbf{168.2 $\pm$ 5.3} \\
{Dreamer-v2} & \textbf{602.3 $\pm$ 48.5} & 438.2 $\pm$ 78.2 & \textbf{59.2 $\pm$ 15.9} & 2.3 $\pm$ 0.4    & \textbf{172.3 $\pm$ 8.0}\\
SAC-MLP & {483.3 $\pm$ 77.2} & 279.8 $\pm$ 190.6 & 19.2 $\pm$ 2.3 & {1.2 $\pm$ 0.1} & 163.6 $\pm$ 9.3\\
SLAC & 105.1 $\pm$ 30.1 & 139.2 $\pm$ 3.4 & 36.1 $\pm$ 15.3 &  0.9 $\pm$ 0.1 & \textbf{167.3 $\pm$ 11.2}\\
PSR & 173.7 $\pm$ 25.7 & 57.4 $\pm$ 7.4 & 23.2 $\pm$ 9.5 &  0.8 $\pm$ 0.1 & 159.4 $\pm$ 9.2\\
\midrule
%Best-ML & 639.3$\pm$24.5 & 724.2$\pm$37.8 &  72.9$\pm$40.6 & 11.8$\pm$6.8 &  167.1$\pm$3.1\\
Best-FO & 639.3$\pm$24.5 & 724.2$\pm$37.8 &  72.9$\pm$40.6 & 11.8$\pm$6.8 &  167.1$\pm$3.1\\
\bottomrule 
\end{tabular}
}
\end{table*}

%%%%%%%%%%%%%%%%%%%%%%%%%%%%%%%%%%%%%%%%%%%%%%%%%%%%%%%%%%%%%%%%%%%%%%%%%%%%%%%
%%%%%%%%%%%%%%%%%%%%%%%%%%%%%%%%%%%%%%%%%%%%%%%%%%%%%%%%%%%%%%%%%%%%%%%%%%%%%%%

\end{document}